\documentclass{article}

\usepackage[english]{babel}

\usepackage[letterpaper,top=2cm,bottom=2cm,left=3cm,right=3cm,marginparwidth=1.75cm]{geometry}
\usepackage{amsmath, amssymb, amsthm}
\usepackage{graphicx}
\usepackage[colorlinks=true, allcolors=blue]{hyperref}
\usepackage{authblk}
\usepackage{natbib}
\usepackage{algorithm}
\usepackage{algorithmic}
\usepackage{caption}
\usepackage{float}
\graphicspath{{./figures/}}

\usepackage{booktabs}

\theoremstyle{plain}
\newtheorem{theorem}{Theorem}[section]

\newtheorem{lemma}[theorem]{Lemma}

\theoremstyle{definition}
\newtheorem{definition}[theorem]{Definition}
\newtheorem{assumption}[theorem]{Assumption}
\theoremstyle{remark}
\newtheorem{remark}[theorem]{Remark}

\newcommand{\bw}{\boldsymbol{w}}
\newcommand{\bone}{\boldsymbol{1}}
\newcommand{\sD}{\mathcal{D}}
\newcommand{\sB}{\mathcal{B}}

\DeclareMathOperator{\Regret}{Regret}
\DeclareMathOperator{\LinRegret}{LinRegret}

\newcommand{\seq}[1]{\left( #1 \right)}
\newcommand{\set}[1]{\left\{ #1 \right\}}

\newcommand{\abs}[1]{\left| #1 \right|}
\newcommand{\sign}{{\rm sign}}

\newcommand{\field}[1]{\mathbb{#1}}

\newcommand{\R}{\field{R}}

\newcommand{\E}{\field{E}}

\newcommand{\argmin}{\mathop{\mathrm{argmin}}}
\DeclareMathOperator{\KL}{KL}

\title{Online Conformal Prediction via Universal Portfolio Algorithms}
\author[1]{Tuo Liu\thanks{\texttt{tuo.liu@kaust.edu.sa}}}
\author[2]{Edgar Dobriban\thanks{\texttt{dobriban@wharton.upenn.edu}}}
\author[1]{Francesco Orabona\thanks{\texttt{francesco@orabona.com}}}

\affil[1]{King Abdullah University of Science and Technology, Saudi Arabia}
\affil[2]{University of Pennsylvania, USA}

\begin{document}
\maketitle

\begin{abstract}
Online conformal prediction (OCP) seeks prediction intervals that achieve long-run 
$1-\alpha$ coverage for arbitrary (possibly adversarial) data streams, while remaining as informative as possible. Existing OCP methods often require manual learning-rate tuning 
to work well, and may also require algorithm-specific analyses. 
Here, we develop a general regret-to-coverage theory for interval-valued OCP based on the $(1-\alpha)$-pinball loss. 
Our first contribution is to identify \emph{linearized regret} 
as a key notion, showing that controlling it implies coverage bounds for any online algorithm. 
This relies on 
a black-box reduction that depends only on the Fenchel conjugate of an upper bound on the linearized regret.
Building on this theory, we propose UP-OCP, 
a parameter-free method for OCP, via a reduction to a two-asset portfolio selection problem, leveraging universal portfolio algorithms. 
We show strong finite-time bounds on the miscoverage of UP-OCP, even for polynomially growing predictions.
Extensive experiments support that UP-OCP delivers consistently better size/coverage trade-offs than prior online conformal baselines.
\end{abstract}

\section{Introduction}
\label{sec:intro}

Reliable uncertainty quantification is a central goal in modern statistical learning, especially when predictions must be accompanied by measures of confidence. 
A popular way to express such uncertainty is through \emph{prediction sets}, which return a set of predicted outcomes---rather than a single prediction---for a test point.
The key challenge is to construct these sets so that they are both informative (small) and valid (they achieve a prescribed coverage level) without relying on strong distributional assumptions.
Conformal Prediction (CP) has emerged as a prominent methodology for constructing prediction sets with finite-sample statistical validity~\citep[see, e.g.,][]{saunders1999transduction,vovk1999machine,papadopoulos2002inductive,vovk2005algorithmic,Vovk2013,Chernozhukov2018,lei2013distribution,lei2018distribution,guan2023localized,romano2020classification}.
CP can wrap around any predictive model to produce sets that contain the true label with a user-specified probability of at least $1 - \alpha$.
This framework has been used in various settings, including regression, classification, and structured prediction.

In its most basic form, CP provides coverage if the datapoints are
 \emph{exchangeable}. 
Since data streams are often not
exchangeable, 
conformal methods have been developed that can account for various 
distribution shifts, such as covariate shift~\citep{tibshirani2019conformal,qiu2022prediction,park2021pac}, label shift~\citep{podkopaev2021distribution,si2024pac}, more general distribution shifts~\citep{10.1214/23-AOS2276,gauthier2025values}, and
time series models~\citep{xu2021conformal,zaffran2022adaptive}.

A different line of work, sometimes called Online Conformal Prediction (OCP),
aims to completely do away with making assumptions on the data, and consider instead \emph{deterministic} and \emph{adversarial} data \citep[see, e.g.,][]{gibbs2021adaptive,zaffran2022adaptive,bastani2022practical,gibbs2024conformal,angelopoulos2023conformal,podkopaev2024adaptive}.
In this setting, one aims to achieve coverage $1-\alpha$ averaged over time (and any possible algorithmic randomness).
This work can be viewed to belong to the setting of \emph{online learning}~\citep{cesa2006prediction,hazan2016introduction,orabona2019modern}.

More formally, the observed data $\seq{\seq{X_{t}, Y_{t}}}_{t \geq 1}$ arrive over time $t=1, 2, \dots$.
At each time point (or round) $t$, a prediction set $\hat{C}_{t}$ is constructed for $Y_{t}$ using all the previously observed datapoints $\seq{\seq{X_{i}, Y_{i}}}_{i \le t-1}$, as well as the current features $X_{t}$.
Let $\hat{Y}_{t}$ represent a point prediction given by a model $\hat{f}_{t}$ trained using all the information available before the true response $Y_{t}$.
We are interested in regression problems, 
focusing on the perhaps most popular form of centered prediction sets \citep[see, e.g.,][]{Lei2018}: $\hat{C}_{t}(b) := [\hat{Y}_{t} - b, \hat{Y}_{t} + b]$, defined as empty if $b<0$.
The goal is to design a conformal predictor whose observed long-term miscoverage rate is close to the nominal level, denoted as $\alpha\in (0,1)$. Formally, we aim to construct a sequence of radii $\seq{b_t}_{t \geq 1}$ so that, as the time horizon $T$ grows, the corresponding prediction sets satisfy
\begin{equation} \label{eq:miscoverage-convergence}
    \lim_{T \to \infty} \ \abs{\frac{1}{T} \sum_{t=1}^{T} \bone \set{Y_t \notin \hat{C}_{t}(b_{t})} - \alpha}
    = 0~.
\end{equation}

\noindent\textbf{Related Work.}
Adaptive Conformal Inference (ACI) \citep{gibbs2021adaptive}
 maintains a quantile threshold $\alpha_t$ over time $t=1,2, \dots$, and updates it via Online Subgradient Descent (OSD) on the pinball (aka quantile) loss.
ACI achieves long-term coverage close to the target $1 - \alpha$ level.
However, its performance depends heavily on the stepsize. 
A small stepsize results in slow adaptation and miscoverage after a large distribution shift, 
while a large stepsize induces high variance and instability in the prediction set widths~\citep{gibbs2024conformal,angelopoulos2023conformal}.

To address this limitation, later work introduces alternatives. 
Multi-valid Conformal Prediction (MVP)~\citep{bastani2022practical} selects 
a threshold with the best historical coverage from a discretized grid.
MVP guarantees long-term and threshold-calibrated coverage at multiple levels, 
but lacks rapid adaptivity to abrupt changes~\citep{bastani2022practical}.

However, the above approaches require tuning a stepsize, 
which is challenging in the online setting.
Because the sequences are adversarial,
 we cannot rely, for example, on train-test validation to tune hyperparameters.
To avoid this problem, \citet{zaffran2022adaptive} propose Aggregated ACI, which uses online expert aggregation, running multiple copies of ACI with various stepsizes and forming a weighted ensemble.
Relatedly, Dynamically-Tuned ACI (DtACI)~\citep{gibbs2024conformal} re-weights experts in order to emphasize recent data; in effect it tunes ACI's stepsize online by minimizing a quantile loss.
However, as noted by \citet{angelopoulos2023conformal}, ACI-based methods can sometimes output infinite or null prediction sets, when $\alpha_{t}$ drifts below zero or above unity, respectively. Strongly Adaptive OCP~\citep{bhatnagar2023improved} aggregates a number of different base algorithms to guarantee the worst-case regret over each sub-intervals, but it requires uniformly bounded predictions and the knowledge of the maximum range of the predictions, and both assumptions often fail to hold in practice.
\citet{zhang2024discounted} and \citet{podkopaev2024adaptive} instead avoid the use of a stepsize by using ``parameter-free'' online learning algorithms~\citep{OrabonaP16,JunOWW17b,OrabonaP21},
such as scale-free online gradient descent \citet{podkopaev2024adaptive}.

A different approach has been recently proposed by \citet{srinivas2026online}, where one directly addresses the optimal trade-off between coverage and size of the confidence sets in a competitive ratio framework. Notably, their analysis confirms that robust coverage in the worst-case setting fundamentally necessitates larger prediction sets.
However, their algorithm, as the one by \citet{bhatnagar2023improved}, requires uniformly bounded predictions and the knowledge of the maximum range of the predictions.

A complementary line of work frames online conformal calibration as a feedback-control problem and proposes conformal P/PI/PID controllers~\citep{angelopoulos2023conformal}. Unlike these controller-based schemes---which introduce gain hyperparameters and typically require controller-specific analyses---we are interested in parameter-free approaches, more suited to the online setting.

Moreover, prior work used specialized analyses to prove that the algorithms can guarantee asymptotic coverage. 
To date, the precise connection between online learning and OCP is unclear.
In online learning, the central goal is to obtain a sublinear \emph{regret}, i.e., the difference between the cumulative loss of the algorithm and the one of the best fixed predictor chosen in hindsight. Algorithms that satisfy this property are said to be \emph{no regret}.
However, \citet{angelopoulos2025gradient} also show that achieving coverage
is, in general, completely distinct from achieving sublinear regret.
Similarly, as we discuss later, the notion of proximal regret~\citep{cai2024tractable} also implies coverage, but has been established only for gradient descent.
This makes it unclear when one can port methods from online learning for OCP.

\noindent\textbf{Contributions.}
We answer the following questions:

\emph{Is there a form of regret that implies coverage (see \eqref{eq:miscoverage-convergence})?}

In Section~\ref{sec:guarantees}, we show that the general notion of \emph{linearized regret} directly implies coverage.
Since linearized regret bounds have been established for several online learning methods, this will enable us to directly obtain coverage guarantees for these methods.
In particular, we 
answer this question by making a connection to the regret-reward duality from online learning~\citep{orabona2019modern}.
Once a bound on the miscoverage is established, it is natural to ask about optimality:

\emph{Is it possible to construct an optimal online algorithm (in terms of regret and coverage) for OCP?}

This remained unexplored in the past literature.
In Section~\ref{sec:universal-portfolio}, we design a parameter-free strategy that guarantees the best known finite-time coverage guarantee. Following classical work in parameter-free online learning \citep[see, e.g.,][]{orabona2019modern}, this is achieved by observing the equivalence between the OCP problem and a gambling one, then using universal portfolio algorithms~\cite{cover2002universal} to optimally solve the gambling problem. We will also show that our algorithm guarantees online coverage with any polynomial growth of the nonconformity scores. We call the resulting algorithm Universal Portfolio OCP (UP-OCP).

Finally, we introduce a way to empirically quantify the \emph{trade-off} between size and coverage, for a wide range of values of $\alpha$. In our extensive experiments on real and simulated datasets, UP-OCP achieves the best such trade-off among a number of strong baselines (Section~\ref{sec:experiments}).

\section{Notation and Problem Setup}
\label{sec:setup}

In this section, we formally introduce our notation and the problem setup.

\noindent\textbf{Notation.}
We define here some basic concepts and tools from convex analysis~\citep[see, e.g.,][]{Rockafellar70}.
For a function $f:\R \rightarrow \R$, we define a \emph{subgradient} of $f$ in $x \in \R$ as $g \in \R$ that satisfies $f(y)\geq f(x) + g (y-x), \ \forall y \in \R$.
The set of subgradients of $f$ in $x$ is called the \emph{subdifferential set} and we denote it by $\partial f(x)$.
The \emph{indicator function of the set $\mathcal{V}$}, $\bone_\mathcal{V}:\R\rightarrow (-\infty, +\infty]$, has value $0$ for $x \in \mathcal{V}$ and $+\infty$ otherwise.
For a function $f: \R\rightarrow [-\infty,\infty]$, we define the \emph{Fenchel conjugate} $f^\star:\R \rightarrow [-\infty,\infty]$ as $f^\star(\theta) = \sup_{x \in \R} \ (\theta x - f(x))$. The Fenchel conjugate is always well-defined and convex.

\noindent\textbf{Problem Setup.}
We consider the problem of OCP,
for arbitrary data streams, even adversarially generated ones, as introduced in Section~\ref{sec:intro}. 
Let $S_{t} \geq 0$ denote the radius of the smallest prediction set that contains the true response $Y_{t}$, i.e.,
$S_{t}:= \inf \{b \in [0,\infty): Y_{t} \in \hat{C}_{t}(b)\} = |Y_{t} - \hat{Y}_{t}|$.
We will also refer to $S_t$ as the non-conformity score \cite{vovk2005algorithmic}.
In terms of $S_{t}$, the target property \eqref{eq:miscoverage-convergence} is equivalent to
\[
    \lim_{T \to \infty} \ \abs{\frac{1}{T} \sum_{t=1}^{T} \bone \set{b_t < S_t} -\alpha}
    = 0~.
\]
This can be viewed as the problem of sequentially learning the $(1 - \alpha)$-th quantile of the nonconformity scores $\seq{S_{t}}_{t \geq 1}$.

A standard approach \citep[see, e.g.,][]{gibbs2021adaptive,podkopaev2024adaptive}, to learn this quantile is to use a proper scoring rule \citep[see, e.g.,][]{gneiting2007strictly},
namely the \emph{pinball (or quantile) loss}, defined as
\begin{equation} \label{eq:pinball-loss}
    \ell^{(1 - \alpha)}(b, S)
    := \max \set{(1 - \alpha) (S - b), \alpha (b - S)},
\end{equation}
where $S$ is the non-conformity score and $b$ is the radius of the prediction interval.
This loss is convex and $L$-Lipschitz in the first argument, where $L := \max \set{1 - \alpha, \alpha}$.
These two properties make it online learnable~\citep[see, e.g.,][]{hazan2016introduction,orabona2019modern,cesa2021online}.
For $t \ge 1$, let $b_{t}$ be the prediction, and $\ell_{t}(b_t) := \ell^{(1 - \alpha)}(b_t, S_t)$ be the  loss at round $t$. 
The \emph{regret} of the algorithm with respect to any fixed comparator $u \in \mathbb{R}$ is defined as
\begin{equation}\label{reg}
    \Regret_{T}(u) := \sum_{t=1}^{T} \ell_{t}(b_{t}) - \sum_{t=1}^{T} \ell_{t}(u)~.
\end{equation}
The subdifferential set of $\ell_{t}$ is
\[
    \partial \ell_{t}(b) = 
    \begin{cases}
        \{\bone \set{b \ge S_t} - (1 - \alpha)\}, & b \neq S_t \\
        [-(1 - \alpha), \alpha], & b = S_t~.
    \end{cases}
\]
At $b = S_t$, there is an infinite set of subgradients; throughout, we select the right subgradient $g_t=\alpha$, so that $g_t \in \{-(1-\alpha),\alpha\}$ for all $t$.
This choice is not arbitrary: since $\hat{C}_t(b_t)$ is closed, $Y_t \in \hat{C}_t(b_t)$ when $b_t = S_t$, and the right subgradient is the only one that matches this coverage event, which is necessary to connect the subgradients to the miscoverage error in \eqref{eq:miscoverage-gradients-equality} below.
Thus, an online learning algorithm predicting $b_t$ and receiving the pinball loss $\ell_t$ will receive the subgradient
\begin{equation}
\label{eq:subgradients}
    g_t = \bone \set{b_t \ge S_t} - (1 - \alpha)~.
\end{equation}
As explained in \citet{gibbs2021adaptive,angelopoulos2025gradient} the miscoverage error is closely related to the observed subgradients, because
\begin{align}\label{eq:miscoverage-gradients-equality}
\textnormal{MisCov}_T
:= \abs{\frac{ \sum_{t = 1}^{T} \bone \set{b_t \ge S_t}}{T} - (1 - \alpha)}
=\frac{\abs{\sum_{t = 1}^{T} g_{t}}}{T} .
\end{align}

\begin{remark}[Negative radii]
\label{rem:negative-radii}
The identity \eqref{eq:miscoverage-gradients-equality} relies on $\bone\set{b_t \ge S_t} = \bone\set{Y_t \in \hat{C}_t(b_t)}$, which holds for \emph{every} $b_t \in \R$, including negative ones: if $b_t < 0$, then $\hat{C}_t(b_t) = \emptyset$ and $b_t < 0 \le S_t$ following the convention in~Section~\ref{sec:intro}, so both indicators are zero.
Hence, an online algorithm may predict negative radii without any modification.
\end{remark}

\section{Coverage Guarantees for No-Regret Algorithms}
\label{sec:guarantees}

In this section, we describe our main result providing a coverage guarantee for any online algorithm controlling an appropriate form of linearized regret.

Consider an online learning algorithm that in each round $t=1, 2, \ldots$ 
produces an action $b_t \in \R$, and let $g_{t} \in \partial \ell_{t}(b_{t})$ denote a subgradient of the loss at round $t$. 
We consider the \emph{linearized regret}~\citep{Gordon99b,Zinkevich03} of the algorithm on this sequence at the action $u \in \R$, defined as
\begin{equation}\label{linreg}
    \LinRegret_{T}(u) := \sum_{t=1}^{T} g_{t} (b_t - u)~.
\end{equation}   
In contrast to the standard notion of regret from~\eqref{reg},
this quantity sums up to the linearizations
$g_{t} (b_t - u)$
of the loss differences
$\ell_{t}(b_{t}) - \ell_{t}(u)$ around $b_t$.
By the definition of subgradients, we have that $\Regret_{T}(u) \le \LinRegret_{T}(u)$ for all $u$.
Thus, \emph{any algorithm that controls the linearized regret also controls the usual regret}. 
However, an algorithm may control regret but not linearized regret. 
Crucially, \emph{our analysis shows that controlling the linearized regret suffices to ensure coverage}.

Specifically, we have the following result which bounds the range of the sum of the gradients $\sum_{t=1}^T g_{t}$  
depends on the Fenchel conjugate of a bound on the regret function (proof in Appendix~\ref{sec:proof_main_thm}).
Due to \eqref{eq:subgradients}, this immediately implies a bound on the coverage.  
\begin{theorem}
\label{thm:finite-time-bound}
For an online learning algorithm, let $F_{T}:\R \to \R$ such that
$\LinRegret_T(u)\leq F_T(u)$ on the pinball losses $\seq{\ell_{t}}_{1 \le t \le T}$.
Then, 
\begin{equation} \label{eq:finite-time-coverage}
    - \sum_{t=1}^T g_{t}
    \in \set{z \in \R: F_{T}^{\star}(z) \le (1 - \alpha) \sum_{t=1}^{T} S_{t}},
\end{equation}
where $F_{T}^{\star}(\cdot)$ is the Fenchel conjugate of $F_{T}(\cdot)$.
\end{theorem}

In Appendix~\ref{sec:asymptotic}, we show that a simpler asymptotic coverage result can be obtained more directly from our theory, without requiring the machinery of Fenchel conjugates.
Moreover, since, as we discussed, coverage can be achieved in 
trivial ways,
it is also important to have other correctness guarantees.
For this reason, in the standard conformal prediction setting of i.i.d. scores, we show 
 in Appendix~\ref{sec:stochastic}
that any no-regret algorithm ensures that  the averaged thresholds $b_t$ converge to the optimal one.
This provides an additional desired correctness guarantee in our framework. 

\begin{remark}
\citet[Example 1]{angelopoulos2025gradient} show that sublinear regret does not imply coverage.
Their example reduces to the following: consider $S_t=S>0$ for all $t$,
 and an online learning algorithm that predicts $S+{1}/{\sqrt{t}}$.
 While
the regret with respect to the optimal prediction $S$ is sub-linear, $\alpha \sum_{t=1}^T {1}/{\sqrt{t}} \le 2\alpha\sqrt{T}$,
we have that $g_t=\alpha$ for all $t$, so the coverage error does not vanish.
Our results are not in contradiction.
Specifically, as already discussed,
the linearized regret is stronger than the standard one.
By taking $u = S-\epsilon$, where $\epsilon>0$, as the competitor in $\LinRegret_T(u)$,
 we see that the linearized regret is \emph{not controlled}, growing at least as $\alpha \epsilon T$. Hence, $b_t=S+{1}/{\sqrt{t}}$ does not contrl the linearized regret.
This shows that our results are not in contradiction.
\end{remark}

\begin{remark}
\citet{cai2024tractable} derive coverage guarantees for OSD by showing that it minimizes a notion called \emph{proximal regret} for the specific class of linear functions.
However, they only prove this property for OSD, while we handle any online algorithm with a suitable linearized regret.
\citet{angelopoulos2025gradient} 
discuss
another notion, \emph{no-move regret},
as a special case of proximal regret, and show a corresponding asymptotic coverage result
for smooth losses.
Our guarantees are instead derived for algorithms minimizing the non-smooth pinball loss.
\end{remark}

\textbf{Warm-up: Coverage Guarantee for KT and OSD.}
To show the generality of our result, we first present a coverage analysis for the KT approach in~\citet{podkopaev2024adaptive}.
This consists of the parameter-free KT algorithm of \citet{OrabonaP16}, applied to the sequence of pinball losses. 
The \citet{podkopaev2024adaptive} 
only provide the asymptotic guarantee \eqref{eq:miscoverage-convergence}, 
while here we obtain a finite-time convergence using Theorem~\ref{thm:finite-time-bound}.

As proved by \citet{OrabonaP16}, 
a valid choice for 
the Fenchel conjugate $F_T^\star$ for the KT algorithm is
\begin{equation} \label{eqn:KT_bettor}
F_T^\star(\theta)
= \frac{1}{\sqrt{24 T}} \exp\left(\frac{\theta^2}{2 T}\right) - 1~.
\end{equation}
We now assume use the mild assumption that the growth rate of $S_t$ is bounded polynomially: Let $D>0$ and $q\geq0$, and assume $S_t \leq D t^q$ for all $t$. This assumption is strickly weaker than the standard boundedness condition ($S_t \leq C$) or i.i.d. assumptions typically required in prior work. It allows our guarantees to hold even in non-stationary environments where the scale of nonconformity scores expands over time.
As we demonstrate in Appendix~\ref{app:growth-demo}, this polynomial growth model is a more valid representation of real-world dynamics than a static bound.

Now, using Theorem~\ref{thm:finite-time-bound}, we immediately obtain
\[
\left|-\sum_{t = 1}^{T} g_{t}\right|
\le \sqrt{2 T\ln\left(\frac{\sqrt{24} D(1 - \alpha)}{q+1}T^{3/2+q} + \sqrt{24 T} \right)}.
\]
By \eqref{eq:subgradients}, the miscoverage is
$\textnormal{MisCov}_T=\mathcal{O}(\sqrt{\ln (DT) / T})$, regardless of the growth rate exponent $q$.

To show the full generality of our approach based on the conjugate of the linearized regret, in Appendix~\ref{sec:proof_coverage_osd}, we also show a minor variant of Theorem~\ref{thm:finite-time-bound} specialized for OSD with stepsize $\eta$ and $b_1=0$, that, under the same assumptions on $S_t$, gives the following bound:
$\textnormal{MisCov}_T= T^{-1} \abs{\sum_{t=1}^{T} g_{t}}
\le T^{-1}\left(D T^q / \eta + 1\right).$
With $q=0$, this rate matches the one in \citet{gibbs2021adaptive},
extending their analysis to the case that the scores $S_t$ can grow over time.

\section{A Universal-Portfolio Based Strategy}
\label{sec:universal-portfolio}

Thanks to Theorem~\ref{thm:finite-time-bound}, we now have a direct relationship between the linearized regret of an online algorithm and its miscoverage error, through the Fenchel conjugate $F_T^\star$.
Since a tighter (smaller) regret bound $F_T$ corresponds to a larger, steeper $F_T^\star$, minimizing regret leads to better coverage bounds.
Thus, for a fast bound on coverage, it is desirable to use an online algorithm with optimal regret.

It is known that the optimal linearized regret in unconstrained online learning is achieved only by parameter-free algorithms, as the ones in \citet{zhang2024discounted} and \citet{podkopaev2024adaptive}, see the lower bound in \citet[Section 5.3]{orabona2019modern}.
However, these algorithms are not fully optimal for OCP. The reason is that they implicitly assume a degree of symmetry (as explained below), whereas the coverage problem is inherently asymmetric: The target miscoverage rate $\alpha$ is typically chosen to be small (e.g., $\alpha = 0.05$ or $0.01$), implying that the positive and negative subgradients of the losses are very different.

Instead, we propose reducing our problem to a portfolio selection problem, 
and leveraging the \emph{Universal Portfolio} (UP) algorithm~\citep{cover2002universal} for OCP.  
In the following, we explain how UP methods lead to optimal solutions to online learning problems with asymmetric subgradients.
For our reduction, considering pinball losses, we construct a market with two synthetic stocks, whose market gains are driven by the observed miscoverage. 
\begin{definition}[The Conformal Market]
\label{def:market}
For a miscoverage rate $\alpha \in (0,1)$,
given the subgradient $g_t \in \set{-(1-\alpha), \alpha}$, $t\ge 1$ defined in \eqref{eq:subgradients}, we define the vector of \emph{returns} $\bw_t = (w_{t,1}, w_{t,2})^\top \in \R^2$ of two synthetic stocks as
\[
    w_{t,1}= \frac{-g_t}{\alpha} + 1, \quad
    w_{t,2}= 1 + \frac{g_t}{1 - \alpha}~.
\]
The returns are the coordinates of a \emph{market gain vector} $\bw_{t}$ representing the ratio of the closing price to the opening price for the two stocks.
\end{definition}
We have $w_{t,1}, w_{t,2} \geq 0$, because $-g_t \in \set{-\alpha, 1 - \alpha}
$. Stock 1 yields high returns when coverage is lost ($g_t < 0$), while Stock 2 yields moderate returns when coverage is maintained ($g_t > 0$).

We can now formally define the wealth of an algorithm operating in this market.
\begin{definition}[Wealth Process]
\label{def:wealth}
Consider an online algorithm that, at each round $t$, chooses a portfolio weight $\lambda_t \in [0, 1]$ representing the fraction of capital invested in Stock 1. The \emph{wealth} $W_t$ is defined as $W_0 = 1$ and
\begin{equation} \label{eq:wealth-recursion}
    W_t
    = W_{t-1} \cdot \left( \lambda_t w_{t,1} + (1 - \lambda_t) w_{t,2} \right)~.
\end{equation}
\end{definition}
The Universal Portfolio algorithm computes the weight $\lambda_{t}$ as the wealth-weighted average over the simplex of all possible constant portfolios. 
Let $W_{t-1}(\lambda) = \prod_{i=1}^{t-1} (\lambda w_{i,1} + (1 - \lambda) w_{i,2})$ be the wealth of a constant portfolio $\lambda$. Given a prior $\mu$ over $\lambda \in [0,1]$, the prediction is
\begin{equation}
\label{eq:up-update}
\lambda_{t}
= \frac{\int_0^1 \! \lambda \cdot W_{t-1}(\lambda) \, \mathrm{d}\mu(\lambda)}{\int_0^1 \! W_{t-1}(\lambda) \, \mathrm{d}\mu(\lambda)}~.
\end{equation}
Choosing $\mu(\lambda)={1}/[{\pi\sqrt{\lambda(1-\lambda)}}]$, \citet{cover2002universal} proved that the log wealth of the algorithm is at least the log wealth
of the best constant $\lambda$ in each round, up to a slack of $\frac{1}{2}\ln (\pi (T+1))$. This regret guarantee is optimal up to constant additive factors~\citep{cover2002universal}.

The critical insight formalized in the following Theorem---which follows from \citet[Lemma 1]{orabona2023tight} and the standard reduction of OCO to coin betting \citep{orabona2019modern}---is that \emph{maximizing the logarithmic growth of this wealth is equivalent to minimizing the linearized regret on the pinball loss}. 
This allows us to translate the wealth guarantees of portfolio algorithms to coverage through \eqref{eq:finite-time-coverage}.
The proof is in Appendix~\ref{sec:appendixC}.

\begin{theorem}[Regret of UP-OCP]\label{thm:regret_up_ocp}
Let $\mathcal{A}$ be the universal portfolio algorithm with $\mu(\lambda)=\frac{1}{\pi\sqrt{\lambda(1-\lambda)}}$ that outputs weights $(\lambda_t)_{t\ge 1}$ on the conformal market weights $\bw_t \in \R^2$.
Define $b_t$ via 
\begin{equation} \label{eq:radius-mapping}
    b_t
    = W_{t-1} \cdot [ -(1-\alpha)^{-1} + \lambda_t/(\alpha(1-\alpha)) ]~.
\end{equation}
Then, the resulting sequence of actions $(b_t)_{t\ge 1}$
achieves linearized regret
$\mathrm{LinRegret}_T(u)\le F_T(|u|)$ for all $u\in\R$ 
in online quantile loss minimization, 
where $F_T(|u|)$ is
\[
\max \left\{|u| \sqrt{2T \alpha(1-\alpha) \ln (4 (T+1)^{3/(2 \alpha)}(1-\alpha) u^2 +1)}, \ \frac{4}{3}|u|\left(\ln (3 |u| \sqrt{T+1})-1\right)
\right\}~.
\]
\end{theorem}

In the worst case,
 the regret is of order $\mathcal{O}(\sqrt{T \ln T})$. 
 However, when $\alpha\to 1$,
 the bound is  of  order $\ln T$.
In Appendix~\ref{sec:up_vs_kt}, we also show that the regret upper bound 
for UP-OCP is better than the one 
for the KT approach from \citet{podkopaev2024adaptive}.

\subsection{Coverage Guarantee for UP-OCP}
\label{sec:up_analysis}

We now provide a coverage guarantee for the proposed UP-OCP strategy.
We also quantify its advantage over the Krichevsky-Trofimov (KT) bettor, particularly in the regime of small $\alpha$.
The argument relies on a second-order expansion of the optimal wealth, which reveals that \emph{UP adapts to the variance of the gradients}, whereas KT implicitly assumes a worst-case symmetric variance.

\begin{theorem}[Coverage bound for UP-OCP]
\label{thm:up-coverage}
Let $\alpha\in(0,1)$ and let $\seq{g_t}_{t=1}^T$ be the sequence of subgradients observed by UP-OCP, with $g_t\in\{-(1-\alpha),\alpha\}$.
Let $D>0$ and $q\geq 0$, and assume that $S_t \leq D t^q$ for all $t$.
For every integer $T\ge 1$, define
\[
\varepsilon_T:=\frac{1}{T}\left[\ln \left(1+\frac{(1-\alpha) D (T+1)^{q+1}}{q+1}\right)+\frac12\ln(\pi(T+1))\right]~.
\]
Then, $\textnormal{MisCov}_T=|T^{-1} \sum_{t=1}^T g_t|\le \varepsilon_T+\sqrt{2\alpha(1-\alpha)\varepsilon_T}$.
\end{theorem}

\textbf{Comparison with Previous OCP Algorithms.}
Now, we compare the coverage rates of UP-OCP with
existing bounds. 
While only compare upper bounds and not actual coverage, 
we will show in Section~\ref{sec:experiments} that 
the bounds are consistent with the empirical results.

First, we compare
 with the coverage guarantee 
 for OSD. 
In OSD, one has to choose a stepsize $\eta$. The OSD bound decreases with $\eta$, so one might be tempted to set $\eta$ to be large.
However, in that case,
OSD would predict zero 
on the first round and, if the scores were bounded by $D$,
 values larger than $D$ later,
until predicting a non-positive number.
The cycle would then repeat.
This behavior would not be
informative for uncertainty quantification.

A more meaningful setting of $\eta$ is the one that minimizes the worst-case regret, that is $\eta=\frac{\sqrt{D T^q}}{\max(\alpha, 1-\alpha)\sqrt{T}}$. 
This gives a 
coverage bound of ${\sqrt{D T^q}\max(\alpha, 1-\alpha)}/{\sqrt{T}}+{1}/{T}$. 
Contrary to the bound we derived for the KT strategy, this rate deteriorates with $q$.
However, things are even worse: the choice of $\eta$ that depends on $q$ and $D$ cannot be used, because $D$ and $q$ are not available to the algorithm. 
In this situation one can only use the stepsize $\eta={c}/[{\max(\alpha, 1-\alpha)\sqrt{T}}]$ where $c>0$ is a hyperparameter. 
With this choice the coverage will converge with the worse rate of $\tfrac{D T^q\max(\alpha, 1-\alpha)}{c\sqrt{T}}+\tfrac{1}{T}$.

Overall, we can see that OSD, tuned or untuned with the oracle knowledge of $q$ and $D$, has a worse dependency in $T$ if $q>0$. Moreover, when $\alpha \to 1$ or $\alpha \to 0$, the coverage rate of UP-OCP approaches $({\ln T})/{T}$ while the one of OSD cannot be faster than $1/\sqrt{T}$.

Next, let us now consider the KT strategy. UP-OCP and the KT strategy have the same dependency on $T$, $D$, and $q$. However, KT is designed to be min-max optimal for coin-betting games with symmetric outcomes, that is, with $\alpha=1/2$.
So, as for OSD, we see that the rate of the KT strategy does not improve when $\alpha\to 1$ or $\alpha\to 0$.
A similar rate was shown for a parameter-free algorithm in \citet{zhang2024discounted}, but only in an asymptotic sense. Moreover, the betting strategy implicit in the algorithm in \citet{zhang2024discounted} is provably inferior to the one of universal portfolio, because it matches only the leading term of the growth rate of best rebalanced portfolio.

\subsection{Closed-Form Update for Universal Portfolios}

A direct evaluation of \eqref{eq:up-update} with
$\mu(\lambda)={1}/[{\pi\sqrt{\lambda(1-\lambda)}}]$ 
can be implemented with cumulative time complexity up to time $t$ of $\mathcal{O}(t^2)$~\cite{cover2002universal}.
However, for our specific Conformal Market defined in Definition~\ref{def:market},
there is a simple closed-form update $\lambda_t = \frac{1}{t}(\sum_{i=1}^{t-1} \bone\set{g_i = -(1-\alpha)} +\frac12)$, see Appendix~\ref{sec:closed_form}.
Substituting this into \eqref{eq:radius-mapping} yields a parameter-free update rule for the conformal radius $b_{t}$ that adapts to the asymmetry of the gradients.
See Algorithm \ref{alg:up_ocp} for the complete pseudocode.

\begin{algorithm}[t]
  \caption{Universal Portfolio for OCP (UP-OCP)}
  \label{alg:up_ocp}
  \begin{algorithmic}
    \STATE {\bfseries Input:} Target miscoverage rate $\alpha \in (0,1)$
    \STATE {\bfseries Initialize:} Wealth $W_0 \leftarrow 1$, miscoverage count $ N \leftarrow 0$
    \FOR{$t=1, 2, \dots$}
      \STATE $\lambda_t \leftarrow \dfrac{N + 1/2}{t}$;\, $b_t \leftarrow W_{t-1} \cdot \dfrac{\lambda_t - \alpha}{\alpha(1-\alpha)}$

      \STATE Output prediction set $\hat{C}_t \leftarrow [\hat{Y}_t - b_t, \hat{Y}_t + b_t]$ \COMMENT{$\hat{C}_t = \emptyset$ if $b_t < 0$}
      \STATE Observe true label $Y_t$
      \STATE Compute nonconformity score $S_t \leftarrow |Y_t - \hat{Y}_t|$

      \IF{$S_t > b_t$}
        \STATE $N \leftarrow N + 1$;\, $W_t \leftarrow W_{t-1} \cdot \lambda_t/\alpha$
      \ELSE
        \STATE $W_t \leftarrow W_{t-1} \cdot (1 - \lambda_t)/(1 - \alpha)$
      \ENDIF
    \ENDFOR
  \end{algorithmic}
\end{algorithm}

Note the multiplicative nature of the algorithm, characteristic of all parameter-free algorithms, that allows to adapt to any polynomial growing sequence of $S_t$.

\section{Experiments}
\label{sec:experiments}

We support our theoretical findings through evaluations on both synthetic data and empirical time series. 
Our experiments cover the finance and energy domains, where data can be highly non-stationary.
Across these settings, we compare the proposed UP-OCP method (Algorithm~\ref{alg:up_ocp}) with state-of-the-art parameter-free and tuned baselines.

\textbf{Datasets and Models.}
It is notoriously extremely difficult to test algorithms in the online setting, because it would require constructing adversarial sequences for each algorithm. Hence, we follow the OCP literature~\citep{gibbs2021adaptive,angelopoulos2023conformal,gibbs2024conformal,podkopaev2024adaptive} and choose standard non-adversarial benchmark datasets.
We consider three categories of data:
(1) daily opening prices (log-scale) of four major US stocks (American Express, Apple, Amazon, and Google) from 2008 to 2018~\citep{nguyen2018s};
(2) electricity demand records from New South Wales~\citep{harries1999splice};
and
(3) synthetic environments.
For the base predictive models, we employ domain-standard choices:
\begin{itemize}
    \item Stock prices: We use \texttt{Prophet}~\citep{taylor2018forecasting}, re-fitted daily.
    \item Electricity demand: We use a standard Auto-Regressive (AR) model with a lag 3.
    \item Synthetic data: Following the protocol in \citet[Appendix F.5]{angelopoulos2023conformal}, we bypass the prediction step and directly simulate the nonconformity scores $S_t$. We generate three distinct patterns: a sinusoidal wave with Gaussian noise, and two trends (constant and quadratic) with random sparse bumps. Full details are provided in Appendix~\ref{app:additional-experiments}.
\end{itemize}

The synthetic tests simulate a forecaster with residuals exhibiting specific challenging patterns, isolating the OCP method's behavior from the base model's dynamics.
Also, the synthetic nature allows us to use multiple trials to generate error bars.

In all cases, the nonconformity scores are defined as the absolute residuals $S_t = |Y_t - \hat{Y}_t|$. Due to space constraints, we focus our main analysis on the American Express (AXP) dataset and the synthetic sinusoidal environment.
Full results for synthetic data, electricity demand, and other stock tickers are deferred to the Appendix.

\textbf{Baselines.} 
We compare UP-OCP with previous OCP algorithms that do not rely on prior knowledge of a uniform upper bound on the nonconformity scores:
\begin{itemize}
    \item Krichevsky-Trofimov (KT)~\citep{podkopaev2024adaptive}.
    \item Dynamically-tuned Adaptive Conformal Inference (DtACI)~\cite{gibbs2024conformal}.
    \item Scale-Free Online Gradient Descent (SF-OGD): A variant of OSD that adapts to the scale of the gradients~\citep{orabona2018scale}.
    \item Conformal P/PI Control (P/PI Ctrl)~\citep{angelopoulos2023conformal}.
\end{itemize}

For the parameterized baselines, we performed a grid search to select the best hyperparameters based on ex-post performance. 
In contrast, UP-OCP, KT, and DtACI are fully parameter-free and require no tuning. 
Crucially, tuning baselines on ex-post data grants them an oracle advantage. This means that we may overestimate their performance. Also, we emphasize that for parameterized baselines, a reasonable performance heavily depends on fine-tuned hyperparameter choices. As we demonstrate in Appendix~\ref{app:sensitivity}, without careful tuning, they can fail to maintain valid coverage, showing poor behavior locally and globally. Detailed update rules and hyperparameter search grids are provided in Appendix~\ref{app:baselines}.

\textbf{Metrics.}
From what we said above, it should be clear that the coverage alone is a meaningless metric. Indeed, we show a trivial predictor in Appendix~\ref{app:baselines} that achieves coverage without any informative prediction set.
So, here \emph{we heavily focus on the trade-off between coverage and prediction set sizes}, by proposing the use of Pareto frontier plots.
We measure the size of a prediction set $\hat{C}_t(b_t)$ by its Lebesgue measure, i.e., its length $2\max(0, b_t)$; in particular, the empty set output when $b_t < 0$ (see Remark~\ref{rem:negative-radii}) has size $0$.

\subsection{Results for AXP}

We begin by evaluating the efficiency-coverage trade-off on the American Express (AXP) dataset, setting $\alpha = 0.05$. An initial warm-up period of 100 days is used for training the initial base forecaster. Table~\ref{tab:AXP-1vall-metrics} reports the performance metrics. 
Due to space limitations, the full results for control-based methods are in Appendix~\ref{app:AXP-full-results}.

\begin{table}[H]
  \caption{Quantitative Comparison on the AXP Dataset. Performance metrics for UP-OCP versus parameter-free (KT, DtACI) baselines and tuned SF-OGD (lr=25).
  }
  \label{tab:AXP-1vall-metrics}
  \begin{center}
    \begin{tabular}{lcccc}
      \toprule
      Metric & UP-OCP & KT & DtACI & SF-OGD \\
      \midrule
      Marginal Coverage & 0.932 & 0.920 & 0.956 & 0.948 \\
      Longest Err. Seq. & \textbf{4} & 15 & 6 & \textbf{3} \\
      Avg. Set Size     & \textbf{14.8} & 16.9 & $\infty$ & \textbf{16.4} \\
      Median Set Size   & \textbf{11.5} & 12.9 & 12.6 & 13.8 \\
      75\% Quantile Size& \textbf{18.8} & 24.9 & 21.8 & 21.5 \\
      90\% Quantile Size& \textbf{32.3} & 32.5 & $\infty$ & \textbf{32.3} \\
      95\% Quantile Size& 38 & 36.1 & $\infty$ & 36.1 \\
      \bottomrule
    \end{tabular}
  \end{center}
\end{table}

Our UP-OCP method achieves valid coverage; although slightly lower (around 93\%) than 
the target, such differences 
may typically be ignorable in terms of practical importance \citep{podkopaev2024adaptive}.
Furthermore, UP-OCP limits consecutive miscoverage events to just four days, matching the best tuned baseline (SF-OGD), indicating a rapid correction mechanism. In contrast, the KT strategy suffers from a long error sequence of 15 days. This supports that the symmetric betting strategy of KT is too conservative for small $\alpha$, failing to expand intervals sufficiently fast during extreme events.
UP-OCP also demonstrates competitive efficiency compared to 
KT across most size metrics.
 Specifically, UP-OCP achieves smaller set sizes across average, median, and 75\% quantile metrics, while maintaining higher marginal coverage (0.932 vs. 0.920). It is only in the extreme upper tails (95\% quantiles) that UP-OCP produces larger sets, a necessary behavior to correct for miscoverage during high-volatility events.

Notably, the degeneracy of DtACI is clear from the table despite its valid marginal rate---the 90\% quantile size is infinite. This supports that DtACI does not just produce large sets occasionally; it relies on trivial predictions to compromise the coverage at least 10\% of the rounds.
Instead, UP-OCP achieves coverage without generating infinite sets.

\textbf{Local Coverage and Efficiency.}
Global coverage can be insufficient if the errors are grouped in time.
To complement Table~\ref{tab:AXP-1vall-metrics}, we provide detailed 1-vs-1 comparisons of local adaptivity between UP-OCP and the baselines in Appendix~\ref{app:AXP-full-results}. These plots show that UP-OCP maintains local coverage tightly around the 95\% target with no significant swings, while other methods exhibit marked volatility.

\textbf{Takeaway.}
UP-OCP matches the performance of the oracle tuned SF-OGD, the best tuned baseline. Both methods achieve minimal error clustering and similar average set sizes.
However, SF-OGD was ex-post tuned to select the optimal learning rate,
whereas UP-OCP is parameter-free.

\textbf{Pareto Frontiers.}
In general, higher coverage requires larger sets; however 
this trade-off varies across algorithms.
Here, we prioritize the \emph{average} prediction set size as our primary metric, as it reflects the cumulative cost of uncertainty in downstream tasks.
Unlike the quantiles of the length, 
the average better captures unreasonably large prediction sets.
This distinction is critical for separating stable algorithms from those that trivially satisfy coverage by outputting vacuous sets during volatility.
In Appendix~\ref{app:AXP-full-results}, we also show that UP-OCP remains dominant across quantiles of the length.

\begin{figure}[H]
  \centering
  \includegraphics[trim={0 0.4cm 0 1.2cm}, clip, width=0.7\columnwidth]{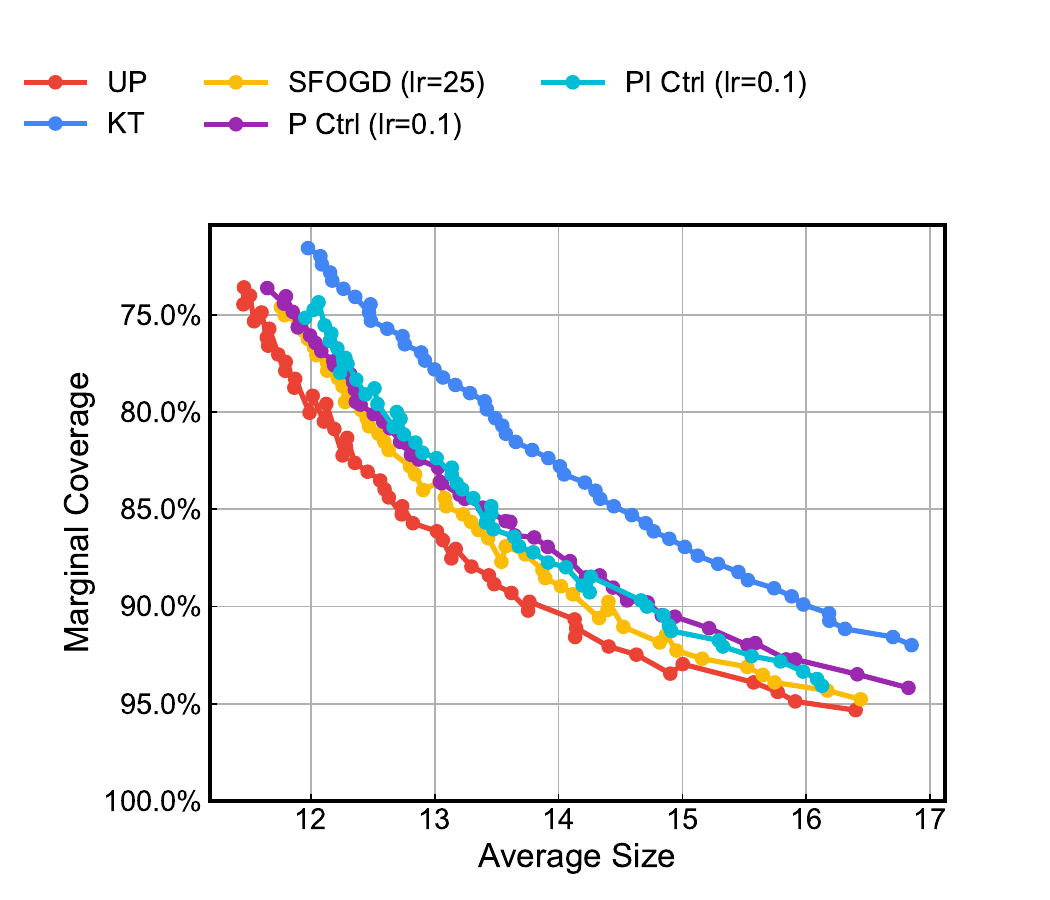}
  \caption{
    Pareto frontiers for average prediction set size on the AXP Dataset, for 50 target miscoverage rates $\alpha$ uniformly from 0.05 to 0.25.
    Better performance is closer to the bottom-left corner.
  }
  \label{fig:AXP-pareto-frontier-mean}
\end{figure}
Figure~\ref{fig:AXP-pareto-frontier-mean} illustrates 
that UP-OCP (red) empirically achieves the best Pareto trade-off. 
UP-OCP consistently achieves the smallest average set size for any given target coverage with in the range $[0.75, 0.95]$. 
The parameter-free KT baseline (blue) is strictly suboptimal.\footnote{DtACI is excluded here, because it produces prediction sets with infinite radius ($b_t = \infty$) to satisfy coverage; see Table~\ref{tab:AXP-1vall-metrics}. We observed this at all target levels, making the mean set size infinite.
}

\textbf{Target Calibration.} 
Beyond efficiency, a reliable OCP algorithm must also track the user-specified target $1-\alpha$.
Figure~\ref{fig:AXP-target-level-tracking} plots the realized versus target coverage for $\alpha \in [0.05, 0.25]$. 
The results confirm that UP-OCP and the baselines maintain calibration within a tight $\pm 0.03$ tolerance band (dashed lines) over all targets.
In Appendix~\ref{app:alpha-correction}, we also provide a heuristic to improve the tracking of any OCP algorithm.

\begin{figure}[H]
  \centering
  \includegraphics[trim={0 0.5cm 0 1.0cm}, clip,width=0.7\columnwidth]{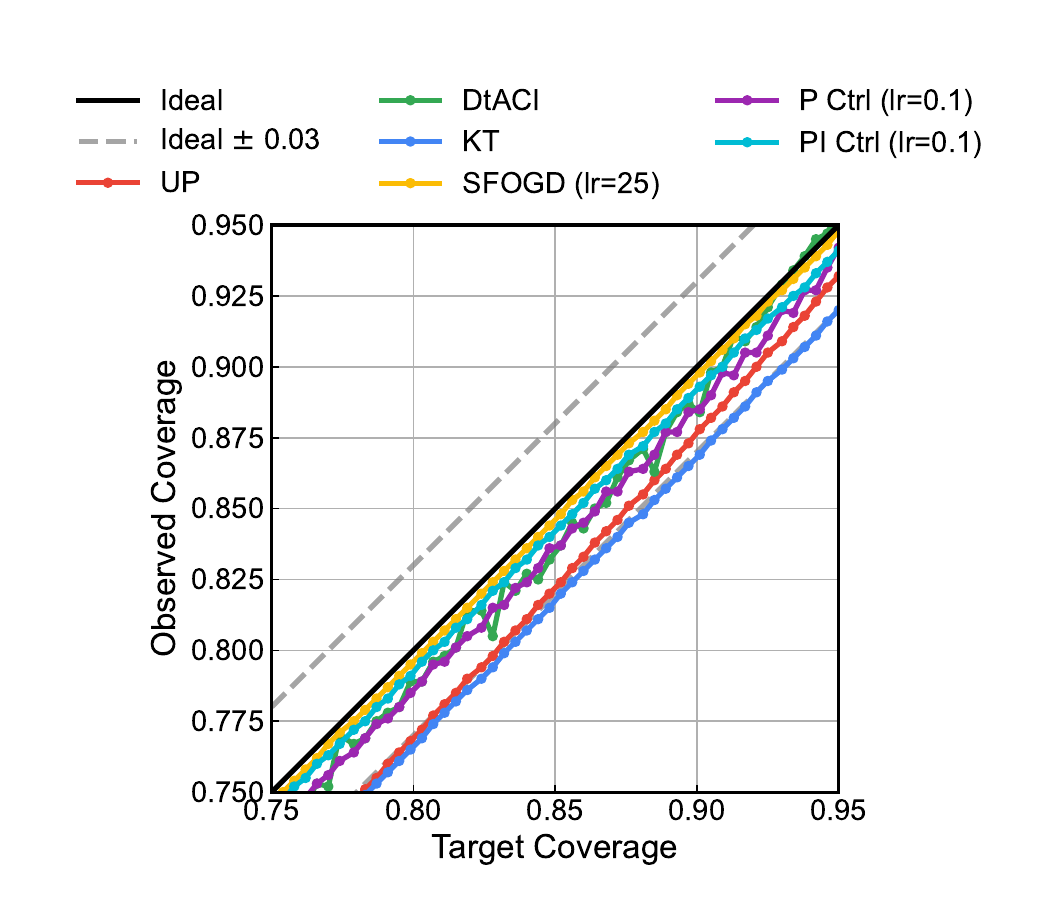}
  \caption{
    Realized vs. target coverage. }
  \label{fig:AXP-target-level-tracking}
\end{figure}

\subsection{Results for Synthetic Sinusoid}

We analyze the performance on a synthetic dataset designed to test adaptivity to periodic volatility. Following \citet[Appendix F.5]{angelopoulos2023conformal}, we generate nonconformity scores $S_t$ as a sinusoid with Gaussian noise:
\(
    S_t
    = \max(0, [\sin((2\pi t)/P) + 0.5] S_{\text{mag}} + S_{\min} + \epsilon_t ),
\)
where $\epsilon_t \mathop{\sim}\limits^{\text{i.i.d.}} \mathcal{N}(0, \sigma^2)$.
We fix the period $P=200$, magnitude $S_{\text{mag}}=10$, minimum offset $S_{\min}=2$, and noise scale $\sigma=0.3$.
The total sequence length is $T=3000$, and we report results averaged over 10 independent trials.

\begin{figure}[ht]
  \centering
  \includegraphics[trim={0 1.2cm 0 1.2cm}, clip,width=0.7\columnwidth]{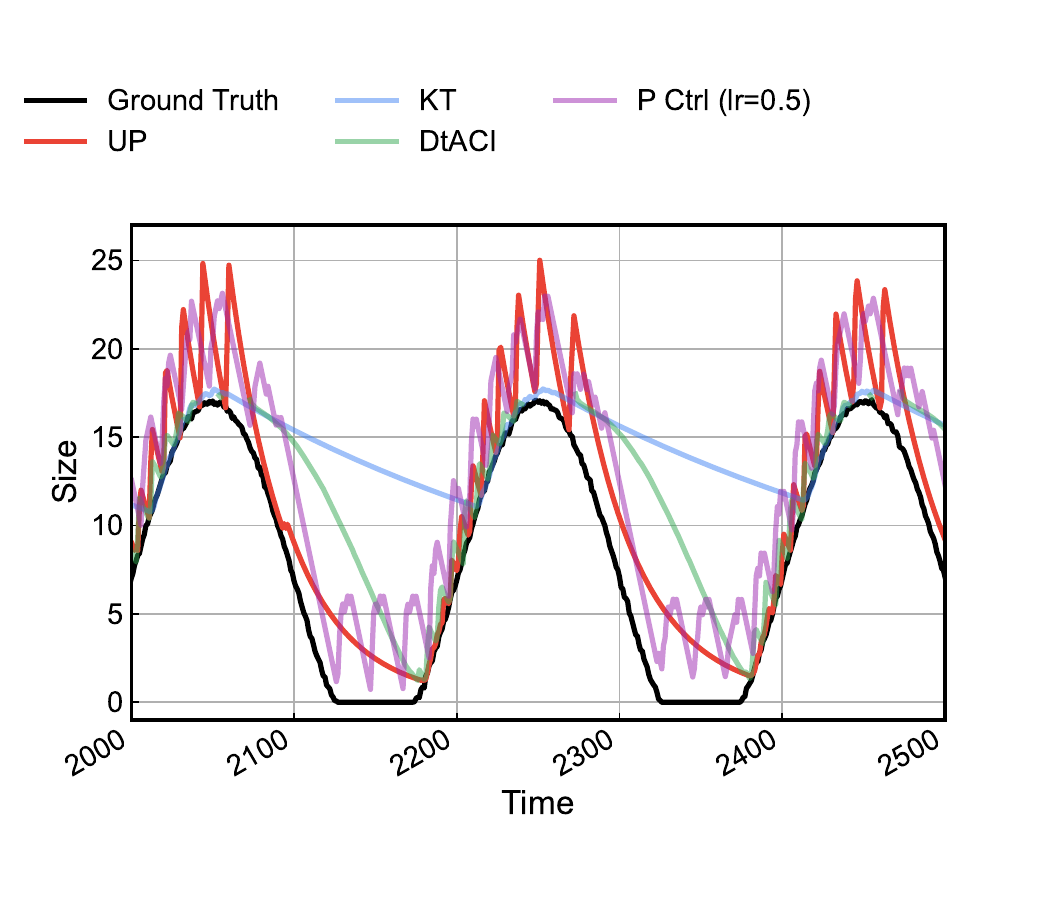}
  \caption{
    Pareto frontiers on the synthetic sinusoid.
    Optimal performance is the bottom-left corner (tightest sets for highest coverage).
  }
  \label{fig:sinusoid-sizes-demo}
\end{figure}

\textbf{Tracking Dynamics.} Figure~\ref{fig:sinusoid-sizes-demo} visualizes the evolution of interval widths over a representative window $t \in [2000, 2500]$. 
The ground truth (black) exhibits a clear periodic pattern.
UP-OCP (red) tracks this,
expanding rapidly during high-volatility phases ($t \approx 2050$) to guarantee coverage, 
then shrinking as the noise variance decreases. 
In contrast, KT (blue) 
shows significant lag and fails to reduce interval widths sufficiently during low-noise periods.
DtACI (green) shows strong instability; the trace terminates after $t > 2400$ (marked by missing values in the plot), indicating the algorithm has diverged and is outputting \emph{infinite} sets to reach coverage.
Again, UP-OCP matches the best tuned baseline, P-Control (purple), 
without hyperparameter tuning.
Other parameterized baselines behave similarly to the P-Controller; 
see Appendix~\ref{app:sinusoid}.
\begin{figure}[H]
  \centering
  \includegraphics[trim={0 0.4cm 0 0.8cm}, clip,width=0.7\columnwidth]{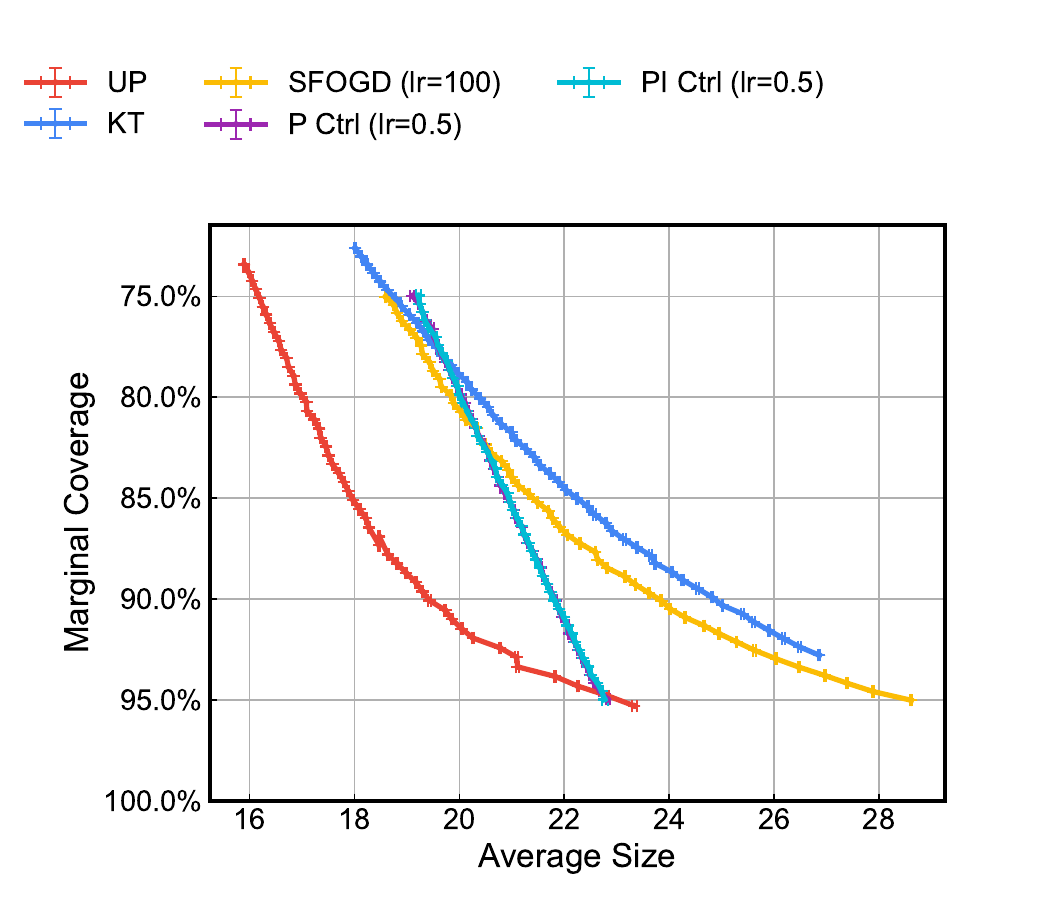}
  \caption{
    Plot of the realized marginal coverage against the average prediction set size, averaged over 10 independent random seeds.
    Error bars indicate the standard error of the mean along both axes.
  }
  \label{fig:sinusoid-pareto-mean}
\end{figure}

\textbf{Pareto Frontiers.}
Figure~\ref{fig:sinusoid-pareto-mean} quantifies the
stability of these findings by showing error bars over random repetitions.
Consistent with the AXP results, UP-OCP (red) 
dominates the baselines.
The vertical error bars are negligible, confirming that all methods 
satisfy validity. 
In contrast, the size metrics show larger oscillations.

\vspace{-1em}
\section{Discussion}
\label{sec:conclusion}

This paper shows that coverage guarantees for online conformal prediction (OCP) can be derived from linearized regret bounds.  
Then it proposes UP-OCP, a parameter-free strategy for OCP 
using Universal Portfolio methods.
Open directions include extending the framework beyond symmetric intervals toward richer prediction sets or conditional/feature-dependent validity.

\section*{Acknowledgements}
ED's work was supported in part by the US NSF, ARO, AFOSR, ONR, the Simons Foundation and the Sloan Foundation.

\newpage
\bibliographystyle{plainnat}
\bibliography{conformal}

\newpage
\appendix
\section{Proof of Theorem~\ref{thm:finite-time-bound}}
\label{sec:proof_main_thm}

Our main result leverage the following simple but fundamental lemma.
This lemma allows us to lower bound the term $g_tb_t$ which arises in the linearized regret, without proving a bound on the iterates of the algorithm.
\begin{lemma}
\label{lemma:bound_reward}
Let $b_t \in \R$, $t\ge 1$ be generated by any online learning algorithm and let $g_t \in \partial \ell_t(b_t)$ denote a subgradient of the pinball loss $\ell^{(1 - \alpha)}(b, S_t)$ at $b=b_t$,  for all $t$.
Then, we have
\begin{equation} \label{eq:wealth_gain}
    -g_t b_t \leq (1 - \alpha) S_t, \qquad \forall t~.
\end{equation}
\end{lemma}
\begin{proof}
We consider cases.
If $b_t<0$, then $b_t<S_t$ since $S_t\ge 0$, hence $g_t=-(1-\alpha)$ and so
$-g_t b_t = (1-\alpha)b_t \leq 0 \leq (1-\alpha)S_t$.
If $0 \leq b_t < S_t$, then again $g_t=-(1-\alpha)$ and thus
$-g_t b_t = (1-\alpha)b_t \leq (1-\alpha)S_t$.
Finally, if $b_t \ge S_t$, then (including the tie case $b_t=S_t$) we have $g_t=\alpha$, so
$-g_t b_t = -\alpha b_t \leq 0 \leq (1-\alpha)S_t$.
\end{proof}

We can now prove our main result.

\begin{proof}[Proof of Theorem~\ref{thm:finite-time-bound}]
By the assumed linearized-regret bound~\eqref{eq:regret-bound}, for every $u\in\R$ we have
\[
\sum_{t=1}^T g_t(b_t-u)
\leq F_T(u)~.
\]
Rearranging yields
\[
\left(-\sum_{t=1}^T g_t\right)u - F_T(u)
\leq -\sum_{t=1}^T g_t b_t~.
\]
Taking the supremum over $u\in\R$ gives
\[
F_T^\star\left(-\sum_{t=1}^T g_t\right)
\leq -\sum_{t=1}^T g_t b_t~.
\]
Finally, Lemma~\ref{lemma:bound_reward} implies $-g_t b_t \leq (1-\alpha)S_t$ for all $t$, hence $-\sum_{t=1}^T g_t b_t \leq (1-\alpha)\sum_{t=1}^T S_t$.
Combining the two inequalities completes the proof.
\end{proof}

\section{Proof of Coverage of OSD}
\label{sec:proof_coverage_osd}

We show that our analysis can recover the known coverage bound for online subgradient descent \citep{gibbs2021adaptive}.
From the well-known regret guarantee of online subgradient descent---OSD---see, e.g., \citet{orabona2019modern},
for all $u \in \R$, we have the following bound on the linearized regret:
\begin{equation}
\LinRegret_T(u)
\leq \frac{u^2}{2\eta} + \frac{\eta}{2}\sum_{t=1}^T g_t^2 =: F_T(u), \label{eq:osd_regret2}
\end{equation}
where $g_t$ are defined in~\eqref{eq:subgradients}.
It follows that
\(
F^{\star}_T(\theta)
= \sup_{u \in \R} \ \theta u - F_T(u)
= \frac{\eta \theta^2}{2 } - \frac{\eta }{2} \sum_{t=1}^T g_t^2,
\)
Let $D>0$ and $q\geq 0$. Assuming that $S_t \leq D t^q$ for all $t$, by Lemma~\ref{lemma:bound_reward} and Theorem~\ref{thm:finite-time-bound}, the miscoverage can be bounded as
\[
\frac{1}{T} \abs{\sum_{t=1}^{T} g_t}
\leq \frac{1}{\sqrt{T}} \sqrt{\frac{2(1-\alpha)D}{(1+q)\eta}T^{q} + \max(\alpha^2, (1-\alpha)^2)}~.
\]
For the case of $q=0$, \citet{gibbs2021adaptive} proved the better bound of $\mathcal{O}(1/(\eta T)+1/T)$.
The reason is that Lemma~\ref{lemma:bound_reward} is not sharp in this specific case.
Indeed, we can prove the following improved bound for OSD.

\begin{lemma}[Improved Guarantee for OSD]\label{io}
Let $D>0$ and $q\geq 0$, and assume that $S_t \leq D t^q$ for all $t$.
The OSD algorithm with fixed stepsize $\eta$ and $b_1=0$ over a sequence of pinball losses $\ell_t$ guarantees, with  $g_t \in \partial \ell_t(b_t)$,
\[
-\sum_{t=1}^T g_t b_t
\leq \frac{(D T^q+\eta)^2}{2\eta} - \frac{\eta}{2}\sum_{t=1}^T g_t^2~.
\]
\end{lemma}
\begin{proof}
The (unconstrained) OSD update is $b_{t+1}=b_t-\eta g_t$. Expanding gives
$b_{t+1}^2 = b_t^2 - 2\eta g_t b_t + \eta^2 g_t^2$.
Summing from $t=1$ to $T$ and using $b_1=0$ yields
\[
\frac{b_{T+1}^2}{2\eta}
= -\sum_{t=1}^T g_t b_t + \frac{\eta}{2}\sum_{t=1}^T g_t^2~.
\]
Thus, it suffices to show $|b_{T+1}|\leq D T^q +\eta$.
Similarly
to \citet{gibbs2021adaptive}, we will show a boundedness property for the iterates $b_t$.
We prove by induction that $|b_t|\leq D (t-1)^q + \eta$ for all $t$. The base case $b_1=0$ is immediate. Assume $|b_t|\leq D (t-1)^q + \eta$ and let's prove that $|b_{t+1}|\leq D t^q + \eta$.
If $0 \leq b_t < D t^q$, then
$|b_{t+1}| = |b_t-\eta g_t| \leq |b_t| + \eta |g_t|\leq D t^q + \eta$ since $|g_t|\leq 1$.
If $b_t\ge D t^q$, then $b_t \ge S_t$ because $S_t\leq D t^q$, hence $g_t=\alpha$ and
$D t^q - \eta \leq b_{t+1} \leq b_t \leq D (t-1)^q +\eta$.
Finally, if $b_t<0$, then $b_t<S_t$ and $g_t=-(1-\alpha)$, so $\eta \geq b_{t+1} \geq b_t \geq - D (t-1)^q -\eta$.
In all cases, $|b_{t+1}|\leq D t^q + \eta$, completing the induction.

Substituting $b_{T+1}^2\leq (D T^q+\eta)^2$ into the identity yields the claim.
\end{proof}

Using this improved guarantee instead of Lemma~\ref{lemma:bound_reward} in the proof of Theorem~\ref{thm:finite-time-bound}, we have that
\[
\frac{1}{T} \abs{\sum_{t=1}^{T} g_t}
\leq \frac{1}{T}\left(\frac{D T^q}{\eta}+1\right)~.
\]

\section{Asymptotic Coverage from Sublinear Regret}
\label{sec:asymptotic}

Now, we show an asymptotic coverage result, for any algorithm whose linearized regret grows sufficiently slowly, as a function of time $t$ and at the radius $|r| = t^p$, when the growth rate of the scores $S_t$ is not too rapid.
\begin{theorem}[No-linearized-regret implies coverage]
\label{thm:asymptotic-coverage}
Consider an online learning algorithm that in each round $t=1, 2, \ldots$
produces $b_t \in \R$. Assume that for all rounds $T\geq 1$ there exists
a function $F_T:[0,\infty) \to [0,\infty)$ such that  the
linearized regret from~\eqref{linreg} on the pinball losses $\seq{\ell_t}_{1 \leq t \leq T}$ is bounded as
\begin{equation} \label{eq:regret-bound}
    \LinRegret_{T}(r) \leq F_{T}(\abs{r}), \quad \forall r \in \R~.
\end{equation}
Moreover, assume that the nonconformity scores have a bounded growth:
for some $D>0$, $q\geq 0$, we have
$0\leq S_t \leq D t^q$ for all $t \ge 1$.
Finally, assume that there exists $q < p < 1$ such that\footnote{These two conditions can be summarized into $\max\{\mathrm{LinRegret}_t(t^p),\mathrm{LinRegret}_t(-t^p)\}=o(t)$.}
\begin{equation} \label{eq:sublinear-regret}
\lim_{t \to \infty} \frac{F_t(t^{p})}{t}
= 0~.
\end{equation}
Then, the algorithm satisfies the long-term coverage guarantee~\eqref{eq:miscoverage-convergence}.
\end{theorem}
\begin{proof}
Rearranging~\eqref{eq:regret-bound} gives, for all real $r$
\[
-\sum_{t=1}^{T} g_t b_t \ge \left(- \sum_{t=1}^{T} g_t\right) r - F_{T}(\abs{r})~.
\]
Hence, by Lemma~\ref{lemma:bound_reward}, we have
\begin{equation} \label{eq:duality-regret}
\left(- \sum_{t=1}^{T} g_t\right) r - F_{T}(\abs{r}) \leq (1 - \alpha) \sum_{t=1}^{T} S_t \leq \frac{(1 - \alpha) D (T+1)^{q+1}}{q+1}~.
\end{equation}
Now, set $r = r_T = -\sign \left(\sum_{t=1}^{T} g_t\right) \cdot T^{p}$, where we define $\sign(0)=0$. Invoking~\eqref{eq:sublinear-regret} yields
\begin{align}
\frac{(1 - \alpha) D (T+1)^{q+1}}{q+1}
\ge  T^{p} \abs{\sum_{t=1}^{T} g_t} -F_{T}\left(T^{p}\right),
\end{align}
which implies that
\[
\frac{1}{T} \abs{\sum_{t=1}^{T} g_t}
\leq \frac{F_{T}\left(T^{p}\right)}{T^{p+1}} + \frac{(1 - \alpha) D (T+1)^{q+1} }{(q+1) T^{p+1}} ~.
\]
Now, taking the lim sup of the first term on the r.h.s., we have
\[
\lim\sup_{T\to \infty} \ \frac{F_{T}\left(T^{p}\right)}{T^{p+1}}
= \lim\sup_{T\to \infty} \ \frac{F_{T}\left(T^{p}\right)}{T} T^{-p}
= 0~.
\]
For the second term, given that $q<p$, we have that
\[
\lim\sup_{T\to \infty} \ \frac{(1 - \alpha) D (T+1)^{q+1}}{(q+1) T^{p+1}} = 0~.
\]

We conclude by the equivalence~\eqref{eq:miscoverage-gradients-equality}.
\end{proof}

This result shows that for conformal prediction, the optimal Universal Portfolio strategy is computationally efficient, requiring only $\mathcal{O}(1)$ time per step. Alternatively,
using a uniform prior yields the Laplace rule
\[
    \frac{\sum_{i=1}^{t-1} \bone\{g_i = -(1-\alpha)\} + 1}{t + 1},
\]
which is known to have slightly higher regret~\citep[Theorem~13.1]{orabona2019modern}.

\section{The Average Size in Stochastic Settings}
\label{sec:stochastic}

While the previous sections established coverage guarantees under adversarial settings, characterizing the \emph{efficiency}, i.e., the size of the prediction sets relative to the optimum, seems to require additional assumptions on the data generating process.
In a fully adversarial setting, a globally optimal interval width is ill-defined, as there is no guarantee that the $Y_t$ share any common behaviour over the course of the $T$ rounds.

Therefore, we consider a stochastic setting where the nonconformity scores $S_t$
are sampled i.i.d. from a fixed distribution $\sD$.
In this setting, the optimal fixed prediction interval of the form $[0,b^\star]$
is achieved when $b^\star$ is any $(1 - \alpha)$-th quantile of the distribution.
Our goal is to show that regret-minimizing algorithms do not merely satisfy coverage constraints, but also converge to this optimal radius $b^\star$ at a rate characterized by their regret.

%Let $\ell^{(1-\alpha)}(b, s)$ denote the instantaneous pinball loss.
We define the expected loss $L(b)$ with respect to $\sD$ as
$L(b)
    := \E_{S \sim \sD}\left[ \ell^{(1-\alpha)}(b, S) \right]$.
We denote the optimal radius as $b^\star := \argmin_{b \geq 0} L(b)$; our conditions below will ensure that this is uniquely defined.
To rigorously connect the regret $R_T$ to the convergence of the distance $\abs{b_t - b^\star}$, we require mild regularity conditions on $\sD$.

\begin{assumption}[Regularity of $\sD$]
\label{ass:regularity}
The nonconformity scores $S_1, S_2, \dots$ are i.i.d.~draws from a distribution $\sD$ supported on a compact interval $\sB \subseteq [0,\infty)$, such
 that:
\begin{itemize}
    \item $\sD$ admits a Probability Density Function (PDF) $\phi$ and a Cumulative Distribution Function (CDF) $\tilde\Phi$;
    \item The density $\phi$ is uniformly bounded from below, so there is 
    a constant $\kappa > 0$ such that
    \(
        \phi(b) \geq \kappa\), \(\forall b \in \sB~.
    \)
\end{itemize}
\end{assumption}
In the above notation,  $b^\star := \tilde\Phi^{-1}(1-\alpha)$.
\begin{remark}
Clearly, the continuity of $\phi$ is not necessary for the existence of the derivatives of $L$. The expected loss $L(b) = \E[\ell^{(1-\alpha)}(b, S)]$ 
is a convolution of the continuous pinball loss with the measure $\sD$. 
This  acts as a smoothing operator: as long as $\sD$ has no point masses, $L(b)$ is continuously differentiable~\citep[see, e.g.,][]{joshi2025conformal}. Assumption~\ref{ass:regularity}.2 is imposed primarily to ensure \emph{strong convexity} (positive curvature), which allows us to convert the regret bound into a variance bound.
\end{remark}

Under these conditions, we prove that a sublinear regret implies the convergence of the average radius $\bar{b}_T = \frac{1}{T}\sum_{t=1}^T b_t$ to the optimal oracle radius $b^\star$.

\begin{theorem}[Width Convergence via Regret]
\label{thm:width-convergence}
Let an online algorithm generate radii $b_1, b_2, \dots, b_T$ with $b_t \in \sB$ for all $t$ such that the expected regret is bounded by $R_T$:
\begin{equation}
\E\left[\sum_{t=1}^T \ell^{(1-\alpha)}(b_t, S_t) - \sum_{t=1}^T \ell^{(1-\alpha)}(b^\star, S_t)\right]
\le R_T~.
\end{equation}
Under Assumption~\ref{ass:regularity}, the squared distance between the average radius $\bar{b}_T = \frac{1}{T}\sum_{t=1}^T b_t$ and the true $(1-\alpha)$-quantile is bounded by
\(
\E\left[ \left( \bar{b}_T - b^\star \right)^2 \right] \le \frac{2}{\kappa} \cdot \frac{R_T}{T}~.
\)
\end{theorem}
\begin{proof}
The proof proceeds in
two parts:
a first- and second-order analysis
of the expected loss,
 and applying an online-to-batch conversion.

\textbf{First and second-order analysis.}
This part is a review of well-known results, which were used either explicitly or implicitly to various degrees in a number of prior works,
  see e.g., \citet{gibbs2021adaptive,joshi2025conformal}, etc; which we include here only for the sake of being self-contained.
Recall the $(1-\alpha)$-pinball loss: $\ell^{(1-\alpha)}(b, S) = \max\{(1-\alpha)(S-b), \alpha(b-S)\}$. The expected loss is given by
\[
L(b)
= \int_{0}^{b} \! \alpha(b-s)\phi(s) \, \mathrm{d}s + \int_{b}^{\infty} (1-\alpha)(s-b)\phi(s) \, \mathrm{d}s~.
\]
Differentiating with respect to $b$ yields
\[
L'(b)
= \alpha \tilde\Phi(b) - (1-\alpha)(1 - \tilde\Phi(b))
= \tilde\Phi(b) - (1-\alpha)~.
\]
Setting $L'(b) = 0$, we confirm that the minimizer is unique and satisfies $\tilde\Phi(b^\star) = 1-\alpha$. Thus, $b^\star$ is exactly the $(1-\alpha)$-quantile of $\sD$.
Differentiating $L'(b)$ again, we obtain the Hessian of the expected loss:
\(
L''(b)
= \frac{d}{db}(\tilde\Phi(b) - 1 + \alpha)
= \phi(b)~.
\)
By Assumption \ref{ass:regularity}, we have $L''(b)=\phi(b)\ge \kappa$ for all $b\in\sB$, hence $L$ is $\kappa$-strongly convex on $\sB$. In particular,
\begin{equation}\label{eq:strong_convexity}
L(b) - L(b^\star) \ge \frac{\kappa}{2}(b-b^\star)^2, \quad \forall b \in \sB~.
\end{equation}

% By Assumption \ref{ass:regularity}, we have $L''(b) = \phi(b) \geq \kappa$ uniformly in the neighborhood of $b^\star$. This implies that $L(b)$ is locally $\kappa$-strongly convex. A standard property of strongly convex functions states that the sub-optimality of the loss bounds the squared distance to the minimizer:
% \begin{equation}
% \label{eq:strong_convexity}
% L(b) - L(b^\star)
% \geq \frac{\kappa}{2} (b - b^\star)^2, \quad \forall b \geq 0~.
% \end{equation}

\textbf{Online-to-batch conversion.}
Since the loss function $L(b)$ is convex, Jensen's inequality implies $L(\bar{b}_T) \le \frac{1}{T}\sum_{t=1}^T L(b_t)$. Using standard online-to-batch conversion results \citep[see, e.g.,][Theorem 3.1]{orabona2019modern}, the average regret upper bounds the excess risk:
\[
\E[L(\bar{b}_T)] - L(b^\star)
\le \E\left[ \frac{1}{T}\sum_{t=1}^T L(b_t) - L(b^\star) \right]
\le \frac{R_T}{T}~.
\]
Combining this with the strong convexity bound in~\eqref{eq:strong_convexity}, we have
\[
\frac{\kappa}{2} \E\left[ (\bar{b}_T - b^\star)^2 \right]
\le \E[L(\bar{b}_T)] - L(b^\star)
\le \frac{R_T}{T}~.
\]
Rearranging the terms yields the claim.
\end{proof}

Theorem~\ref{thm:width-convergence} implies that \emph{any} algorithm minimizing pinball loss regret \emph{automatically} converges to the statistically efficient oracle width $b^\star$. %Thus, 
%rather than designing specialized algorithms to optimize length directly, we show that efficiency is a consequence of regret minimization
%for an appropriate loss function.
These 
results are different in nature from those of \citet{srinivas2026online}.
\citet{srinivas2026online} frames the problem as direct length optimization, 
aiming to compete with the best fixed interval length in hindsight. In contrast, our framework specifically leverages the pinball loss as a proper scoring rule \cite{gneiting2007strictly}.
%\end{remark}
% .
% While they also study the efficiency of online conformal prediction, 
% the two approaches have no overlap. 

Recent work by \citet{arecesonline} also investigates the efficiency of online conformal prediction in stochastic settings. Their main result (Theorem 6.1) establishes the convergence of the \emph{last iterate} to the optimal parameters. However, obtaining this guarantee requires a specific decaying schedule for the learning rate ($\eta_t \propto t^{-c}$), which creates an explicit trade-off: faster decay improves efficiency but renders the adversarial coverage bounds vacuous \citep[Section 6]{arecesonline}. In contrast, our width convergence result (Theorem~\ref{thm:width-convergence}) is more general, showing that efficiency is an automatic consequence of regret minimization for \emph{any} algorithm, regardless of its specific update rule. Furthermore, because UP-OCP is parameter-free, it naturally achieves this efficiency through low regret without requiring manual step-size tuning or sacrificing robust coverage guarantees under adversarial distribution shifts.

\section{Proof of Theorem~\ref{thm:regret_up_ocp}}
\label{sec:appendixC}

This appendix details the reduction used in Section~\ref{sec:universal-portfolio}.
The main point is that, once we express the pinball subgradients as an asymmetric coin sequence, our two-stock conformal market in Definition~\ref{def:market} is exactly a two-asset encoding of an asymmetric coin-betting game. Standard coin-betting duality then translates wealth guarantees into linearized-regret guarantees, which in turn control the pinball-loss regret.

\subsection{From pinball subgradients to an asymmetric coin}

Recall that for the pinball loss $\ell_t(b) = \ell^{(1-\alpha)}(b,S_t)$ we use the subgradient
$g_t = \bone\{b_t \ge S_t\}-(1-\alpha)$, as in~\eqref{eq:subgradients}. Hence, with our tie-breaking convention at $b_t=S_t$, we have $g_t \in \{-(1-\alpha),\alpha\}$.

It is convenient to flip signs and work with the bounded outcome of a coin:
\begin{equation}
\label{eq:coin-outcome}
c_t := -g_t \in \{-\alpha,1-\alpha\}~.
\end{equation}
Under our tie-breaking, $c_t$ takes only the two values: $c_t=1-\alpha$ if $b_t<S_t$ (miscoverage) and $c_t=-\alpha$ if $b_t\ge S_t$ (coverage).

Consider now the following asymmetric coin-betting game.
At each round $t$, a bettor chooses a \emph{signed betting fraction} $\beta_t$ and then the outcome $c_t$ is revealed. We define the wealth process as
\begin{equation}
\label{eq:coin-wealth}
W_t := W_{t-1}(1+\beta_t c_t), \qquad W_0:=1~.
\end{equation}
To ensure \emph{no bankruptcy} for all outcomes $c_t \in [-\alpha,1-\alpha]$, the betting fraction must satisfy
\begin{equation}
\label{eq:v-safe}
\beta_t \in \left[-\frac{1}{1-\alpha},\frac{1}{\alpha}\right].
\end{equation}
Indeed, the condition $1+\beta_t c_t \ge 0$ for all $c_t \in \{-\alpha,1-\alpha\}$ is equivalent to simultaneously requiring $1-\alpha \beta_t \ge 0$ and $1+(1-\alpha)\beta_t \ge 0$, which is exactly~\eqref{eq:v-safe}.

We also define the \emph{bet amount}
\begin{equation}
\label{eq:bet-amount}
b_t := \beta_t W_{t-1}~.
\end{equation}
Plugging~\eqref{eq:bet-amount} into~\eqref{eq:coin-wealth} yields the additive form
\begin{equation}
\label{eq:additive-wealth}
W_t
= W_{t-1} + c_t b_t
= W_{t-1} - g_t b_t~.
\end{equation}

\subsection{The conformal market is a two-stock encoding of coin betting}

We now show that the conformal market in Definition~\ref{def:market} is precisely a two-asset representation of the asymmetric coin game above.

Using $c_t=-g_t$, Definition~\ref{def:market} can be rewritten as
\begin{equation}
\label{eq:market-returns-x}
w_{t,1} = \frac{c_t+\alpha}{\alpha}, \qquad
w_{t,2} = \frac{(1-\alpha)-c_t}{1-\alpha}~.
\end{equation}
These returns are nonnegative for every $c_t \in [-\alpha,1-\alpha]$.
Moreover, when $c_t=1-\alpha$ (miscoverage), we have $w_{t,1}=1/\alpha$ and $w_{t,2}=0$, while when $c_t=-\alpha$ (coverage), we have $w_{t,1}=0$ and $w_{t,2}=1/(1-\alpha)$.

Given a portfolio weight $\lambda_t \in [0,1]$ (fraction of capital invested in Stock~1), the wealth update in Definition~\ref{def:wealth} is
$W_t = W_{t-1}(\lambda_t w_{t,1} + (1-\lambda_t)w_{t,2})$.
The next lemma shows that this is exactly~\eqref{eq:coin-wealth} under an affine reparameterization of $\lambda_t$.

\begin{lemma}[Portfolio-to-coin equivalence]
\label{lem:portfolio-coin}
Fix $\alpha \in (0,1)$ and let $c_t \in [-\alpha,1-\alpha]$.
Define the two-stock returns by~\eqref{eq:market-returns-x}. For any $\lambda_t \in [0,1]$, define
\begin{equation}
\label{eq:lambda-to-v}
\beta_t
:= -\frac{1}{1-\alpha} + \frac{\lambda_t}{\alpha(1-\alpha)}
= \frac{\lambda_t-\alpha}{\alpha(1-\alpha)}~.
\end{equation}
Then,
\begin{equation}
\label{eq:mix-equals-coin}
\lambda_t w_{t,1} + (1-\lambda_t) w_{t,2}
= 1 + \beta_t c_t~.
\end{equation}
Consequently, the wealth recursion of Definition~\ref{def:wealth} is identical to the coin-betting recursion~\eqref{eq:coin-wealth}, and the bet amount $b_t=\beta_t W_{t-1}$ equals the mapping~\eqref{eq:radius-mapping}.
\end{lemma}
\begin{proof}
Using~\eqref{eq:market-returns-x},
\begin{align*}
\lambda_t w_{t,1} + (1-\lambda_t) w_{t,2}
&= \lambda_t\left(\frac{c_t+\alpha}{\alpha}\right) + (1-\lambda_t)\left(\frac{(1-\alpha)-c_t}{1-\alpha}\right) \\
&= \lambda_t\left(1+\frac{c_t}{\alpha}\right) + (1-\lambda_t)\left(1-\frac{c_t}{1-\alpha}\right) \\
&= 1 + c_t\left(\frac{\lambda_t}{\alpha} - \frac{1-\lambda_t}{1-\alpha}\right)~.
\end{align*}
The coefficient of $c_t$ simplifies as
$\frac{\lambda_t}{\alpha} - \frac{1-\lambda_t}{1-\alpha}
= -\frac{1}{1-\alpha} + \frac{\lambda_t}{\alpha(1-\alpha)} = \beta_t$,
which proves~\eqref{eq:mix-equals-coin}.
Substituting into the portfolio recursion gives $W_t=W_{t-1}(1+\beta_t c_t)$, and $b_t=\beta_t W_{t-1}$ is exactly~\eqref{eq:radius-mapping}.
\end{proof}

Two immediate consequences are worth recording.
First, since $\lambda_t \in [0,1]$, the induced $\beta_t$ in~\eqref{eq:lambda-to-v} always lies in the safe interval~\eqref{eq:v-safe}.
Second, the one-dimensional action $b_t$ in~\eqref{eq:radius-mapping} is simply the coin-betting bet amount associated with the portfolio choice $\lambda_t$.

\subsection{From wealth lower bounds to linearized-regret bounds}

We now connect wealth to the linearized regret on the pinball loss.
Recall that linearized regret is
$\LinRegret_T(u) = \sum_{t=1}^T g_t(b_t-u)$.
Using $c_t=-g_t$ from~\eqref{eq:coin-outcome}, this can be rewritten as
\begin{equation}
\label{eq:linreg-coin}
\LinRegret_T(u) = \sum_{t=1}^T c_t(u-b_t)~.
\end{equation}
Moreover, by telescoping~\eqref{eq:additive-wealth}, we have
\begin{equation}
\label{eq:wealth-telescope}
W_T = 1 + \sum_{t=1}^T c_t b_t~.
\end{equation}

The standard coin-betting duality is that a \emph{lower bound} on the achievable wealth as a function of the cumulative outcome sum implies an \emph{upper bound} on linearized regret via Fenchel conjugacy~\citep[see, e.g.,][]{orabona2019modern}.

\begin{lemma}[Wealth lower bound $\Rightarrow$ linearized regret bound]
\label{lem:wealth-to-regret}
Let $\Psi_T:\R \to (-\infty,+\infty]$ be a function.
Assume that an algorithm produces $\seq{b_t}_{t=1}^T$ and wealth $W_T$ satisfying
\begin{equation}
\label{eq:wealth-lower-potential}
W_T - 1 \ge \Psi_T\left(\sum_{t=1}^T c_t\right)
\end{equation}
for every sequence $\seq{c_t}_{t=1}^T \subseteq [-\alpha,1-\alpha]$.
Then, for every comparator $u \in \R$, its linearized regret satisfies
\begin{equation}
\label{eq:linreg-bound-from-psi}
\LinRegret_T(u)
\leq \Psi_T^\star(u),
\end{equation}
where $\Psi_T^\star$ is the Fenchel conjugate of $\Psi_T$.
\end{lemma}
\begin{proof}
Let $\theta_T := \sum_{t=1}^T c_t$.
By Fenchel-Young duality, for every $u$ we have $u \theta_T \leq \Psi_T(\theta_T) + \Psi_T^\star(u)$.
By the assumption~\eqref{eq:wealth-lower-potential} and the identity~\eqref{eq:wealth-telescope}, we have
$\sum_{t=1}^T c_t b_t = W_T - 1 \ge \Psi_T(\theta_T)$.
Therefore,
\[
\sum_{t=1}^T c_t(u-b_t)
= u \theta_T - \sum_{t=1}^T c_t b_t
\leq u \theta_T - \Psi_T(\theta_T)
\leq \Psi_T^\star(u),
\]
where in the last equality we used Fenchel-Young inequality.
This is exactly~\eqref{eq:linreg-bound-from-psi} using~\eqref{eq:linreg-coin}.
\end{proof}

Lemma~\ref{lem:wealth-to-regret} is the precise mathematical sense in which ``maximizing wealth'' (in the coin-betting/portfolio game) corresponds to ``minimizing linearized regret'' (in online convex optimization). The only remaining ingredient is to identify an explicit lower bound $\Psi_T$ for $W_T-1$ and calculate (an upper bound to) its Fenchel conjugate.

\subsection{Instantiating the potential via the best constant rebalanced portfolio}

We now connect the portfolio regret guarantee to a wealth lower bound of the form~\eqref{eq:wealth-lower-potential}.

Let $W_T^\star$ denote the wealth of the best constant rebalanced portfolio in the conformal market, i.e.,
$W_T^\star = \max_{\lambda \in [0,1]} \prod_{t=1}^T (\lambda w_{t,1} + (1-\lambda)w_{t,2})$.
By Lemma~\ref{lem:portfolio-coin}, this is equivalently the best constant betting fraction in the asymmetric coin game: $W_T^\star = \max_{\beta \in [-1/(1-\alpha),1/\alpha]} \prod_{t=1}^T (1+\beta c_t)$.

When we use the extreme subgradient convention (so $c_t \in \{-\alpha,1-\alpha\}$), $W_T^\star$ depends on the data only through the number of miscoverages.
Let $M_T := \sum_{t=1}^T \bone\{c_t = 1-\alpha\}$ and $C_T := T-M_T$.
Then, for a fixed $\lambda \in [0,1]$, we have
\[
W_T(\lambda)
= \left(\frac{\lambda}{\alpha}\right)^{M_T}\left(\frac{1-\lambda}{1-\alpha}\right)^{C_T}~.
\]
Maximizing over $\lambda$ yields $\lambda^\star = M_T/T$, and hence
\begin{equation}
\label{eq:WT-star-closed-form}
W_T^\star
= \left(\frac{M_T}{\alpha T}\right)^{M_T}\left(\frac{C_T}{(1-\alpha)T}\right)^{C_T}
= \exp\left(T \cdot \mathrm{KL}\left(\frac{M_T}{T}\middle\Vert \alpha\right)\right),
\end{equation}
where $\mathrm{KL}(p\Vert q)=p\log\frac{p}{q}+(1-p)\log\frac{1-p}{1-q}$ is the Bernoulli KL divergence.

Moreover, $M_T$ can be expressed directly in terms of $\theta_T = \sum_{t=1}^T c_t$:
since $c_t$ equals $1-\alpha$ on miscoverage and $-\alpha$ on coverage, we have
$\theta_T = (1-\alpha)M_T - \alpha C_T = M_T - \alpha T$, hence $M_T = \alpha T + \theta_T$.
Substituting into~\eqref{eq:WT-star-closed-form} gives an explicit function of $\theta_T$:
\begin{equation}
\label{eq:WT-star-as-function-of-G}
W_T^\star
= \exp\left(T \cdot \mathrm{KL}\left(\alpha+\frac{\theta_T}{T}\middle\Vert \alpha\right)\right)~.
\end{equation}

From the assumption that the portfolio algorithm guarantees $\log W_T \ge \log W_T^\star - \mathcal{R}_T$, exponentiating yields
\begin{equation}
\label{eq:wealth-lower-from-portfolio}
W_T
\ge \exp\left(-\mathcal{R}_T\right) W_T^\star
= \exp\left(-\mathcal{R}_T\right) \exp\left(T \cdot \mathrm{KL}\left(\alpha+\frac{\theta_T}{T}\middle\Vert \alpha\right)\right)
:= \Psi_T(\theta_T) +1 ~.
\end{equation}

Using Lemma~\ref{lem:wealth-to-regret}, we need to calculate the Fenchel conjugate of $\Psi_T$. Unfortunately, it does not have a closed form expression. Hence, we use Lemma~\ref{lemma:kl_lower_bound}, to obtain an easier lower bound:
\begin{equation}
\exp\left(T \cdot \mathrm{KL}\left(\alpha+\frac{\theta_T}{T}\middle\Vert \alpha\right)\right)
\geq \exp\left(T \cdot \frac{\theta_T^2/T^2}{2\alpha (1-\alpha) + 2/3|\theta_T|/T}\right)
\geq \min\left\{\exp\left(\frac{\theta_T^2}{4\alpha (1-\alpha)}\right),\exp\left(\frac{3}{4} |\theta_T|\right)\right\}~. \label{eq:lower_bound_wealth}
\end{equation}
Now, observe that if $f(x)=\min\{h_1(x),h_2(x)\}$, then we have
\begin{align*}
f^\star(y)
&= \sup_y x y - f(x)
= \sup_y x y - \min\{h_1(x),h_2(x)\}
= \sup_y \max\{x y - h_1(x), xy - h_2(x)\} \\
&\leq \max\{\sup_y x y - h_1(x), \sup_y x y - h_2(x)\}
= \max\{h_1^\star(y), h_2^\star(y)\}~.
\end{align*}
Hence, it suffices to find the Fenchel conjugates (or upper bounds) of the two functions in the min in the right-hand side of~\eqref{eq:lower_bound_wealth}.
This can be done immediately by \citet[Example 6.18, Lemma 6.24, and Theorems C.3 and C.4]{orabona2019modern}:
\[
\Psi^\star_T(u)
\leq \max\left\{|u| \sqrt{2T \alpha(1-\alpha) \ln (2T \alpha(1-\alpha) u^2 \exp\left(\mathcal{R}_T\right)+1)},
\frac{4}{3}|u|\left(\ln \frac{4 |u| \exp(\mathcal{R}_T)}{3}-1\right)
\right\}~.
\]

Using the regret of universal portfolio of $\frac{1}{2}\ln (\pi(T+1))$ and overapproximating completes the proof.

\section{Regret Guarantee of UP-OCP vs KT}
\label{sec:up_vs_kt}

The proof the regret guarantee of \citet{podkopaev2024adaptive} goes exactly through the same steps of the one of UP-OCP. Moreover, the KT approach can also be written as a universal portfolio algorithm, with the same prior. The only difference is the transformation of the subgradients into 2 stocks, that will change the wealth of the best constant rebalanced portfolio, $W^\star_T$.

The transformation for KT is the following one:
\[
    w_{t,1} = 1-g_t,
    \quad w_{t,2} = 1+g_t~.
\]
Given that $g_t\in \{\alpha-1, \alpha\}\subset \{-1,1\}$, we have that $w_{t,1}\geq0$ and $w_{t,2}\geq 0$.
In the coin-betting view, this corresponds to the wealth process
\[
W_t=W_{t-1}(1-g_t\beta_t),
\]
where we constrain $\beta_t \in [-1,1]$. This means that $W_t^\star = \max_{\beta \in [-1,1]} \ \prod_{t=1}^T (1-g_t \beta)$. Contrast this with the one we derived for UP-OCP:
\[
\max_{\beta \in [-1/(1-\alpha),1/\alpha]} \prod_{t=1}^T (1-\beta g_t)~.
\]
Given that $[-1,1] \subset [-1/(1-\alpha),1/\alpha]$, the wealth of the best constant rebalanced portfolio in the UP-OCP reduction is always at least the same of the KT one, but it is potentially much larger.

\subsection{Missing proofs in Section~\ref{sec:universal-portfolio}}

\begin{lemma}
\label{lemma:kl_lower_bound}
Let $p, q \in [0,1]$, then
\[
\KL(p\|q)
= p \ln \frac{p}{q}+(1-p)\ln\frac{1-p}{1-q}
\geq \frac{(p-q)^2}{2 q(1-q)+\frac{2}{3}|p-q|}~.
\]
\end{lemma}
\begin{proof}[Proof of Lemma~\ref{lemma:kl_lower_bound}]
Define $h(x)=(1+x) \ln (1+x)-x$ for $x>-1$ and extend it by continuity in $x=-1$ with $h(-1):=1$.
Also, define $\Delta=p-q \in [-1,1]$.

We have that
\[
p \ln \frac{p}{q}
=  (q+\Delta) \ln \left(1+\frac{\Delta}{q}\right)
= q \left( h\left(\frac{\Delta}{q}\right) + \frac{\Delta}{q}\right)
= q \, h\left(\frac{\Delta}{q}\right) + \Delta~.
\]
Similarly, we have
\[
(1-p) \ln \frac{1-p}{1-q}
= (1-q-\Delta) \ln \left(1-\frac{\Delta}{1-q}\right)
= (1-q)\left(1-\frac{\Delta}{1-q}\right) \ln \left(1-\frac{\Delta}{1-q}\right)
= (1-q) \, h\left(-\frac{\Delta}{1-q}\right)-\Delta~.
\]
Hence, overall we have
\[
\KL(p\|q)
= q \, h\left(\frac{\Delta}{q}\right) + (1-q) \, h\left(-\frac{\Delta}{1-q}\right)~.
\]
Observe that $\frac{\Delta}{q}\geq -1$ and $-\frac{\Delta}{1-q}\geq -1$.
Hence, we use the elementary inequality $h(x)\geq \frac{x^2}{2+\frac{2}{3}|x|}$ for $x\geq-1$, to have
\[
\KL(p\|q)
\geq \Delta^2\left(\frac{1}{2q+\frac{2}{3}|\Delta|}+\frac{1}{2(1-q)+\frac{2}{3}|\Delta|}\right)
= \Delta^2 \frac{2+\frac{4}{3}|\Delta|}{4 q(1-q)+\frac{4}{3}|\Delta|+\frac{4}{9}\Delta^2}
\geq \frac{\Delta^2}{2 q(1-q)+\frac{2}{3}|\Delta|}~.
\]
Using the value of $\Delta$ finishes the proof.
\end{proof}

\begin{proof}[Proof of Theorem~\ref{thm:up-coverage}]
Let $k:=|\{t\in[T]: g_t<0\}|$ and $\hat p:=k/T$.
Since $g_t=\alpha$ on coverage rounds and $g_t=-(1-\alpha)$ on miscoverage rounds,
\[
\sum_{t=1}^T g_t=\alpha(T-k)-(1-\alpha)k=\alpha T-k,
\qquad
\frac{1}{T}\sum_{t=1}^T g_t=\alpha-\hat p~.
\]

\paragraph{Best-constant wealth equals an exact KL term.}
For a constant betting strategy $\lambda\in[0,1]$, the wealth is
\[
W_T(\lambda)
=\left(\frac{\lambda}{\alpha}\right)^k\left(\frac{1-\lambda}{1-\alpha}\right)^{T-k}~.
\]
The maximizer is $\lambda^\star=\hat p$, and substituting yields
\begin{align}
\ln W_T^\star
&=k\ln\frac{\hat p}{\alpha}+(T-k)\ln\frac{1-\hat p}{1-\alpha} \nonumber\\
&=T\left(\hat p\ln\frac{\hat p}{\alpha}+(1-\hat p)\ln\frac{1-\hat p}{1-\alpha}\right)
=T \cdot \KL(\hat p\|\alpha)~. \label{eq:up-lnWstar-KL}
\end{align}

For two assets with Jeffreys prior, the Universal Portfolio wealth is the Beta$(1/2,1/2)$ mixture, whose regret is known~\citep{cover2002universal}:
\begin{equation}
\label{eq:up-regret-constant}
\ln W_T\ge \ln W_T^\star-\frac12\ln(\pi(T+1))~.
\end{equation}
Moreover, by Lemma~\ref{lemma:bound_reward} and the assumption that $S_t\leq D t^q$, we have that
\[
W_T
= 1 - \sum_{t=1}^T b_t g_t
\leq 1 + \frac{(1-\alpha) D}{q+1} (T+1)^{q+1}~.
\]

Combining~\eqref{eq:up-lnWstar-KL} and~\eqref{eq:up-regret-constant} gives
\begin{equation}
\label{eq:up-KL-upper}
\KL(\hat p\|\alpha)\leq \varepsilon_T
:=\frac{1}{T}\left[\ln\left(1+\frac{(1-\alpha)D (T+1)^{q+1}}{q+1}\right)+\frac12\ln(\pi(T+1))\right]~.
\end{equation}

\paragraph{Explicit inversion from KL to miscoverage deviation.}
Let $\delta:=|\hat p-\alpha|$.
From Lemma~\ref{lemma:kl_lower_bound}, we have
\begin{equation}
\label{eq:kl-bernstein}
\KL(a\|b)\ge \frac{(a-b)^2}{2b(1-b)+\frac23|a-b|}~.
\end{equation}
Applying~\eqref{eq:kl-bernstein} with $a=\hat p$ and $b=\alpha$ and using~\eqref{eq:up-KL-upper} yields
\[
\varepsilon_T\ge \frac{\delta^2}{2\alpha(1-\alpha)+\frac23 \delta},
\]
equivalently
\[
\delta^2-\frac{2\varepsilon_T}{3}\delta-2\alpha(1-\alpha)\varepsilon_T
\leq 0~.
\]
Solving this quadratic inequality for the nonnegative root gives
\[
\delta
\leq \frac{\varepsilon_T}{3}+\sqrt{2\alpha(1-\alpha)\varepsilon_T+\frac{\varepsilon_T^2}{9}}
\leq \varepsilon_T+\sqrt{2\alpha(1-\alpha)\varepsilon_T}~.
\]
Since $\delta=|\hat p-\alpha|=\bigl|\frac{1}{T}\sum_{t=1}^T g_t\bigr|$, this proves the stated bound.
\end{proof}

\section{Closed-Form UP Update}
\label{sec:closed_form}

\begin{theorem}[Closed-Form Update (Jeffreys prior)]\label{thm:closed-form}
For the universal portfolio update~\eqref{eq:up-update} with the Jeffreys prior on $\Delta=[0,1]$ (equivalently a $\mathrm{Beta}(1/2,1/2)$ prior in the two-asset case), the weights admit the closed-form update
\begin{equation} \label{lamo}
\lambda_t
= \frac{1}{t}\left(\sum_{i=1}^{t-1} \bone\set{g_i = -(1-\alpha)} +\frac12\right)~.
\end{equation}
\end{theorem}
\begin{proof}
In the conformal market of Definition~\ref{def:market}, at round $t$ the two synthetic stocks have gross returns
\[
(w_{t,1},w_{t,2})
=\left(1-\frac{g_t}{\alpha},\,1+\frac{g_t}{1-\alpha}\right)~.
\]
A constant-rebalanced portfolio that invests fraction $\lambda\in[0,1]$ in stock 1 and $1-\lambda$ in stock 2 achieves one-step gross return
$\lambda\,w_{t,1}+(1-\lambda)\,w_{t,2}$.
Hence, the universal portfolio weight is
\begin{align*}
\lambda_{T+1}
&=\frac{1}{K}\int_0^1 \! \lambda\prod_{t=1}^T\left[\lambda\left(1-\frac{g_t}{\alpha}\right)+(1-\lambda)\left(1+\frac{g_t}{1-\alpha}\right)\right]\mu(\lambda)\,\mathrm{d}\lambda,
\end{align*}
where
\[
K
=\int_0^1 \! \prod_{t=1}^T\left[\lambda\left(1-\frac{g_t}{\alpha}\right)+(1-\lambda)\left(1+\frac{g_t}{1-\alpha}\right)\right]\mu(\lambda)\,\mathrm{d}\lambda~.
\]

Now use that $g_t\in\{-(1-\alpha),\alpha\}$.
Let $a$ be the number of rounds with $g_t=-(1-\alpha)$ (miscoverage), so $T-a$ is the number of rounds with $g_t=\alpha$.
For a miscoverage round $g_t=-(1-\alpha)$, we have
$1-g_t/\alpha=1+(1-\alpha)/\alpha=1/\alpha$ and $1+g_t/(1-\alpha)=1-1=0$, hence
\[
\lambda\left(1-\frac{g_t}{\alpha}\right)+(1-\lambda)\left(1+\frac{g_t}{1-\alpha}\right)
=\frac{\lambda}{\alpha}~.
\]
For a coverage round $g_t=\alpha$, we have
$1-g_t/\alpha=0$ and $1+g_t/(1-\alpha)=1+\alpha/(1-\alpha)=1/(1-\alpha)$, hence
\[
\lambda\left(1-\frac{g_t}{\alpha}\right)+(1-\lambda)\left(1+\frac{g_t}{1-\alpha}\right)
=\frac{1-\lambda}{1-\alpha}~.
\]
Therefore,
\[
\prod_{t=1}^T\left[\lambda\left(1-\frac{g_t}{\alpha}\right)+(1-\lambda)\left(1+\frac{g_t}{1-\alpha}\right)\right]
=\left(\frac{\lambda}{\alpha}\right)^a\left(\frac{1-\lambda}{1-\alpha}\right)^{T-a}~.
\]
Plugging this into numerator and denominator gives
\begin{align*}
\lambda_{T+1}
&=\frac{\int_0^1 \! \lambda\left(\frac{\lambda}{\alpha}\right)^a\left(\frac{1-\lambda}{1-\alpha}\right)^{T-a}\mu(\lambda)\,\mathrm{d}\lambda}
{\int_0^1 \! \left(\frac{\lambda}{\alpha}\right)^a\left(\frac{1-\lambda}{1-\alpha}\right)^{T-a}\mu(\lambda)\,\mathrm{d}\lambda}\\
&=\frac{\int_0^1 \! \lambda^{a+1}(1-\lambda)^{T-a}\mu(\lambda)\,\mathrm{d}\lambda}{\int_0^1 \! \lambda^{a}(1-\lambda)^{T-a}\mu(\lambda)\,\mathrm{d}\lambda},
\end{align*}
since the constant factor $\alpha^{-a}(1-\alpha)^{a-T}$ cancels.

Now, let's specialize it to the $\mathrm{Dirichlet}(1/2,1/2)$ prior, whose marginal density is $\mu(\lambda)=\frac{\Gamma(1)}{\Gamma(1/2)^2}\lambda^{-1/2}(1-\lambda)^{-1/2}$ on $(0,1)$.
Define, for $p,q>-1/2$,
\[
I(p,q)
=\int_0^1 \! \lambda^p(1-\lambda)^q\mu(\lambda)\,\mathrm{d}\lambda~.
\]
Then,
\begin{align*}
I(p,q)
&=\frac{\Gamma(1)}{\Gamma(1/2)^2}\int_0^1 \! \lambda^{p-1/2}(1-\lambda)^{q-1/2}\,\mathrm{d}\lambda
=\frac{\Gamma(1)}{\Gamma(1/2)^2}\cdot \frac{\Gamma(p+1/2)\Gamma(q+1/2)}{\Gamma(p+q+1)}~.
\end{align*}
Using $\Gamma(u+1)=u\Gamma(u)$,
\[
\frac{I(p+1,q)}{I(p,q)}
=\frac{\Gamma(p+3/2)}{\Gamma(p+1/2)}\cdot\frac{\Gamma(p+q+1)}{\Gamma(p+q+2)}
=\frac{p+1/2}{p+q+1}~.
\]
Finally, taking $p=a$ and $q=T-a$ yields
\[
\lambda_{T+1}
=\frac{I(a+1,T-a)}{I(a,T-a)}=\frac{a+1/2}{T+1}~.
\]

For a uniform distribution as the prior, a similar derivation yields
\(
    \lambda_{T+1} = \frac{a + 1}{T + 2}~.
\)
\end{proof}

\section{Baseline Update Rules}
\label{app:baselines}

\subsection{Krichevsky-Trofimov}
The Krichevsky-Trofimov (KT) bettor is a parameter-free approach for adaptive conformal inference that addresses the sensitivity of traditional methods to learning rate tuning~\citep{podkopaev2024adaptive}. By framing the selection of the conformal radius $b_t$ as a coin-betting game, the algorithm avoids the need for a manually specified learning rate. In this framework, the algorithm starts with an initial wealth $W_0 = 1$ and places bets on the outcome of coins $c_t \in [-1, 1]$, which are defined as the negative subgradients of the pinball loss: $c_t = -g_t$. The radius at each step $t$ is determined by the betting fraction $\beta_t$ and the current wealth:
\begin{equation}
    b_t = \beta_t W_{t-1}~.
\end{equation}
The wealth is updated recursively based on the bet's success:
\begin{equation}
    W_t = W_{t-1} + b_t c_t = W_{t-1} - g_t b_t~.
\end{equation}
The KT estimator provides a practical betting scheme that adapts to the observed sequence of subgradients by updating the betting fraction as follows:
\begin{equation}
    \beta_{t+1} = \frac{t}{t+1} \beta_t - \frac{1}{t+1} g_t~.
\end{equation}
This strategy is proven to control long-term miscoverage frequency at the nominal level $\alpha$ provided the nonconformity scores are bounded~\citep{podkopaev2024adaptive}.

\begin{algorithm}[H]
  \caption{KT-based Adaptive Conformal Predictor}
  \label{alg:KT}
  \begin{algorithmic}
    \STATE {\bfseries Initialize:} $\alpha \in (0, 1)$, $W_0 = 1$, $\beta_1 = 0$, $b_1 = 0$.
    \FOR{$t = 1, 2, \dots$}
    \STATE Produce a forecast $\hat{Y}_t$ and output a set: $\hat{C}_t = [\hat{Y}_t - b_t, \hat{Y}_t + b_t]$
    \STATE Observe $Y_t$ and compute error: $S_t = \abs{Y_t - \hat{Y}_t}$
    \STATE Compute $g_t \in \partial \ell^{(1-\alpha)}(b, S_t)|_{b = b_t}$ as per \eqref{eq:subgradients};
    \STATE Set $W_t = W_{t-1} - g_t b_t$
    \STATE Set $\beta_{t+1} = \frac{t}{t+1} \beta_t - \frac{1}{t+1} g_t$
    \STATE Set $b_{t+1} = \beta_{t+1} W_t$
    \ENDFOR
  \end{algorithmic}
\end{algorithm}

\subsection{Dynamically-tuned Adaptive Conformal Inference}
Dynamically-tuned Adaptive Conformal Inference (DtACI) is an extension of ACI designed to eliminate the sensitivity to the fixed step-size parameter $\gamma$, which governs the adaptation rate to distribution shifts~\citep{gibbs2024conformal}. Instead of relying on a single $\gamma$, DtACI maintains a set of expert ACI instances running in parallel, each using a different step size $\gamma_k$ from a candidate grid. At each time step $t$, the algorithm aggregates the predictions of these experts using an online learning procedure (e.g., exponentially weighted average) to produce a robust conformal radius $b_t$.
The aggregation relies on the empirical quantiles of the current score, defined as $\beta_t := \hat{F}_{t-1}(S_t)$, where $\hat{F}_{t-1}$ is the empirical distribution of the past nonconformity scores. The experts are weighted based on their performance with respect to the pinball loss evaluated at these levels. Specifically, the loss for an expert proposing target level $\alpha_t^i$ is given by:
\[
    \ell(\beta_t, \alpha_t^i) = 
    \begin{cases} 
        \alpha(\beta_t - \alpha_t^i) & \text{if } \beta_t \ge \alpha_t^i \\
        (\alpha - 1)(\beta_t - \alpha_t^i) & \text{if } \beta_t < \alpha_t^i.
    \end{cases}
\]
\begin{algorithm}[H]
  \caption{Dynamically-tuned Adaptive Conformal Inference}
  \label{alg:DtACI}
  \begin{algorithmic}
    \STATE {\bfseries Input:} Observed values $\set{\beta_t}_{1 \leq t \leq T}$, set of candidate $\gamma$ values $\set{\gamma_i}_{1 \leq i \leq k}$, starting points $\set{\alpha_1^i}_{1 \leq i \leq k}$, and parameters $\sigma$ and $\eta$.
    \STATE {\bfseries Initialize:} $w_1^i = 1$, $1 \leq i \leq k$.
    \FOR{$t=1, 2, \dots, T$}
    \STATE Define the probabilities $p_t^i = w_t^i / \sum_{1 \leq j \leq k} w_t^j$, $\forall 1 \leq i \leq k$
    \STATE Output $\overline{\alpha}_t = \sum_{1 \leq i \leq k} p_t^i \alpha_t^i$
    \STATE $\overline{w}_t^i = w_t^i\exp(-\eta \ell(\beta_t, \alpha_t^i))$, $\forall 1 \leq i \leq k$
    \STATE $\overline{W}_t = \sum_{1 \leq i \leq k} \overline{w}_t^i$
    \STATE $w_{t+1}^i = (1 - \sigma) \overline{w}_t^i + \overline{W}_t \sigma / k$
    \STATE $err_t^i = \bone\set{Y_t \notin \hat{C}_t(\alpha_t^i)}$, $\forall 1 \leq i \leq k$
    \STATE $err_t = \bone\set{Y_t \notin \hat{C}_t(\overline{\alpha}_t)}$
    \STATE $\alpha_{t+1}^i = \alpha_t^i + \gamma_i(\alpha - err_t^i)$, $\forall 1 \leq i \leq k$
    \ENDFOR
  \end{algorithmic}
\end{algorithm}

The default hyperparameters for DtACI follow the experimental setup detailed in Section 4 of \citet{gibbs2024conformal}. The candidate step sizes are set to be on a logarithmic grid $\Gamma = \{0.001 \cdot 2^k\}_{k=0}^7$, ranging from $0.001$ to $0.128$, designed to cover a spectrum of regimes from stable to highly reactive. The fixed share parameter is set to $\sigma = 0.001$, and the learning rate $\eta$ for the expert algorithm is chosen to be $e$. The starting points $\set{\alpha_1^i}_{1 \leq i \leq k}$ are all initialized to the nominal miscoverage level $\alpha$.

\subsection{Scale-Free Online Gradient Descent}
As mentioned in the main text, standard Online Gradient Descent (OGD) is sensitive to the scale of the nonconformity scores, often requiring careful tuning of the learning rate to ensure reasonable performance. To address this, \citet{bhatnagar2023improved} adapted Scale-Free Online Gradient Descent (SF-OGD) \citep{orabona2018scale} for conformal prediction. SF-OGD dynamically adjusts the step size by normalizing the current gradient by the Euclidean norm of the history of past gradients. This normalization makes the algorithm robust to the magnitude of the scores without requiring prior knowledge of their bounds. The update rule for the conformal radius $b_t$ is given by:
\begin{equation}
    b_{t+1} = b_t - \eta \frac{g_t}{\sqrt{\epsilon + \sum_{i=1}^t g_i^2}} ~,
\end{equation}
where $g_t$ is the subgradient of the pinball loss, $\eta$ is a scalar multiplier controlling the learning rate, and $\epsilon$ is a small constant (e.g., $10^{-6}$) to ensure numerical stability.

\begin{algorithm}[H]
  \caption{SF-OGD Adaptive Conformal Predictor}
  \label{alg:sfogd}
  \begin{algorithmic}
    \STATE {\bfseries Input:} Target miscoverage $\alpha \in (0,1)$, learning rate $\eta > 0$.
    \STATE {\bfseries Initialize:} Radius $b_1 = 0$, sum of squared gradients $G_0 = 0$, $\epsilon=10^{-6}$.
    \FOR{$t = 1, 2, \dots$}
    \STATE Output set $\hat{C}_t = [\hat{Y}_t - b_t, \hat{Y}_t + b_t]$
    \STATE Observe $Y_t$ and compute score $S_t = \abs{Y_t - \hat{Y}_t}$
    \STATE Compute gradient $g_t \in \partial \ell^{(1-\alpha)}(b, S_t)|_{b = b_t}$ as per \eqref{eq:subgradients};
    \STATE Update sum of squares: $G_t = G_{t-1} + g_t^2$
    \STATE Update radius: $b_{t+1} = b_t - \eta \frac{g_t}{\sqrt{G_t + \epsilon}}$
    \ENDFOR
  \end{algorithmic}
\end{algorithm}

While robust to scale, SF-OGD still requires the selection of the global learning rate $\eta$. In our experiments, we select $\eta$ via grid search from the set $\set{0.01, 0.1, 0.25, 1, 10, 25, 100}$. These candidate values are adopted from the well-tuned choices reported in \citet{podkopaev2024adaptive}, ensuring the method can accommodate selected datasets.

\subsection{Conformal P/PI Control}
Conformal P/PI Control frames the selection of the conformal radius $b_t$ as a feedback control problem, where the goal is to calibrate the coverage error $err_t = \bone(y_t \notin \hat{C}_t)$ to the set point $\alpha$~\citep{angelopoulos2023conformal}. The algorithm combines two components: a Proportional (P) controller, also known as Quantile Tracking, and an Integral (I) controller. The P controller updates the radius using online gradient descent on the quantile loss, adjusting $b_t$ proportional to the instantaneous error $err_t - \alpha$. While effective, P control can suffer from steady-state error. To mitigate this, the PI controller adds an integrator term $r_t(\cdot)$ that acts on the cumulative sum of past errors, $E_t = \sum_{i=1}^t (err_i - \alpha)$. The update rule for the radius at time $t+1$ is given by combining these terms with the previous radius:
\begin{equation}
    b_{t+1} = b_t + \eta (err_t - \alpha) + r_t(E_t),
\end{equation}
where $\eta$ is the learning rate and $r_t$ is a saturation function (e.g., a tangent function) designed to stabilize coverage under arbitrary distribution shifts~\citep{angelopoulos2023conformal}.

\begin{algorithm}[H]
  \caption{Conformal P/PI Control}
  \label{alg:conformal_pi}
  \begin{algorithmic}
    \STATE {\bfseries Input:} Target miscoverage $\alpha$, learning rate $\eta > 0$, integrator function $r_t(\cdot)$.
    \STATE {\bfseries Note:} For P Control, set $r_t(x) = 0$. 
    \STATE {\bfseries Initialize:} Radius $b_1 = 0$, cumulative error $E_0 = 0$.
    \FOR{$t=1, 2, \dots$}
    \STATE Receive input $x_t$ and forecast $\hat{Y}_t$
    \STATE Output set $\hat{C}_t = [\hat{Y}_t - b_t, \hat{Y}_t + b_t]$
    \STATE Observe $Y_t$ and compute error indicator $err_t = \bone\{Y_t \notin \hat{C}_t\}$
    \STATE Update cumulative error: $E_t = E_{t-1} + (err_t - \alpha)$
    \STATE Compute P-step: $\delta_P = \eta (err_t - \alpha)$
    \STATE Compute I-step: $\delta_I = r_t(E_t)$
    \STATE Update radius: $b_{t+1} = b_t + \delta_P + \delta_I$
    \ENDFOR
  \end{algorithmic}
\end{algorithm}

The hyperparameters for the P/PI controller are set following the heuristics provided in Appendix B of \citet{angelopoulos2023conformal}. The proportional gain (learning rate) is typically set adaptively as $\eta = \lambda \hat{B}_t$, where $\hat{B}_t$ is the maximum score observed in a trailing window (or a global bound) and $\lambda \in (0, 1]$ is a scaling factor (0.1 was recommended as a good default). For the PI controller, the integrator is defined as $r_t(E) = K_I \tan(E \log(T) / (T C_{sat}))$. The constant $K_I$ aligns the integrator's output with the scale of the nonconformity scores; it is recommended to be set to a hypothesized upper bound on the scores (e.g., $K_I \approx \max S_t$). The parameter $C_{sat}$ controls the saturation point of the integrator and is derived from a theoretical guarantee to ensure the miscoverage does not exceed a tolerance $\delta$ by time $T$ (e.g., $C_{sat} \approx \frac{2}{\pi}\log(T\delta)$). In the original implementation, both $K_I$ and $C_{sat}$ are often pre-tuned or fixed heuristically for specific datasets to ensure stability.

Specifically in our experiments, the scalar multipliers $\lambda$ are uniformly selected from the fixed grid $\set{0, 0.05, 0.1, 0.5, 1}$\footnote{In fact, 0.1 was recommended by \citet{angelopoulos2023conformal} as a good default.} $\set{0, 0.05, 0.1, 0.5, 1}$.  for all datasets. For the PI controller, the integrator parameters are pre-tuned heuristically as described previously, while the same grid of $\lambda$ values is used for the proportional component.

Note that we do not compare against the full Conformal PID framework, specifically its derivative (D) component known as \emph{Scorecasting} \citep{angelopoulos2023conformal}.
Scorecasting introduces a secondary modeling layer that fundamentally alters the score distribution by predicting and residualizing errors, effectively transforming the problem into an easier one. To ensure a fair assessment of the conformal update rules themselves and to maintain consistency with the standard evaluation established in prior literature \citep{gibbs2021adaptive, gibbs2024conformal}, we restrict our comparison to methods that adapt to the original sequence of nonconformity scores without modification.

\subsection{Trivial Predictor}
The Trivial Predictor serves as a minimal baseline designed to verify the validity of coverage metrics. It guarantees that the empirical coverage exactly matches the target level $1-\alpha$ at fixed periodic intervals, completely independent of the data distribution. The method approximates the target coverage probability as a rational fraction $K/N \approx 1-\alpha$. It then generates a deterministic, periodic sequence of prediction sets consisting solely of infinite sets ($b_t = \infty$) and empty sets ($b_t = -\infty$). By distributing the $K$ infinite sets as evenly as possible over each cycle of length $N$, the predictor ensures that the cumulative coverage error returns to exactly zero at the end of every cycle, providing perfect validity but practically useless set predictions.

\begin{algorithm}[H]
  \caption{Trivial Predictor (Deterministic Cyclic Coverage)}
  \label{alg:trivial}
  \begin{algorithmic}
    \STATE {\bfseries Input:} Target miscoverage $\alpha \in (0,1)$.
    \STATE {\bfseries Initialize:} Rational approximation $1 - \alpha \approx \frac{K}{N}$ (e.g., via continued fractions).
    \STATE {\bfseries Initialize:} Time step $t = 0$.
    \FOR{$t=1, 2, \dots$}
    \STATE Determine current radius $b_t$ based on previous update.
    \STATE Output set $\hat{C}_t = [\hat{Y}_t - b_t, \hat{Y}_t + b_t]$    
    \STATE Calculate position in cycle: $i = t \pmod N$
    \STATE Compute accumulated coverage credits:
    \STATE \quad $acc_{curr} = \lfloor \frac{i \cdot K}{N} \rfloor$
    \STATE \quad $acc_{next} = \lfloor \frac{(i + 1) \cdot K}{N} \rfloor$
    \IF{$acc_{next} > acc_{curr}$}
        \STATE Set next radius: $b_{t+1} = \infty$ \COMMENT{Output Full Set}
    \ELSE
        \STATE Set next radius: $b_{t+1} = -\infty$ \COMMENT{Output Empty Set}
    \ENDIF
    \ENDFOR
  \end{algorithmic}
\end{algorithm}

\section{Score Growth Demonstration}
\label{app:growth-demo}

In Section~\ref{sec:guarantees}, we introduced the polynomial growth assumption $S_t \le D t^q$, which relaxes the standard boundedness ($S_t \le C$) used in prior literature. Figure~\ref{fig:growth-demo} illustrates the practical implication of this relaxation. To see why this matters in practice, consider the Apple (AAPL) dataset used in our experiments. It is clear that the nonconformity scores follow a roughly quadratic trend ($q \approx 2$). Financial time series frequently exhibit volatility clustering and price drift that violate static bounds. Our analysis extends guarantees to the polynomial-growth regime, which provides a tighter fit for such non-stationary real-world data.

\begin{figure}[H]
    \centering
    \includegraphics[trim={0 1.0cm 0 1.5cm}, clip, width=0.7\columnwidth]{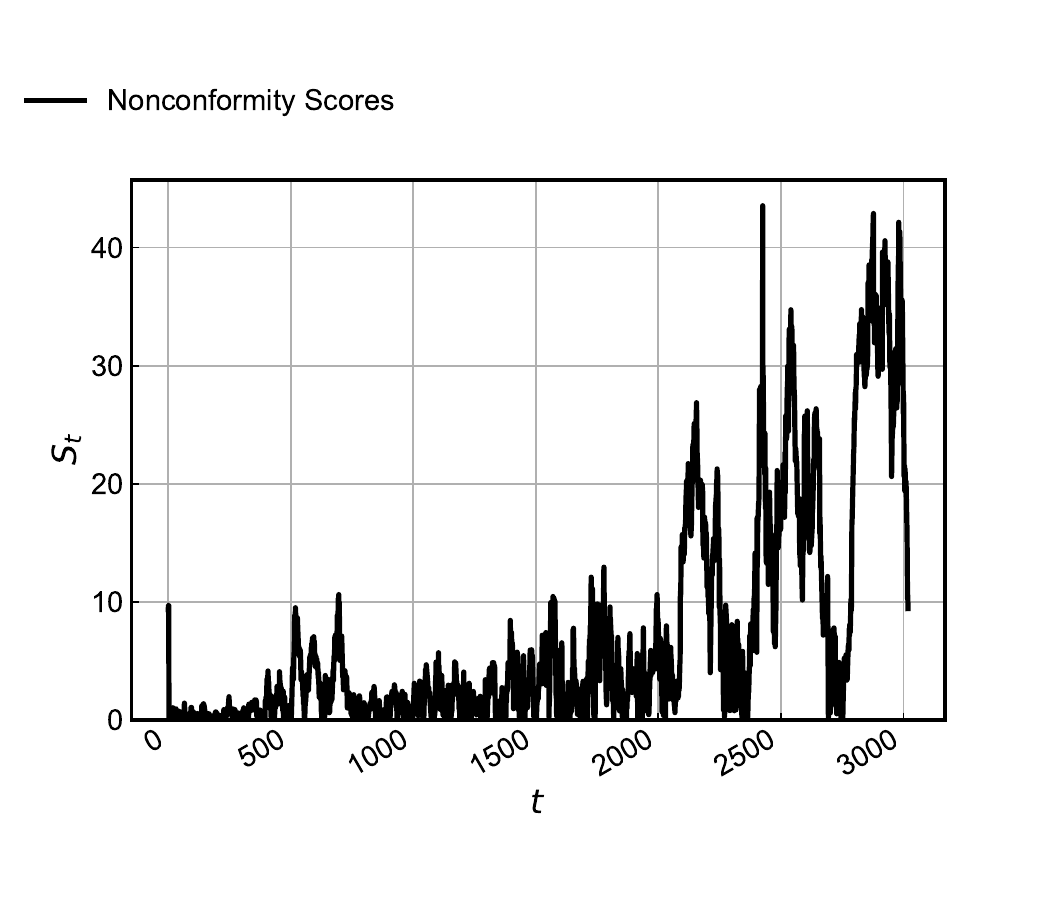}
    \caption{Observed nonconformity scores for AAPL stock returns.}
    \label{fig:growth-demo}
\end{figure}

\section{$\alpha$-Correction}
\label{app:alpha-correction}

In this section, we show how to minimally change \emph{any} OCP algorithm to obtain a possibly better tracking of the desired coverage level, without changing the Pareto frontier curves or the rate of the theoretical guarantees.

First of all, in the online setting \emph{we cannot guarantee coverage at any time step $t$, but only asymptotically}. Hence, what we propose is only a heuristic, but a theoretically principled one. The basic idea is that, from the plots, it seems that most of the algorithms tend to undercover. Hence, one can run the algorithm with a slightly inflated parameter $\alpha$, to obtain a better coverage.

Now, by how much do we inflate $\alpha$? Setting the desired coverage to $\alpha+ \frac{k}{\sqrt{T}}$, for any constant $k>0$, and assuming that the algorithm guarantees a convergence of the coverage of $\mathcal{O}(1/\sqrt{T})$ or worse, will produce exactly the same rate of convergence for the coverage. Also, the Pareto frontier will not change as well, because we are simply using a different setting of $\alpha$, so we are just moving along the Pareto frontier curve.

\section{Full Results for AXP Dataset}
\label{app:AXP-full-results}

In this section, we provide a comprehensive analysis of the American Express (AXP) dataset, complementing the efficiency and calibration results presented in Section~\ref{sec:experiments}. Note that an initial warm-up period of 100 days is used for training the initial base forecaster. Additionally, for all plots presented in this section, we discard the first 50 days of the evaluation period as a burn-in.

\textbf{Local Adaptivity.} In the main text, we noted that aggregate metrics like marginal coverage can mask significant local failures, such as error clustering or instability. To visualize these behaviors, Figures~\ref{fig:AXP-UP-vs-DtACI-local} through \ref{fig:AXP-UP-vs-PI_01-local} provide 1-vs-1 comparisons of local adaptivity between UP-OCP and all parameter-free and tuned baselines.

In each figure, the top left panel shows the local coverage computed over a rolling window of 100 days, with the dashed black line indicating the target coverage level ($1 - \alpha=95\%$). The top right panel displays the local width of the prediction intervals ($2 \max(0, b_t)$, which is $0$ when $\hat{C}_t = \emptyset$) using the same window size. The bottom panel illustrates the raw prediction sets $\hat{C}_t$ around the true observation (central black line).

We observe that UP-OCP maintains local coverage tightly around the $95\%$ target, with no significant swings. 
In contrast, KT exhibits volatility, where the local coverage drops below $75\%$ (e.g., during 2009 and 2015).

\begin{figure}[H]
  \centering
  \includegraphics[trim={0 0 0 1.2cm}, clip, width=0.85\columnwidth]{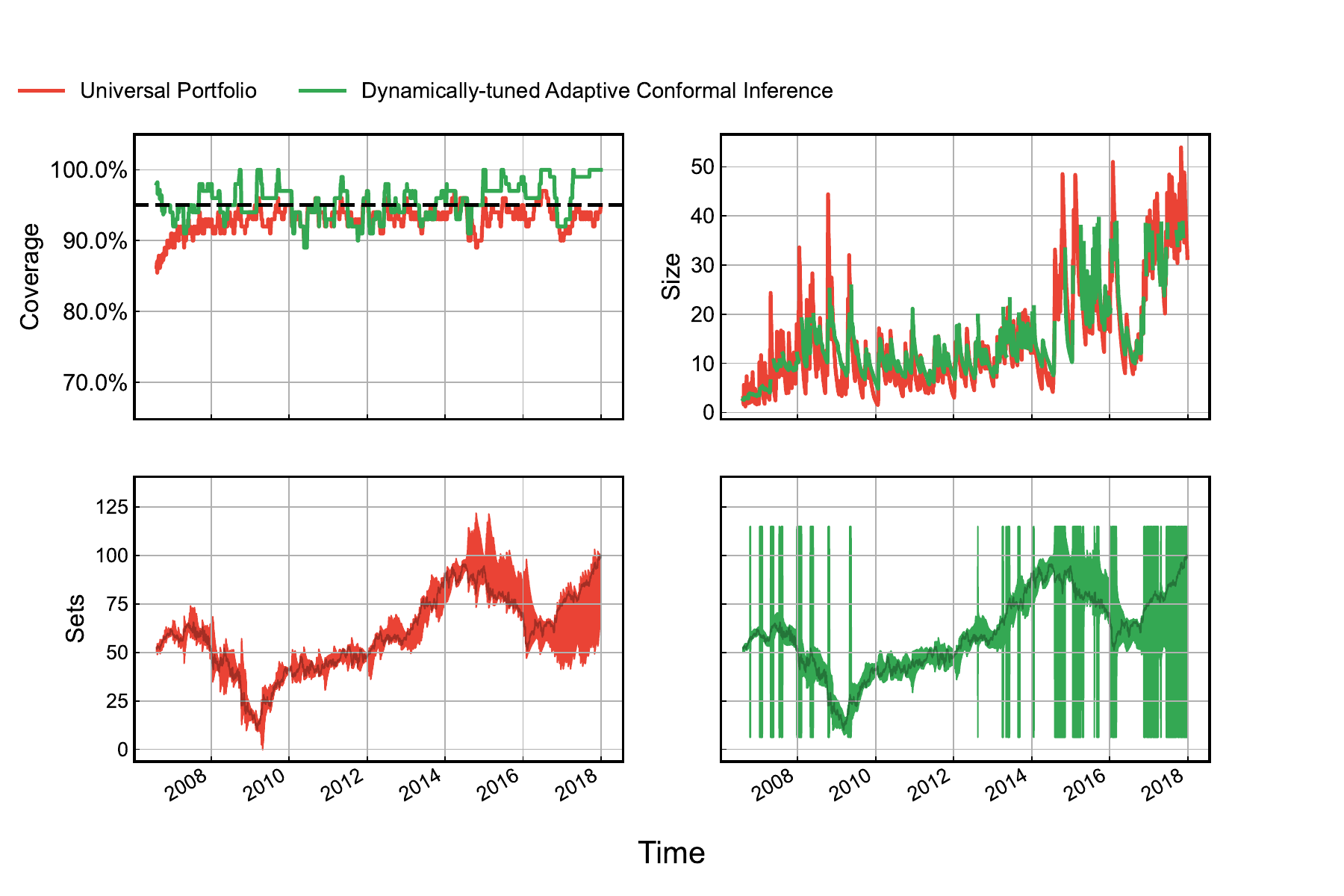}
  \caption{UP-OCP vs. DtACI for forecasting AXP stock return.}
  \label{fig:AXP-UP-vs-DtACI-local}
\end{figure}

\begin{figure}[H]
  \centering
  \includegraphics[trim={0 0 0 1.2cm}, clip, width=0.85\columnwidth]{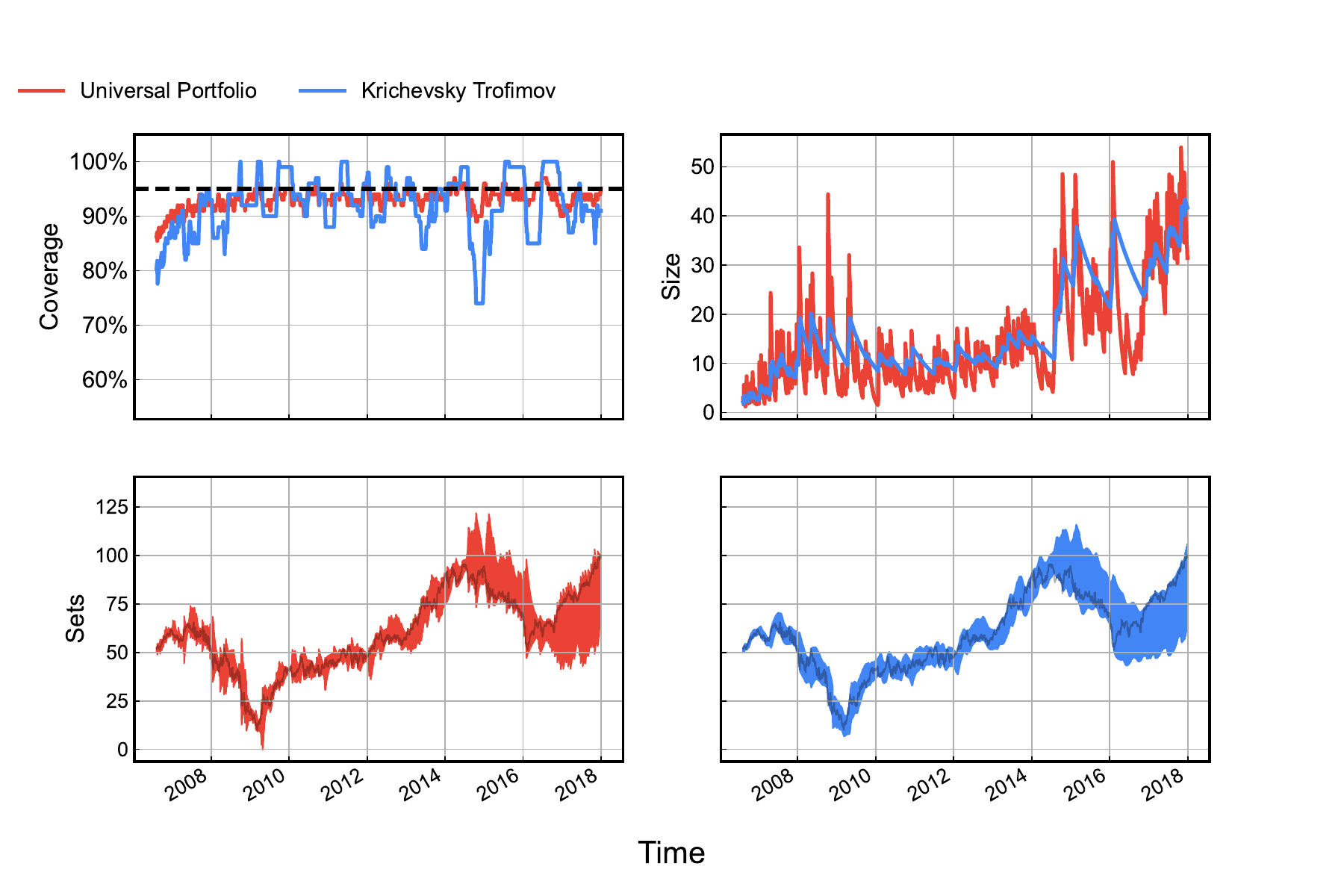}
  \caption{As in Figure~\ref{fig:AXP-UP-vs-DtACI-local}, UP-OCP vs. KT.}
  \label{fig:AXP-UP-vs-KT-local}
\end{figure} 

\begin{figure}[H]
  \centering
  \includegraphics[trim={0 0 0 1.2cm}, clip, width=0.85\columnwidth]{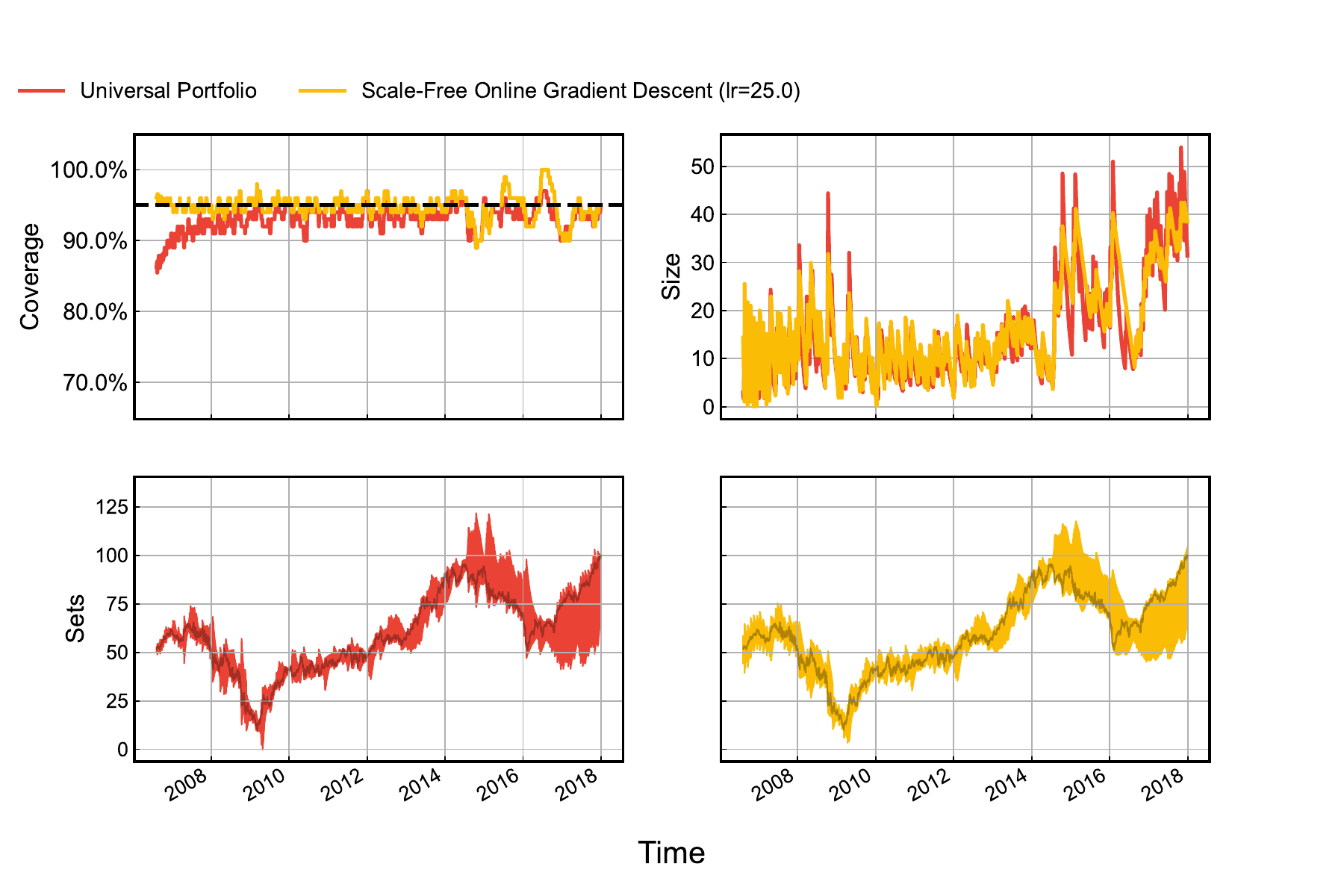}
  \caption{As in Figure~\ref{fig:AXP-UP-vs-DtACI-local}, UP-OCP vs. SFOGD (lr=25).}
  \label{fig:AXP-UP-vs-SGD-25-local}
\end{figure}

\begin{figure}[H]
  \centering
  \includegraphics[trim={0 0 0 1.2cm}, clip, width=0.85\columnwidth]{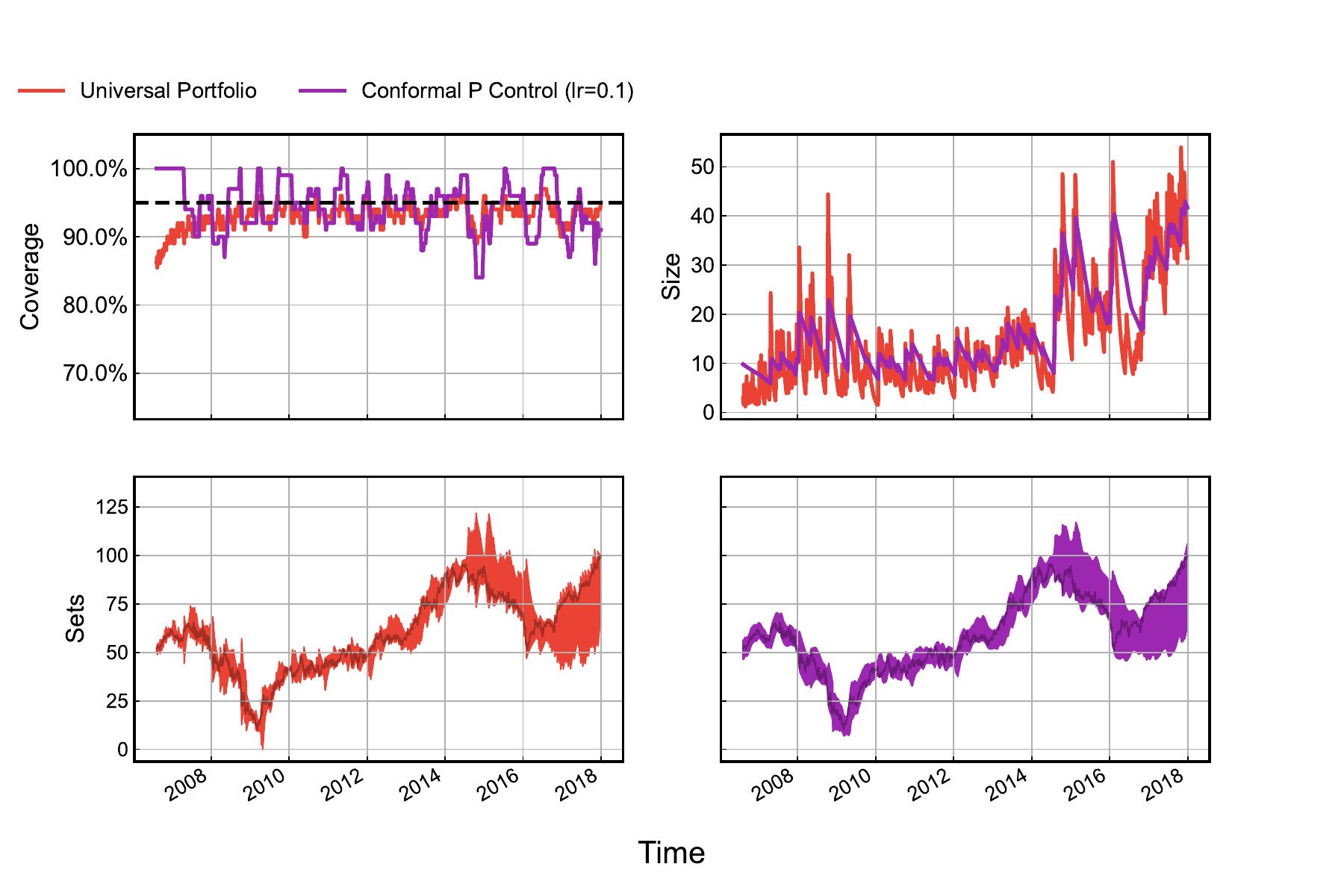}
  \caption{As in Figure~\ref{fig:AXP-UP-vs-DtACI-local}, UP-OCP vs. P Ctrl (lr=0.1).}
  \label{fig:AXP-UP-vs-P_01-local}
\end{figure}

\begin{figure}[H]
  \centering
  \includegraphics[trim={0 0 0 1.2cm}, clip, width=0.85\columnwidth]{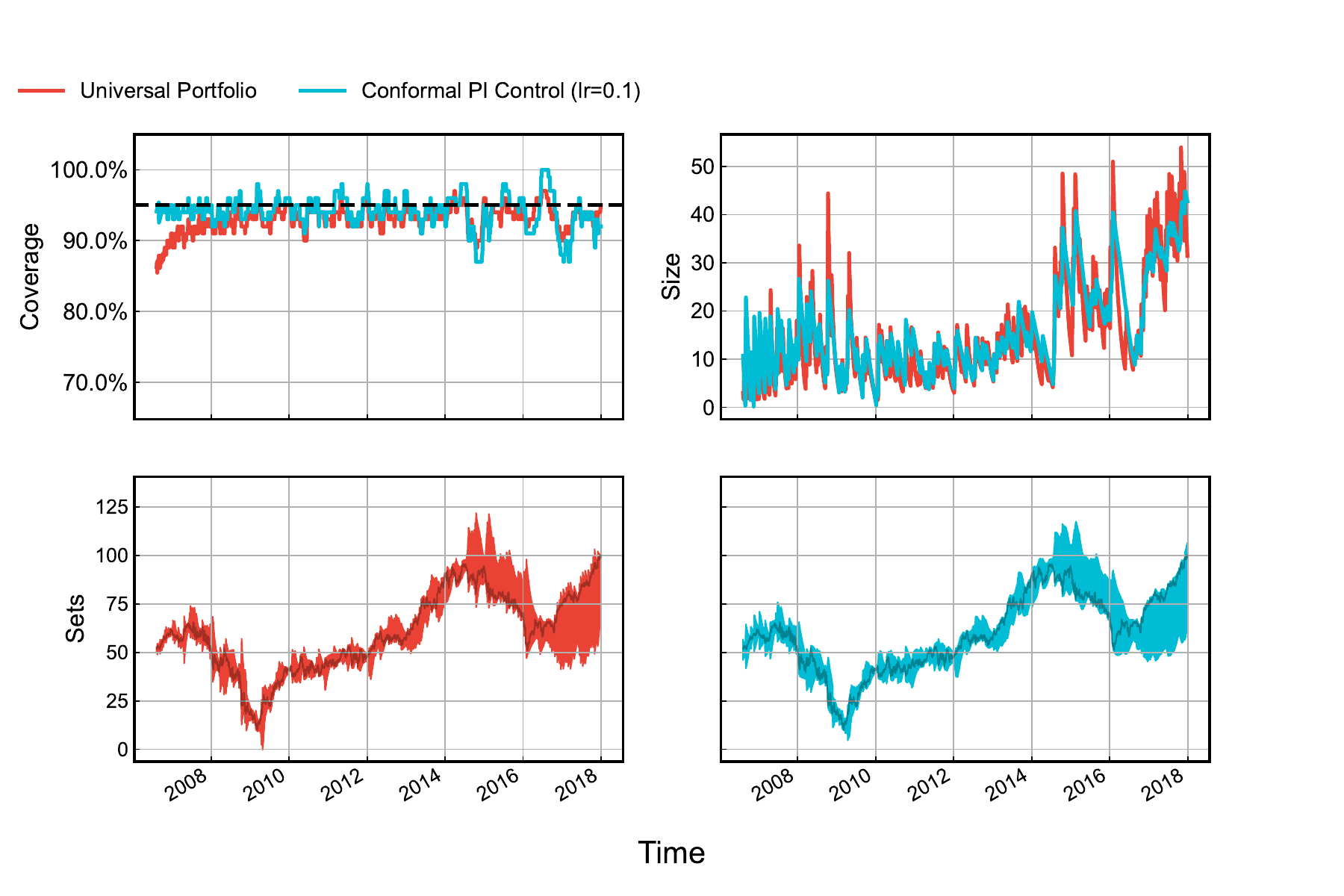}
  \caption{As in Figure~\ref{fig:AXP-UP-vs-DtACI-local}, UP-OCP vs. PI Ctrl (lr=0.1).}
  \label{fig:AXP-UP-vs-PI_01-local}
\end{figure}

For completeness, Table~\ref{tab:AXP-1vall-metrics-full} extends the results from the main text to include the heuristic Conformal P and PI Controllers~\citep{angelopoulos2023conformal}. These methods were tuned via grid search to maximize performance.

\begin{table}[H]
  \caption{Quantitative Comparison on the AXP Dataset. 
  }
  \label{tab:AXP-1vall-metrics-full}
  \centering
  \begin{tabular}{lcccccc}
    \toprule
    & UP & KT & DtACI & SFOGD (lr=25) & P Ctrl (lr=0.1) & PI Ctrl (lr=0.1) \\
    \midrule
    Marginal coverage & 0.932 & 0.92 & 0.956 & 0.948 & 0.942 & 0.941 \\
    Longest err sequence & \textbf{4} & 15 & 6 & \textbf{3} & 7 & \textbf{5} \\
    Average set size & \textbf{14.8} & 16.9 & $\infty$ & 16.4 & 16.8 & 16.1 \\
    Median set size & \textbf{11.5} & 12.9 & 12.6 & 13.8 & 13.1 & 13.2 \\
    75\% quantile set size & \textbf{18.8} & 24.9 & 21.8 & 21.5 & 21.5 & 20.9 \\
    90\% quantile set size & \textbf{32.3} & 32.5 & $\infty$ & 32.3 & 32.4 & \textbf{32} \\
    95\% quantile set size & 38 & 36.1 & $\infty$ & 36.1 & 36.6 & 35.9 \\
    \bottomrule
  \end{tabular}
\end{table}

\textbf{More Pareto Frontiers and Target-level Tracking.} 
In Section~\ref{sec:experiments}, we prioritized the presentation of \emph{average} prediction set size to penalize the infinite sets produced by algorithms like DtACI. However, robust statistics such as the median and 75\% quantile provide insight into the typical performance of the algorithms, ignoring the heavy tails.  We present additional Pareto frontiers for these metrics in Figures~\ref{fig:AXP-Pareto-average} through \ref{fig:AXP-Pareto-q75}.

These results clarify that DtACI is not inefficient on average days; its median performance overlaps with UP-OCP. The inefficiency is driven entirely by its inability to handle tail events properly. UP-OCP, however, dominates or matches all baselines across both mean and robust metrics, proving it is both stable in the worst case and efficient in the typical case.

Figure~\ref{fig:AXP-Pareto-median} clarifies that DtACI is not inefficient on average days; its median performance overlaps with UP-OCP. The inefficiency is driven entirely by its inability to handle tail events properly, which is shown by the divergence in the mean (Figure~\ref{fig:AXP-Pareto-average}) and 75\% quantile metrics (Figure~\ref{fig:AXP-Pareto-q75}). UP-OCP, however, dominates or matches all baselines across both mean and robust metrics, proving it is both stable in the worst case and efficient in the typical case.

\textbf{Target-level Tracking.} 
Finally, Figure~\ref{fig:AXP-tracking-targets} illustrates the ability of the algorithms to track the user-specified target coverage across a spectrum of $\alpha$ values.

\begin{figure}[H]
  \centering
  % --- Top Row: 1x2 ---
  \begin{minipage}[t]{0.49\columnwidth}
    \centering
    \includegraphics[trim={0 0 0 0.5cm}, clip, width=\linewidth]{AXP_mean_burnin50_pareto_frontier.pdf}
    \caption{Mean prediction set sizes on AXP.}
    \label{fig:AXP-Pareto-average}
  \end{minipage}
  \hfill
  \begin{minipage}[t]{0.49\columnwidth}
    \centering
    \includegraphics[trim={0 0 0 0.5cm}, clip, width=\linewidth]{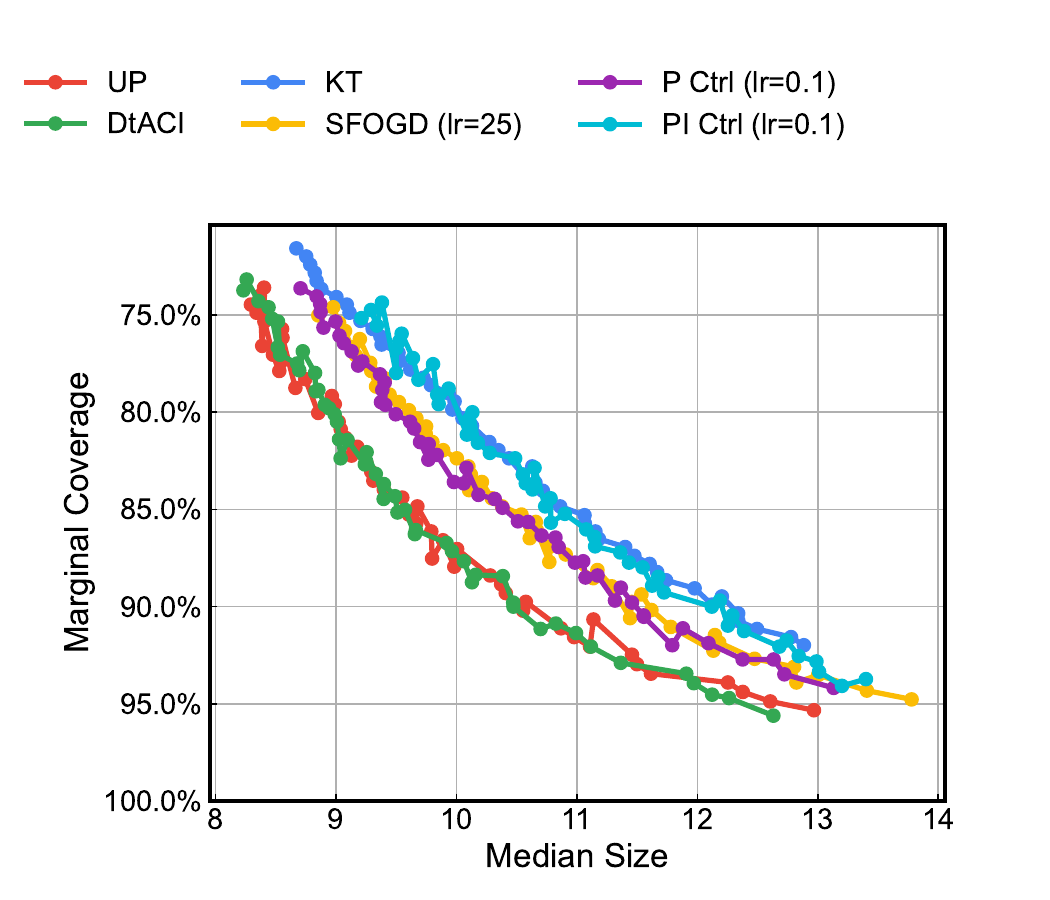}
    \caption{Median prediction set sizes on AXP.}
    \label{fig:AXP-Pareto-median}
  \end{minipage}
  
  \vspace{0.5cm}

  \begin{minipage}[t]{0.49\columnwidth}
    \centering
    \includegraphics[trim={0 0 0 0.5cm}, clip, width=\linewidth]{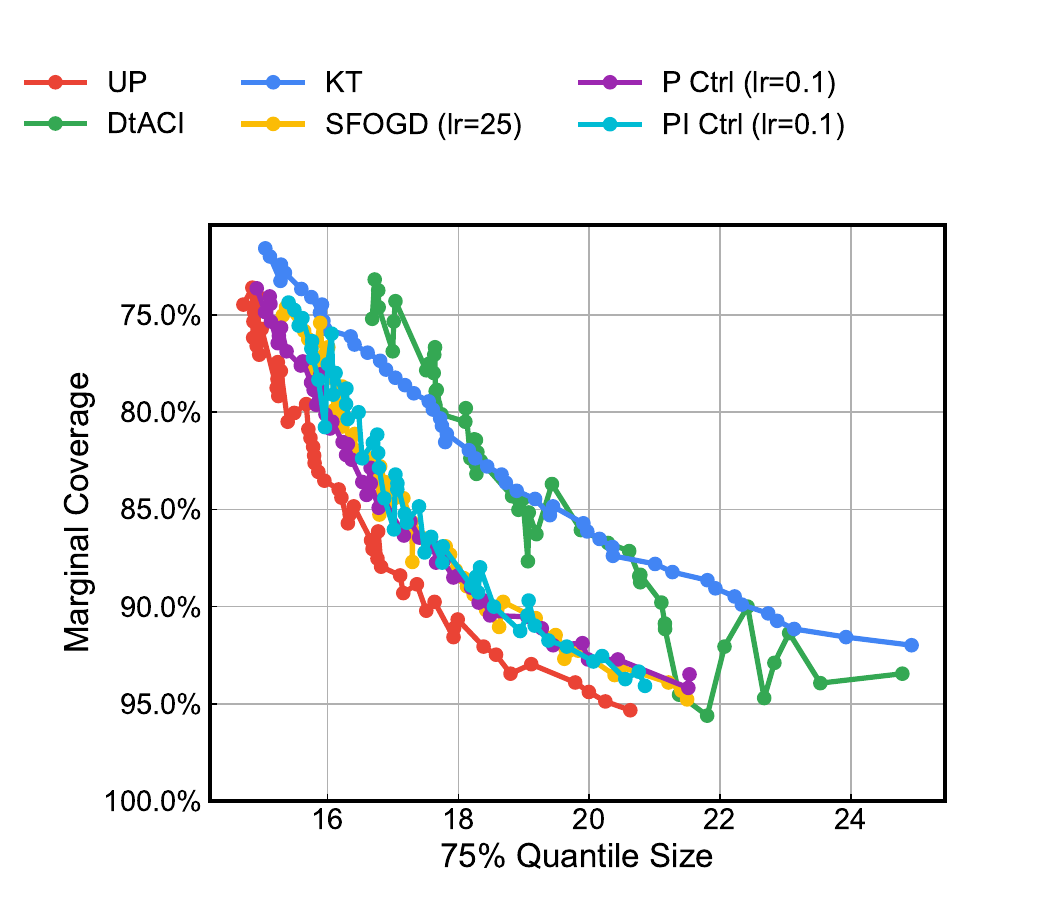}
    \caption{75\% quantile prediction set sizes on AXP.}
    \label{fig:AXP-Pareto-q75}
  \end{minipage}
  \hfill
  \begin{minipage}[t]{0.49\columnwidth}
    \centering
    \includegraphics[trim={0 0cm 0 1cm}, clip, width=\linewidth]{AXP_target_level_tracking.pdf}
    \caption{Realized vs. target coverage on AXP. Most methods track the diagonal within a small tolerance ($\pm$ 0.03).}
    \label{fig:AXP-tracking-targets}
  \end{minipage}
\end{figure}

\newpage
\section{Additional Experiments}
\label{app:additional-experiments}
To demonstrate generalization of our method, we extend our evaluation to a diverse set of real-world and synthetic benchmarks. These include three additional major stocks (AAPL, AMZN, GOOGL), an electricity demand dataset (NSW), and three synthetic environments designed to test adaptivity (sinusoid, stationary wavelet, and quadratic drift). The observed behaviors across these experiments remain qualitatively consistent with the findings from the AXP dataset. Therefore, we omit a detailed discussion to avoid redundancy and present the following figures and tables for completeness. We will pinpoint interesting observations wherever applicable.

\textbf{More on Experimental Setup.}
For the financial datasets (AAPL, AMZN, GOOGL), we follow the same protocol as the AXP experiment: we employ the Prophet model \citep{taylor2018forecasting} as the base forecaster, targeting a miscoverage rate of $\alpha=0.05$ with an initial burn-in period of $T_{\text{burnin}}=100$ days.
For the electricity demand dataset, we utilize a standard Autoregressive (AR) model with a burn-in of $T_{\text{burnin}}=300$ steps to capture the high-frequency intraday seasonality.
The synthetic experiments generate nonconformity scores directly to isolate specific distributional shifts (periodicity, sparse spikes, and drift), also using a burn-in of 300 steps.
To ensure statistical significance, all synthetic results are reported as averages over 10 independent random seeds, with error bars in figures denoting the standard error of the mean.
In all comparisons, we evaluate UP-OCP against the full suite of parameter-free (KT, DtACI) and optimized baselines (SF-OGD, P-Control, PI-Control) described in Section~\ref{sec:experiments}.

\textbf{Take Away.} It is worth noting that the optimal hyperparameters for the baseline methods vary across datasets. We illustrate this variability in Appendix~\ref{app:additional-experiments}, highlighting that no single hyperparameter configuration yields consistent performance across all benchmarks. This sensitivity necessitates dataset-specific tuning, a requirement that our parameter-free UP-OCP method avoids.

\subsection{AAPL Dataset}

\textbf{Local Adaptivity.}

\begin{figure}[H]
  \centering
  \includegraphics[trim={0 0 0 1.2cm}, clip, width=0.85\columnwidth]{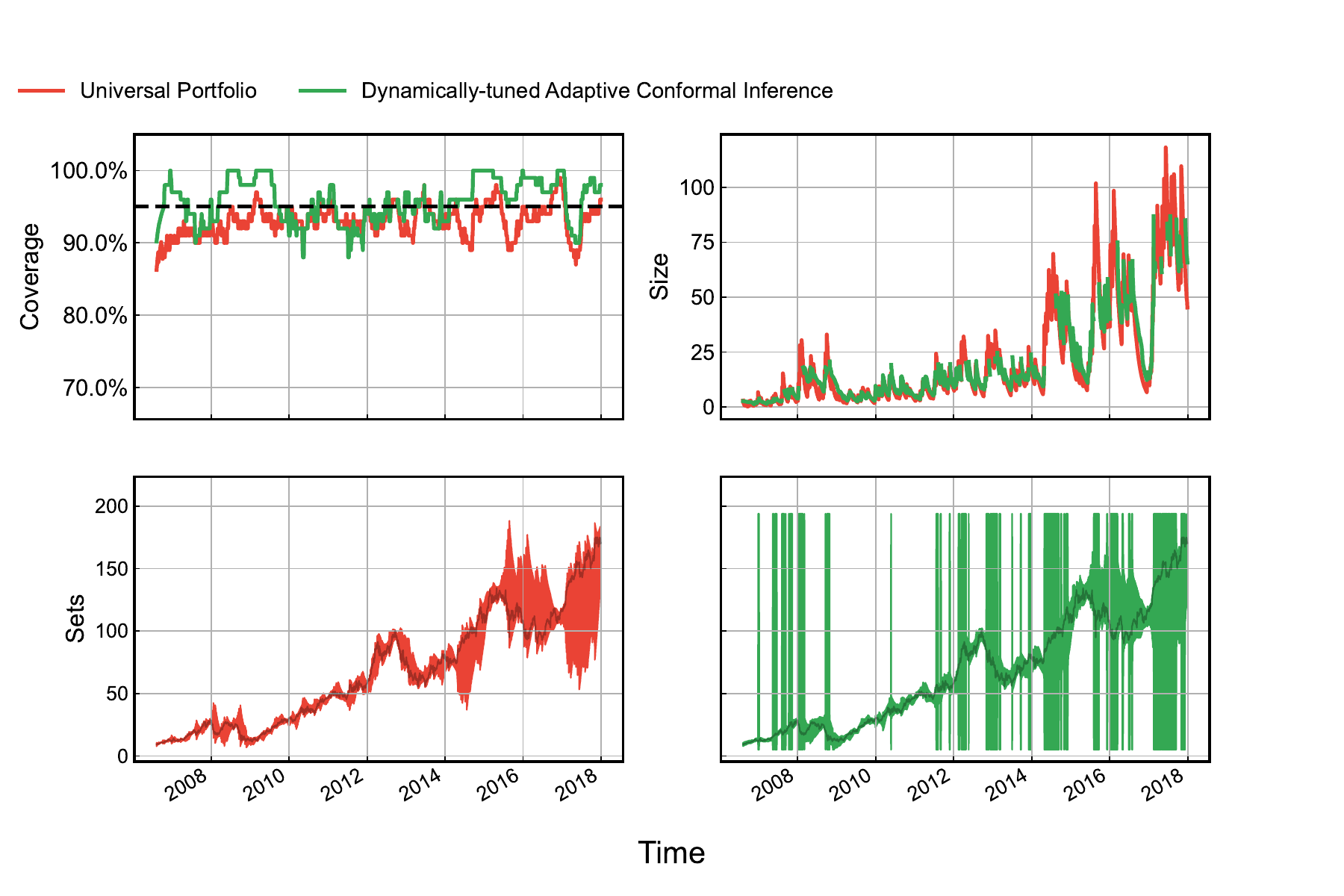}
  \caption{UP-OCP vs. DtACI for forecasting AAPL stock return.}
  \label{fig:AAPL-UP-vs-DtACI-local}
\end{figure} 

\begin{figure}[H]
  \centering
  \includegraphics[trim={0 0 0 1.2cm}, clip, width=0.85\columnwidth]{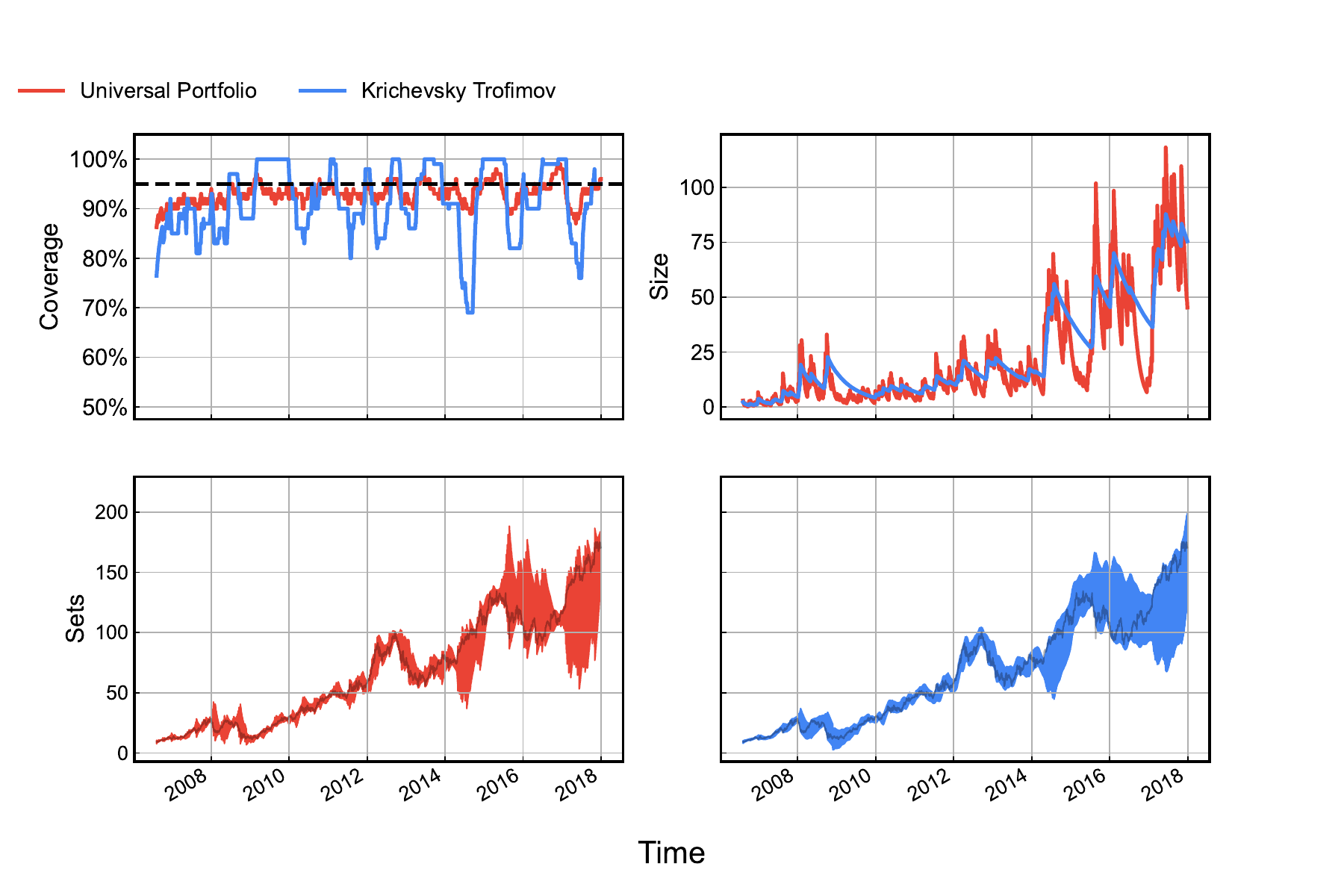}
  \caption{As in Figure~\ref{fig:AAPL-UP-vs-DtACI-local}, UP-OCP vs. KT.}
  \label{fig:AAPL-UP-vs-KT-local}
\end{figure}  

\begin{figure}[H]
  \centering
  \includegraphics[trim={0 0 0 1.2cm}, clip, width=0.85\columnwidth]{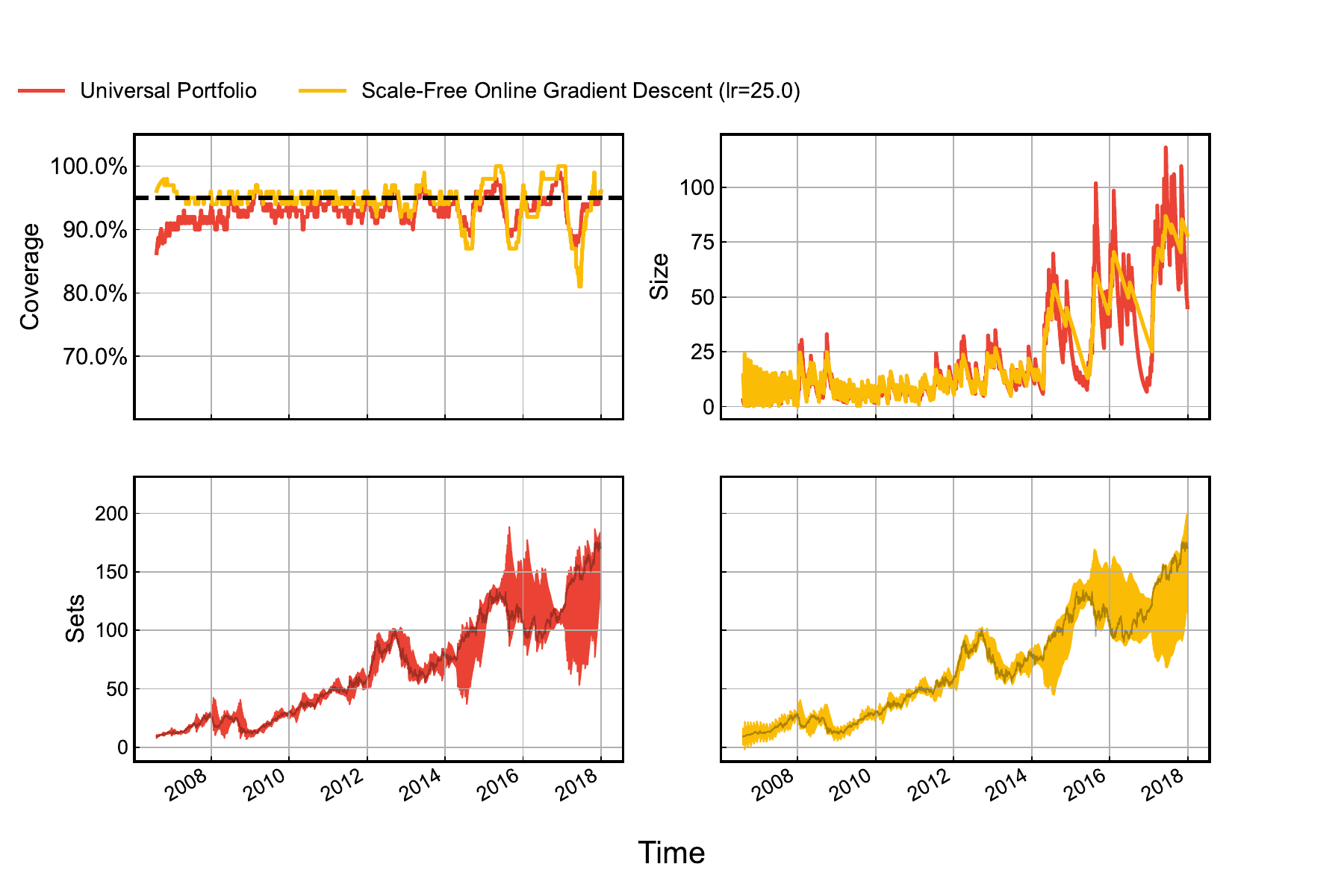}
  \caption{As in Figure~\ref{fig:AAPL-UP-vs-DtACI-local}, UP-OCP vs. SFOGD (lr=25).}
  \label{fig:AAPL-UP-vs-SGD-25-local}
\end{figure}

\begin{figure}[H]
  \centering
  \includegraphics[trim={0 0 0 1.2cm}, clip, width=0.85\columnwidth]{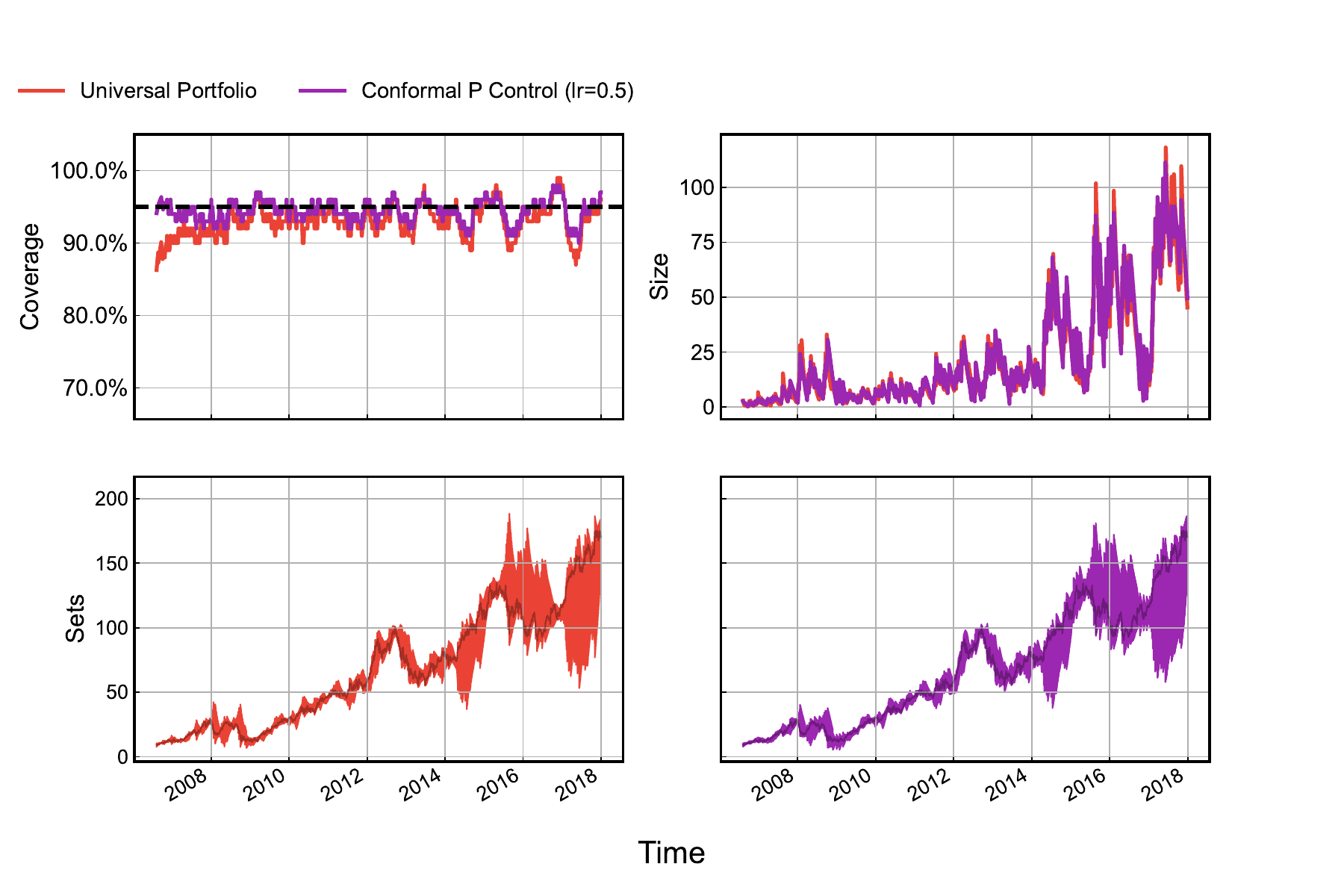}
  \caption{As in Figure~\ref{fig:AAPL-UP-vs-DtACI-local}, UP-OCP vs. P Ctrl (lr=0.5).}
  \label{fig:AAPL-UP-vs-P_05-local}
\end{figure}

\begin{figure}[H]
  \centering
  \includegraphics[trim={0 0 0 1.2cm}, clip, width=0.85\columnwidth]{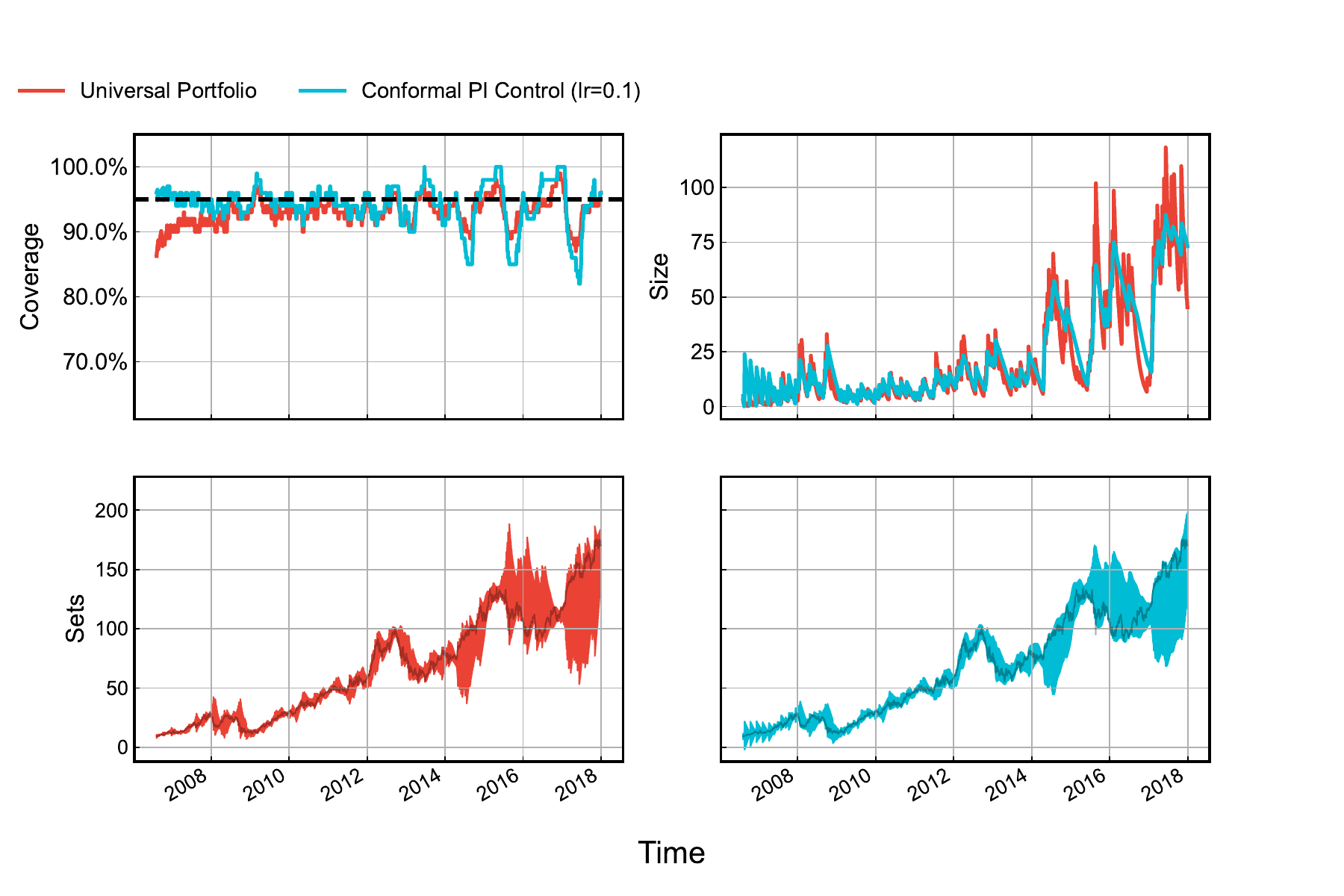}
  \caption{As in Figure~\ref{fig:AAPL-UP-vs-DtACI-local}, UP-OCP vs. PI Ctrl (lr=0.1).}
  \label{fig:AAPL-UP-vs-PI_01-local}
\end{figure}

\begin{table}[H]
  \caption{Quantitative Comparison on the AAPL Dataset.}
  \label{tab:AAPL-1vall-metrics-full}
  \centering
  \begin{tabular}{lcccccc}
    \toprule
    & UP & KT & DtACI & SFOGD (lr=25) & P Ctrl (lr=0.5) & PI Ctrl (lr=0.1) \\
    \midrule
    Marginal coverage      & 0.932 & 0.915 & 0.958    & 0.945 & 0.945 & 0.941 \\
    Longest err sequence   & \textbf{2}     & 16    & 4        & 4     & \textbf{2}     & 5     \\
    Average set size       & \textbf{21.6}  & 24.4  & $\infty$ & 23.7  & 22.8  & 23.2  \\
    Median set size        & \textbf{11.7}  & 14.5  & 14       & 13.8  & 13    & 13.7  \\
    75\% quantile set size & \textbf{27.3}  & 39.8  & 55.8     & 36.9  & 29.2  & 33.3  \\
    90\% quantile set size & 60.1  & 58.8  & $\infty$ & 59    & 63.7  & 59.8  \\
    95\% quantile set size & 75.9  & 75.9  & $\infty$ & 74    & 77.4  & 74.5  \\
    \bottomrule
  \end{tabular}
\end{table}

\textbf{More Pareto Frontiers and Target-level Tracking.} 

\begin{figure}[H]
  \centering
  % --- Top Row: 1x2 ---
  \begin{minipage}[t]{0.49\columnwidth}
    \centering
    \includegraphics[trim={0 0 0 0.5cm}, clip, width=\linewidth]{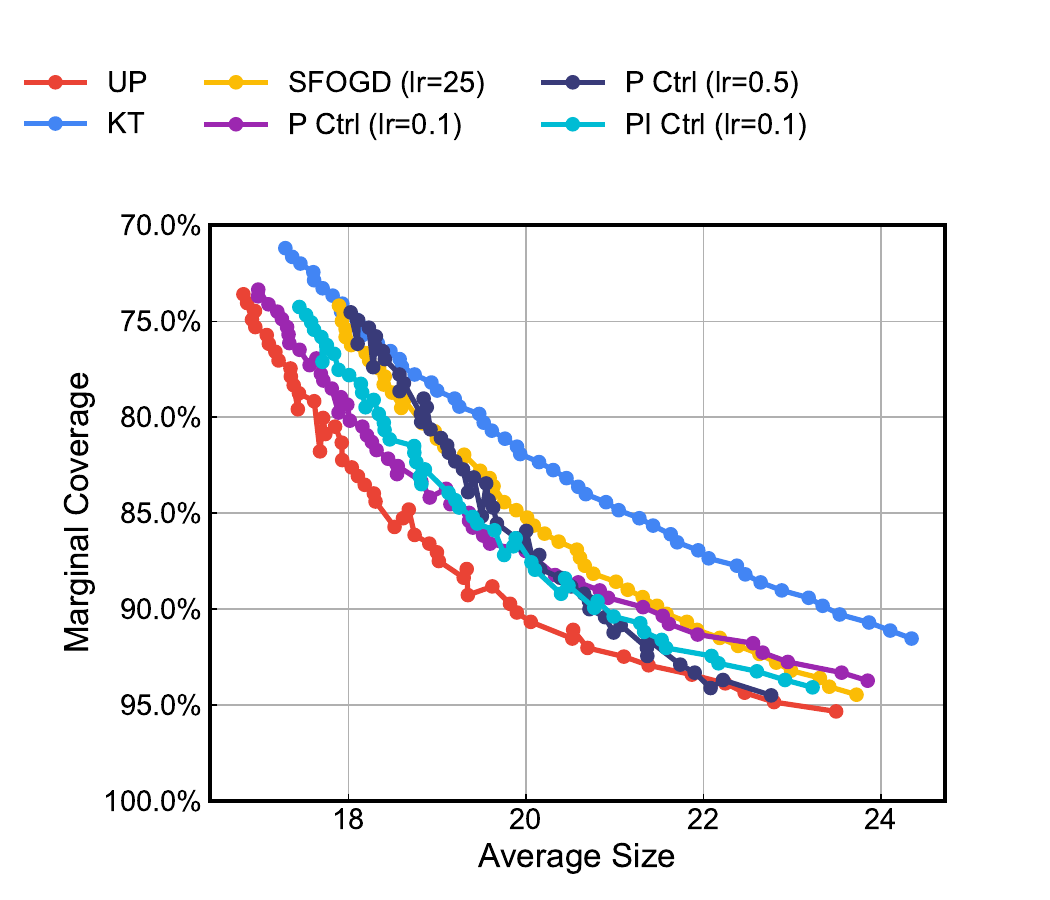}
    \caption{Mean prediction set sizes on AAPL.}
    \label{fig:AAPL-Pareto-average}
  \end{minipage}
  \hfill
  \begin{minipage}[t]{0.49\columnwidth}
    \centering
    \includegraphics[trim={0 0 0 0.5cm}, clip, width=\linewidth]{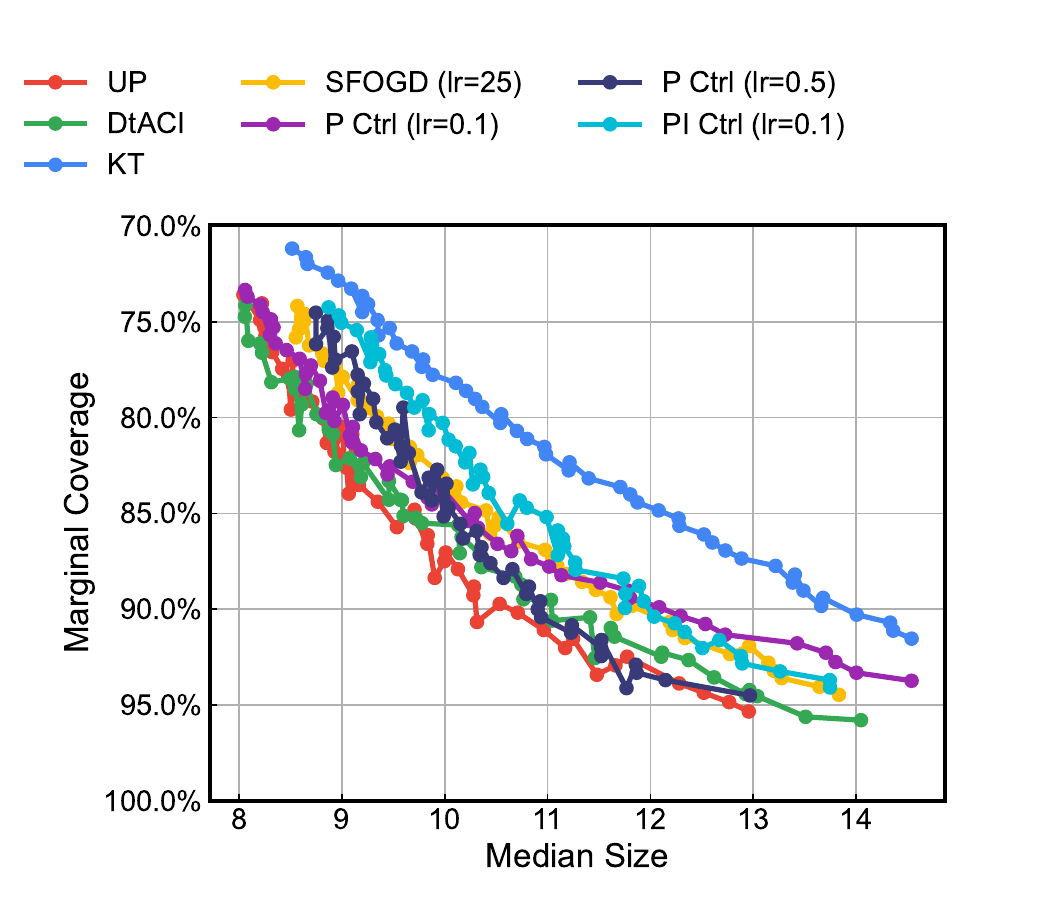}
    \caption{Median prediction set sizes on AAPL.}
    \label{fig:AAPL-Pareto-median}
  \end{minipage}
  
  \vspace{0.5cm}

  \begin{minipage}[t]{0.49\columnwidth}
    \centering
    \includegraphics[trim={0 0 0 0.5cm}, clip, width=\linewidth]{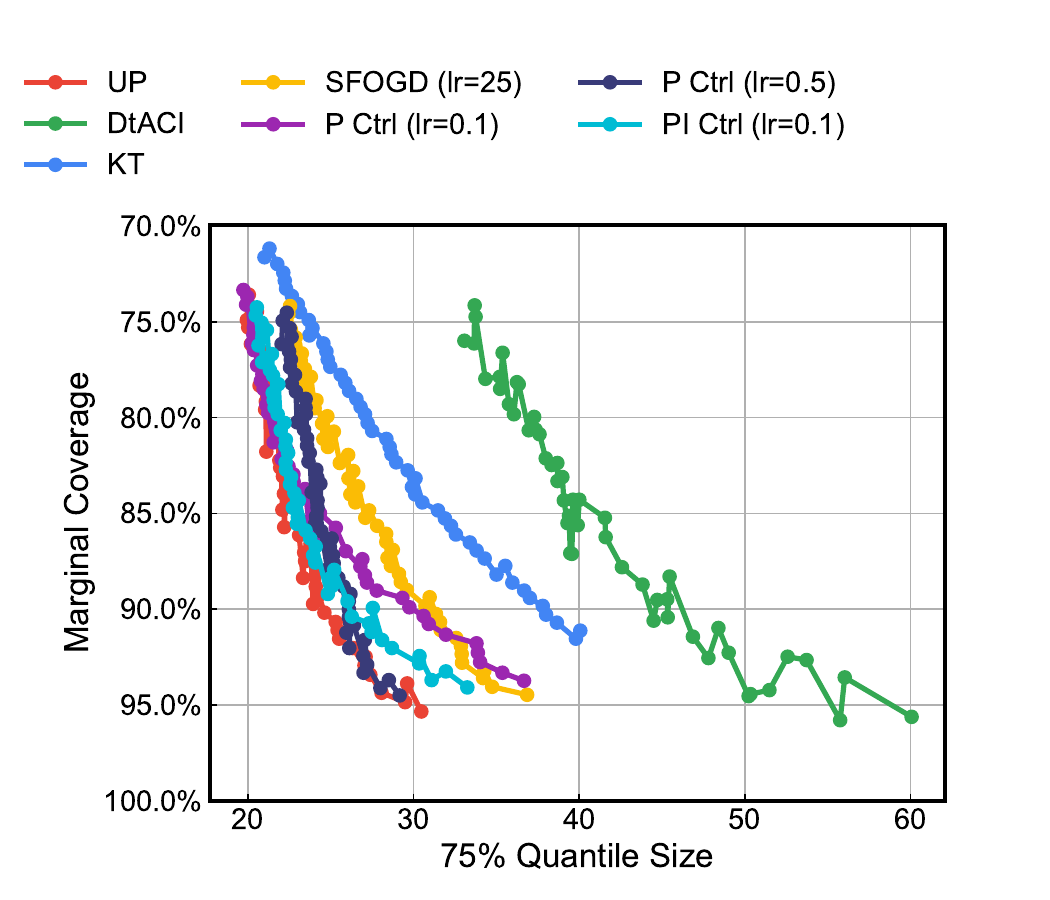}
    \caption{75\% quantile prediction set sizes on AAPL.}
    \label{fig:AAPL-Pareto-q75}
  \end{minipage}
  \hfill
  \begin{minipage}[t]{0.49\columnwidth}
    \centering
    \includegraphics[trim={0 0 0 1cm}, clip, width=\linewidth]{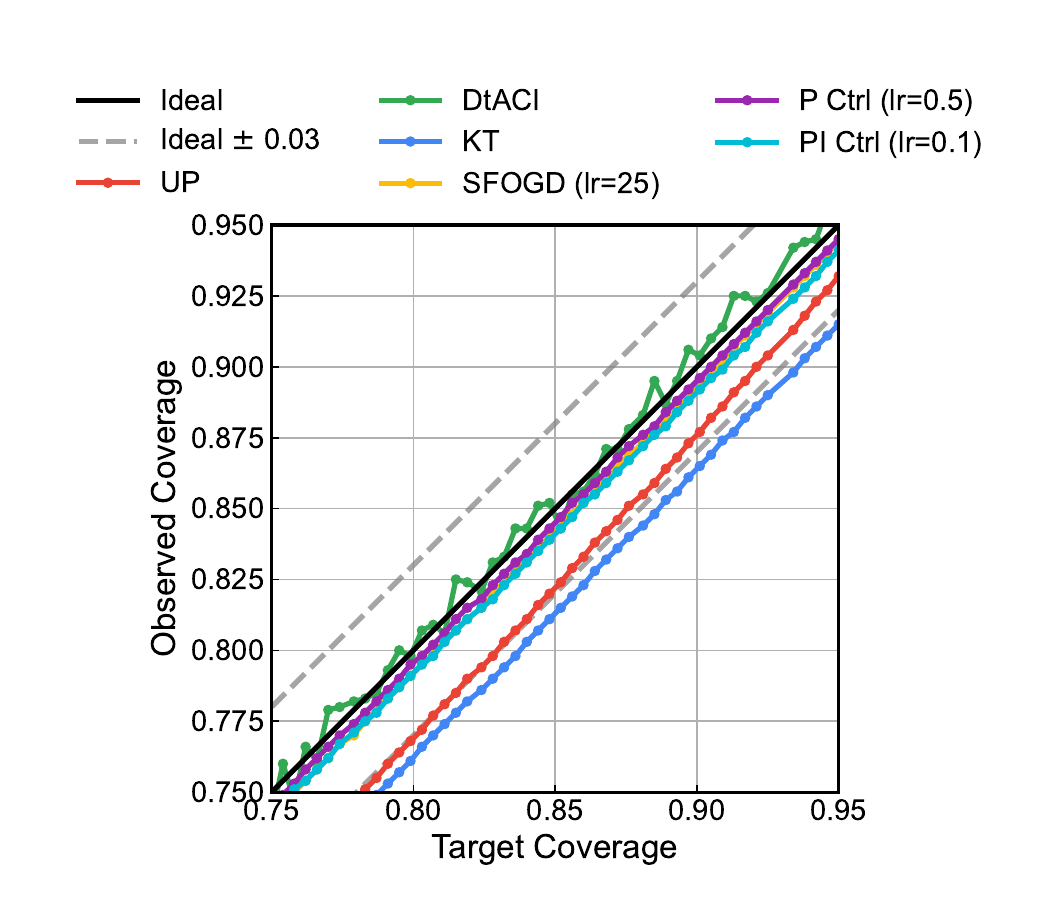}
    \caption{Realized vs. target coverage on AAPL. Most methods track the diagonal within a small tolerance ($\pm$ 0.03).}
    \label{fig:AAPL-tracking-targets}
  \end{minipage}
\end{figure}

\newpage
\begin{remark}
The results here on all three metric offer a compelling demonstration of adaptivity. UP-OCP (red) does not simply outperform a single baseline; rather, it effectively automates the hyperparameter selection process across the target levels.

Observe the behavior of the tuned P-Controllers (lr=0.1, purple): at moderate targets (75--85\% coverage), the controller is optimal, while the higher-gain controller (lr=0.5, dark blue) is less efficient. Conversely, at high targets (90--95\%), the high-gain controller becomes necessary to maintain tight sets, while the low-gain version falls behind. UP-OCP remarkably matches the performance of the \emph{best-tuned} baseline in \emph{each} specific regime. It is able to align with the purple curve at lower targets and the dark blue curve at higher targets, without requiring any manual tuning or gain scheduling.
\end{remark}

\newpage
\subsection{AMZN Dataset}

\textbf{Local Adaptivity.}
\begin{figure}[H]
  \centering
  \includegraphics[trim={0 0 0 1.2cm}, clip, width=0.85\columnwidth]{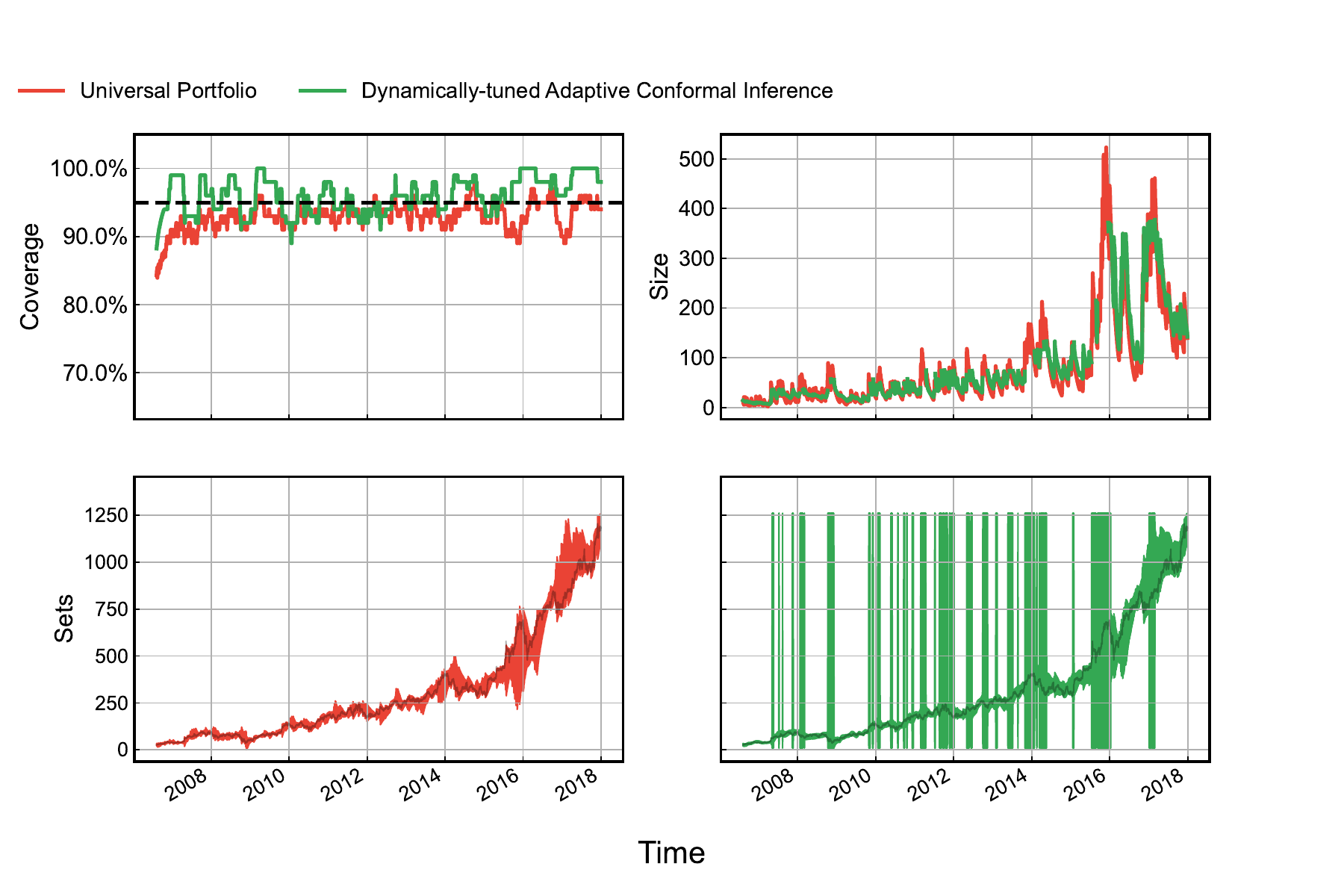}
  \caption{UP-OCP vs. DtACI for forecasting AMZN stock return.}
  \label{fig:AMZN-UP-vs-DtACI-local}
\end{figure} 

\begin{figure}[H]
  \centering
  \includegraphics[trim={0 0 0 1.2cm}, clip, width=0.85\columnwidth]{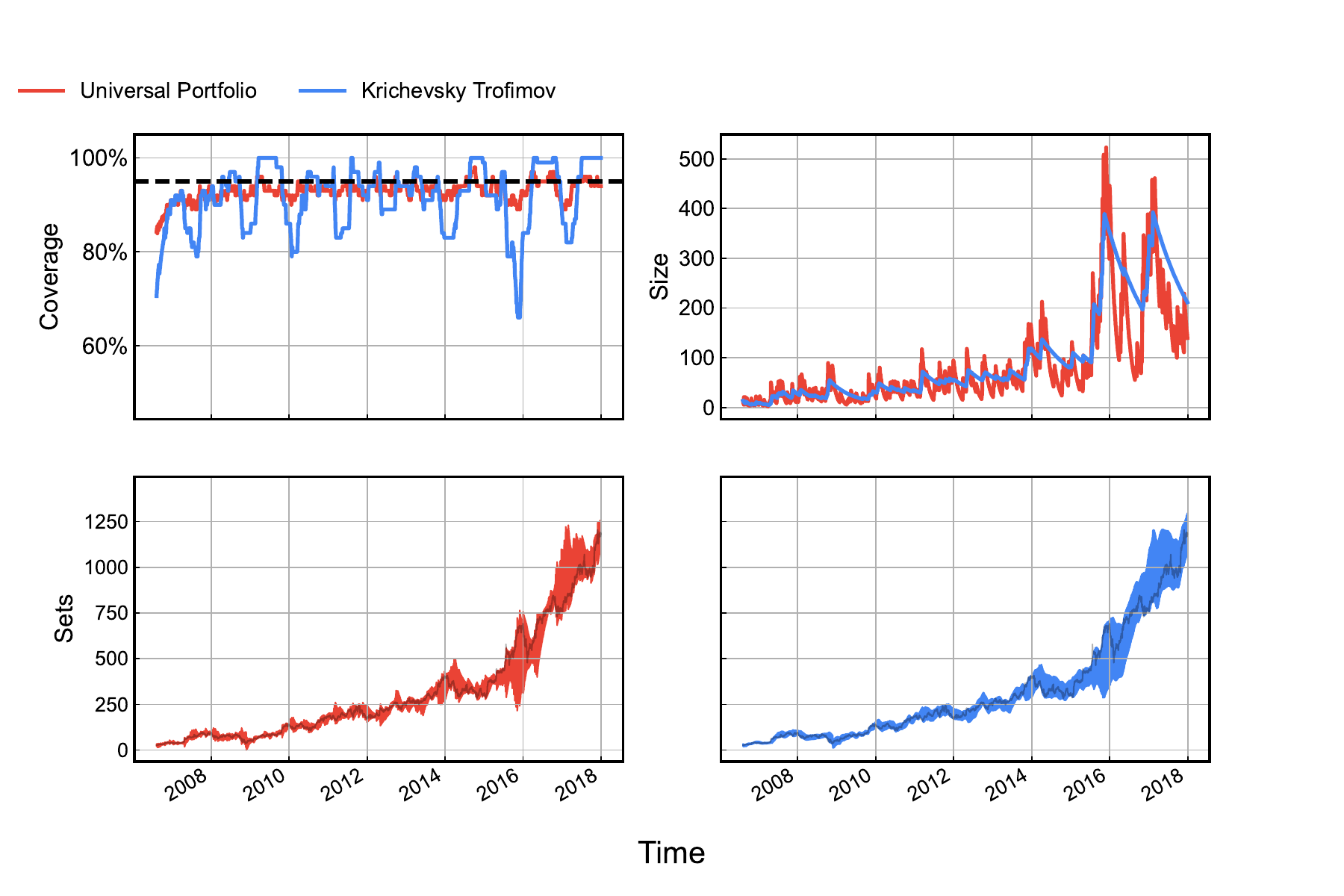}
  \caption{As in Figure~\ref{fig:AMZN-UP-vs-DtACI-local}, UP-OCP vs. KT.}
  \label{fig:AMZN-UP-vs-KT-local}
\end{figure} 

\begin{figure}[H]
  \centering
  \includegraphics[trim={0 0 0 1.2cm}, clip, width=0.85\columnwidth]{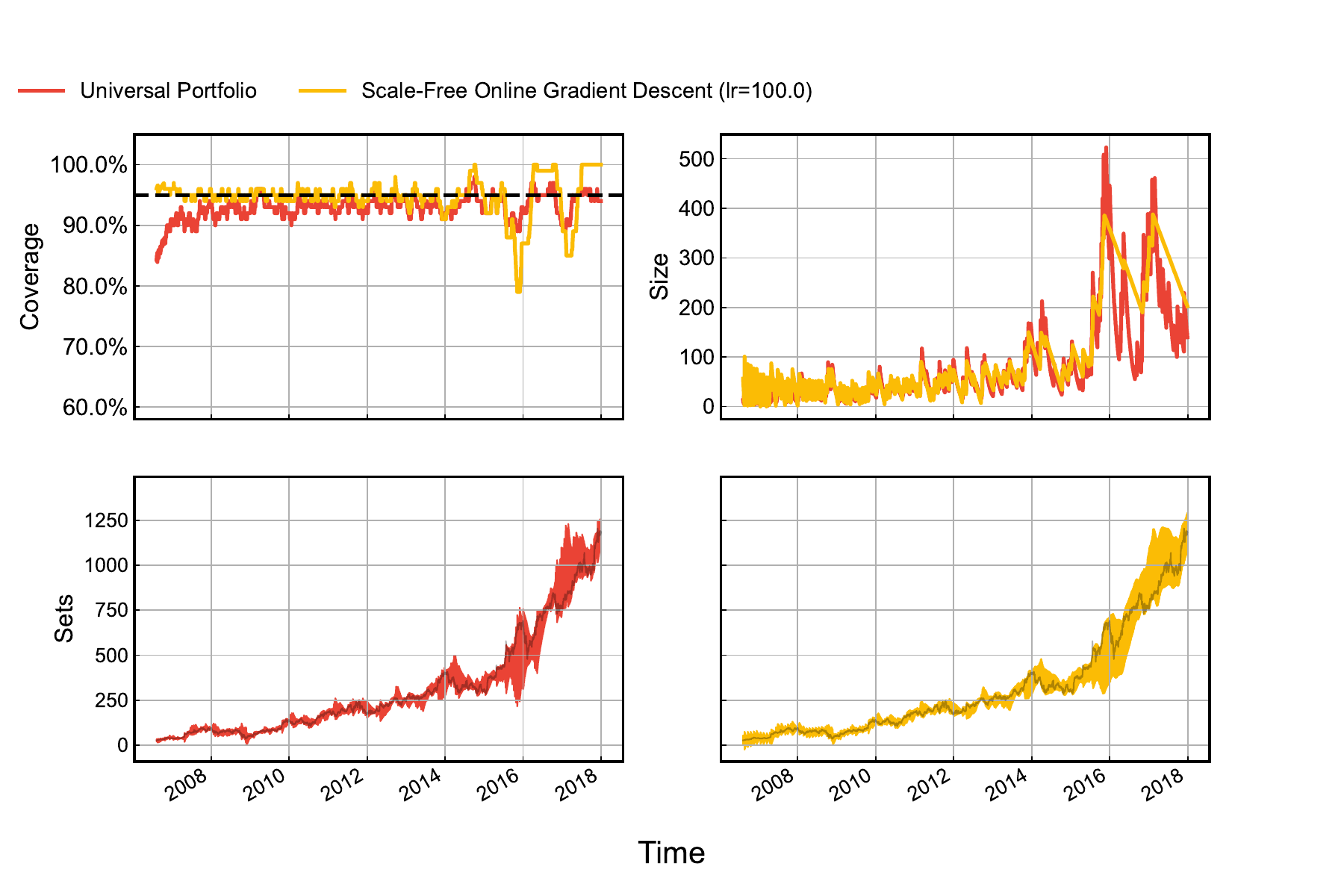}
  \caption{As in Figure~\ref{fig:AMZN-UP-vs-DtACI-local}, UP-OCP vs. SFOGD (lr=100).}
  \label{fig:AMZN-UP-vs-SGD-100-local}
\end{figure}

\begin{figure}[H]
  \centering
  \includegraphics[trim={0 0 0 1.2cm}, clip, width=0.85\columnwidth]{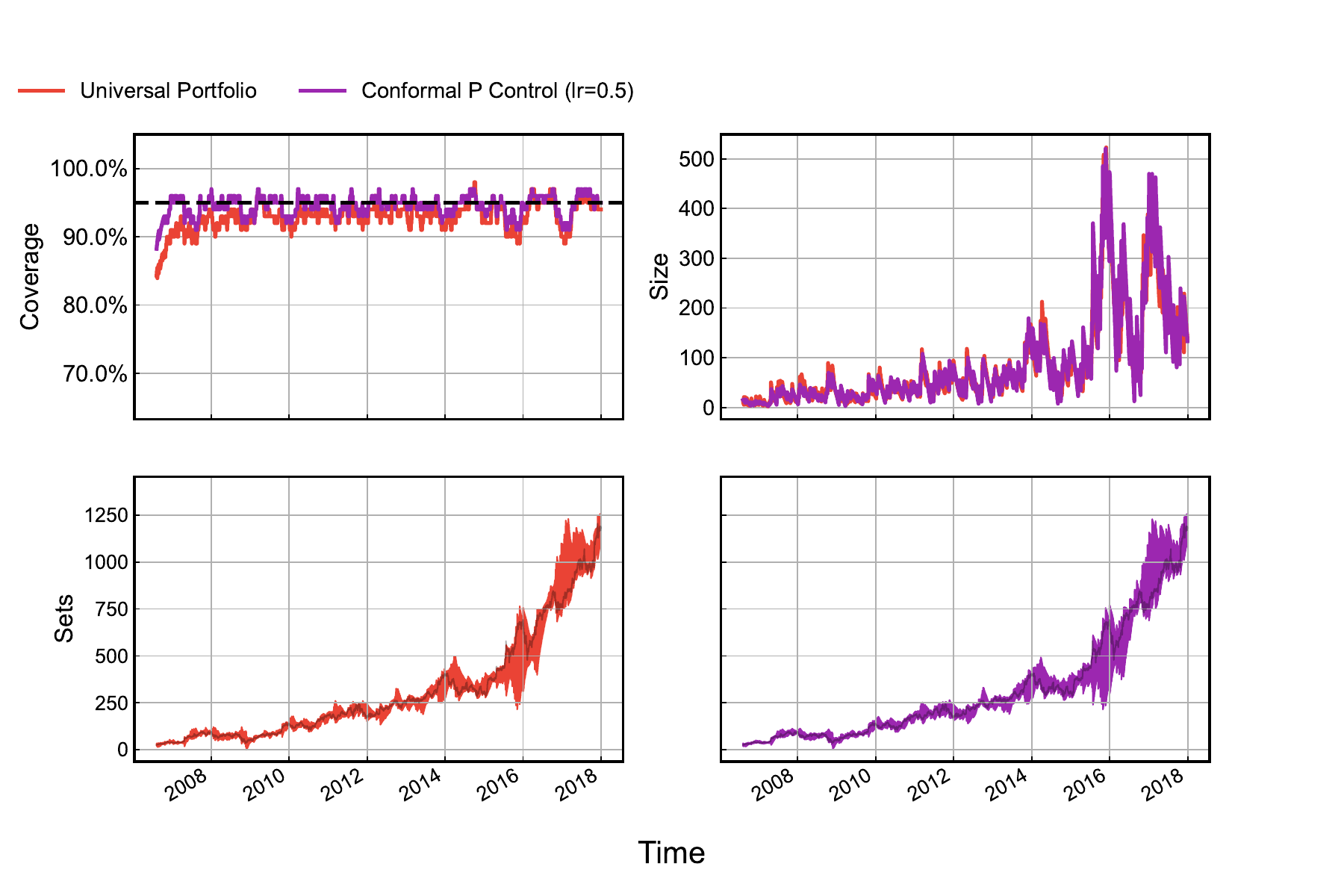}
  \caption{As in Figure~\ref{fig:AMZN-UP-vs-DtACI-local}, UP-OCP vs. P Ctrl (lr=0.5).}
  \label{fig:AMZN-UP-vs-P_05-local}
\end{figure}

\begin{figure}[H]
  \centering
  \includegraphics[trim={0 0 0 1.2cm}, clip, width=0.85\columnwidth]{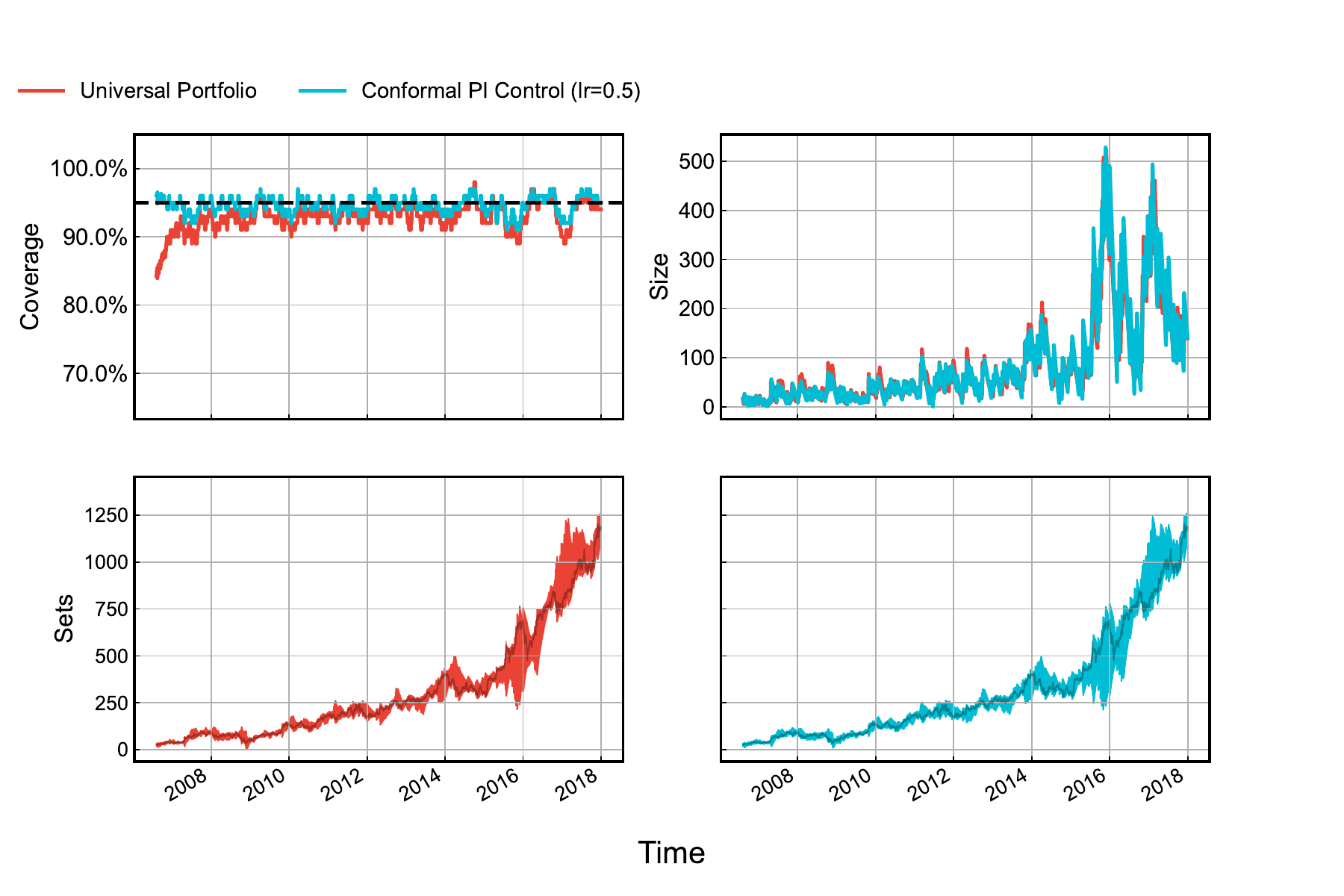}
  \caption{As in Figure~\ref{fig:AMZN-UP-vs-DtACI-local}, UP-OCP vs. PI Ctrl (lr=0.5).}
  \label{fig:AMZN-UP-vs-PI_05-local}
\end{figure}

\begin{table}[ht]
  \caption{Quantitative Comparison on the AMZN Dataset.}
  \label{tab:AMZN-1vall-metrics-full}
  \centering
  \begin{tabular}{lcccccc}
    \toprule
    & UP & KT & DtACI & SFOGD (lr=100) & P Ctrl (lr=0.5) & PI Ctrl (lr=0.5) \\
    \midrule
    Marginal coverage      & 0.931 & 0.919 & 0.962    & 0.947 & 0.946 & 0.946 \\
    Longest err sequence   & \textbf{3}     & 18    & 4        & 4     & \textbf{3}     & \textbf{3}     \\
    Average set size       & \textbf{82.8}  & 99.4  & $\infty$ & 102   & 88.9  & 89.3  \\
    Median set size        & \textbf{49.1}  & 57.2  & 60.4     & 57.5  & 51.5  & 51    \\
    75\% quantile set size & \textbf{97.4}  & 113   & 274      & 117   & 107   & 108   \\
    90\% quantile set size & \textbf{208}   & 276   & $\infty$ & 288   & 234   & 239   \\
    95\% quantile set size & \textbf{292}   & 330   & $\infty$ & 335   & 324   & 321   \\
    \bottomrule
  \end{tabular}
\end{table}

\newpage
\textbf{More Pareto Frontiers and Target-level Tracking.}
\begin{figure}[H]
  \centering
  % --- Top Row: 1x2 ---
  \begin{minipage}[t]{0.49\columnwidth}
    \centering
    \includegraphics[trim={0 0 0 0.5cm}, clip, width=\linewidth]{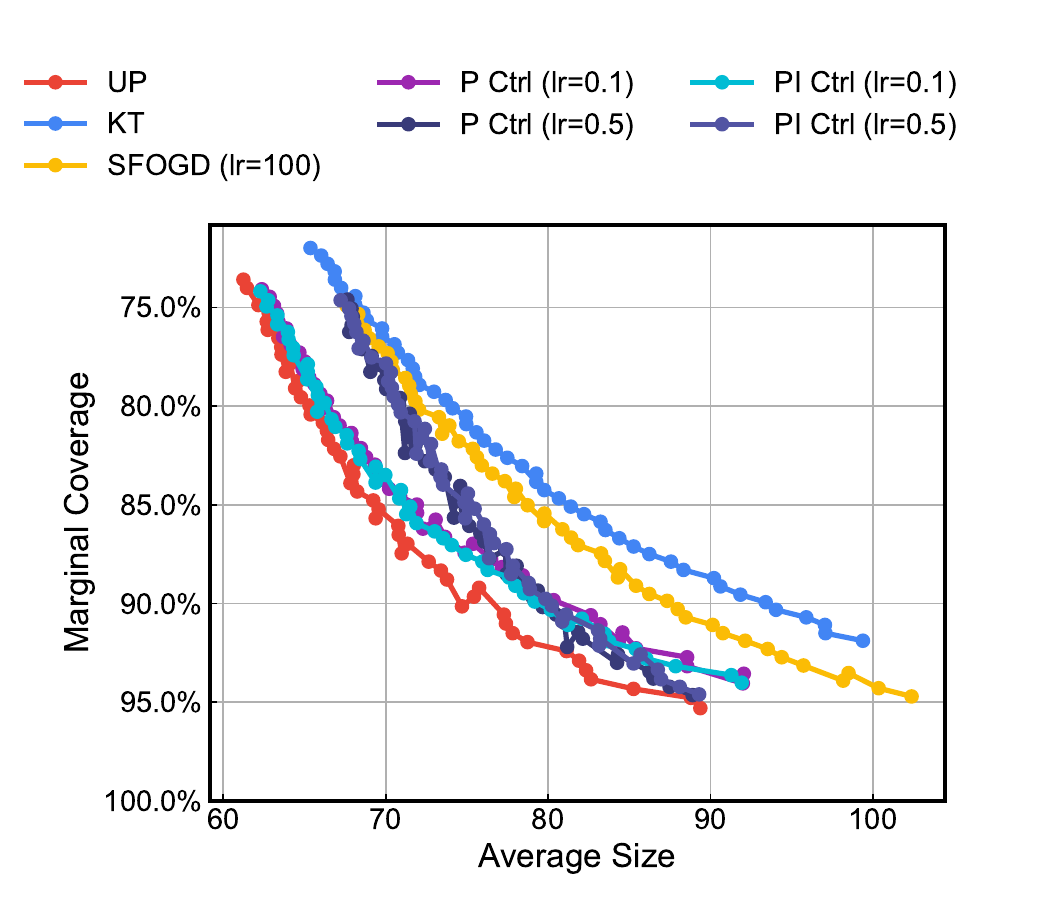}
    \caption{Mean prediction set sizes on AMZN.}
    \label{fig:AMZN-Pareto-average}
  \end{minipage}
  \hfill
  \begin{minipage}[t]{0.49\columnwidth}
    \centering
    \includegraphics[trim={0 0 0 0.5cm}, clip, width=\linewidth]{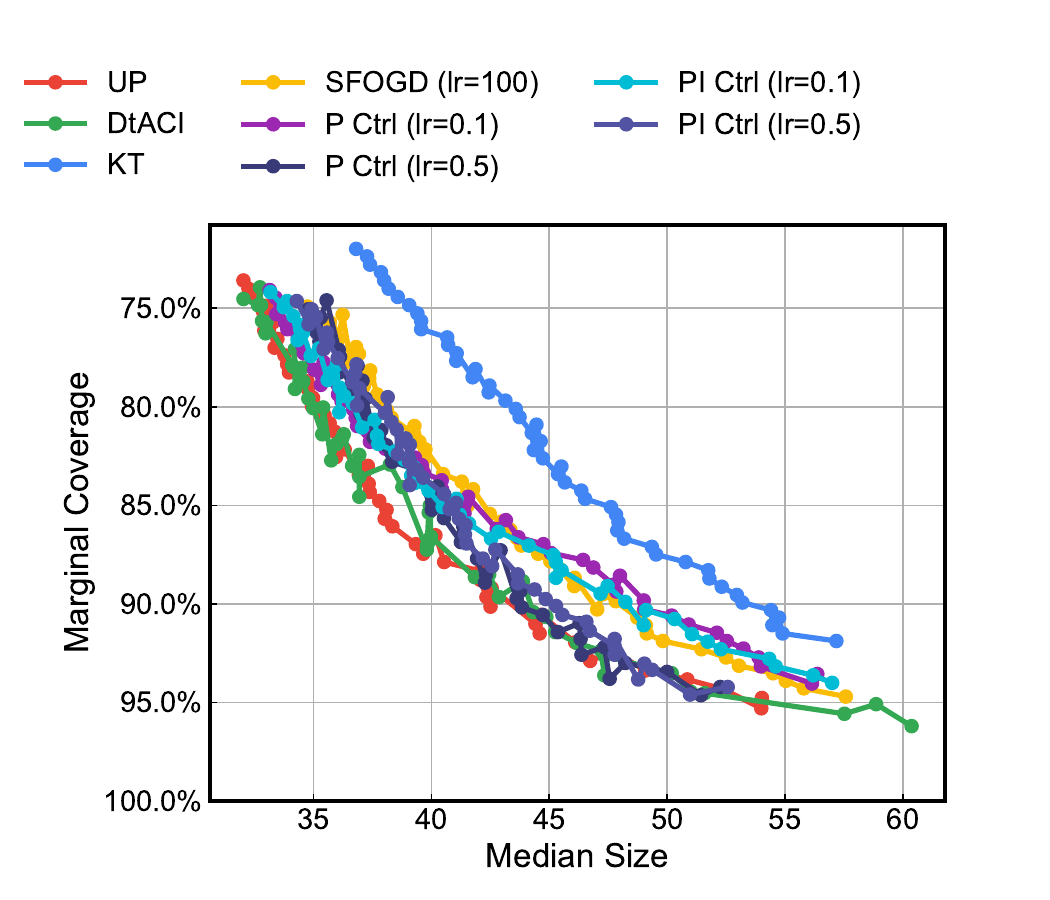}
    \caption{Median prediction set sizes on AMZN.}
    \label{fig:AMZN-Pareto-median}
  \end{minipage}
  
  \vspace{0.5cm}

  \begin{minipage}[t]{0.49\columnwidth}
    \centering
    \includegraphics[trim={0 0 0 0.5cm}, clip, width=\linewidth]{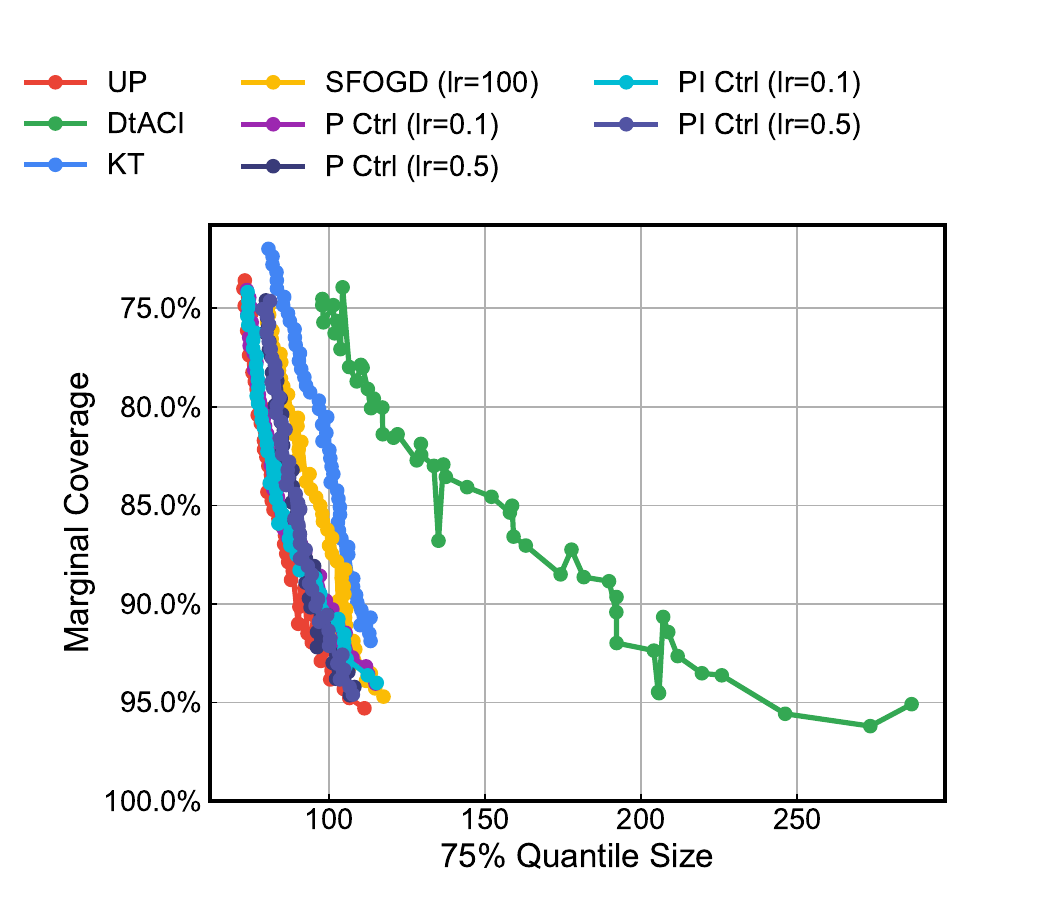}
    \caption{75\% quantile prediction set sizes on AMZN.}
    \label{fig:AMZN-Pareto-q75}
  \end{minipage}
  \hfill
  \begin{minipage}[t]{0.49\columnwidth}
    \centering
    \includegraphics[trim={0 0 0 1cm}, clip, width=\linewidth]{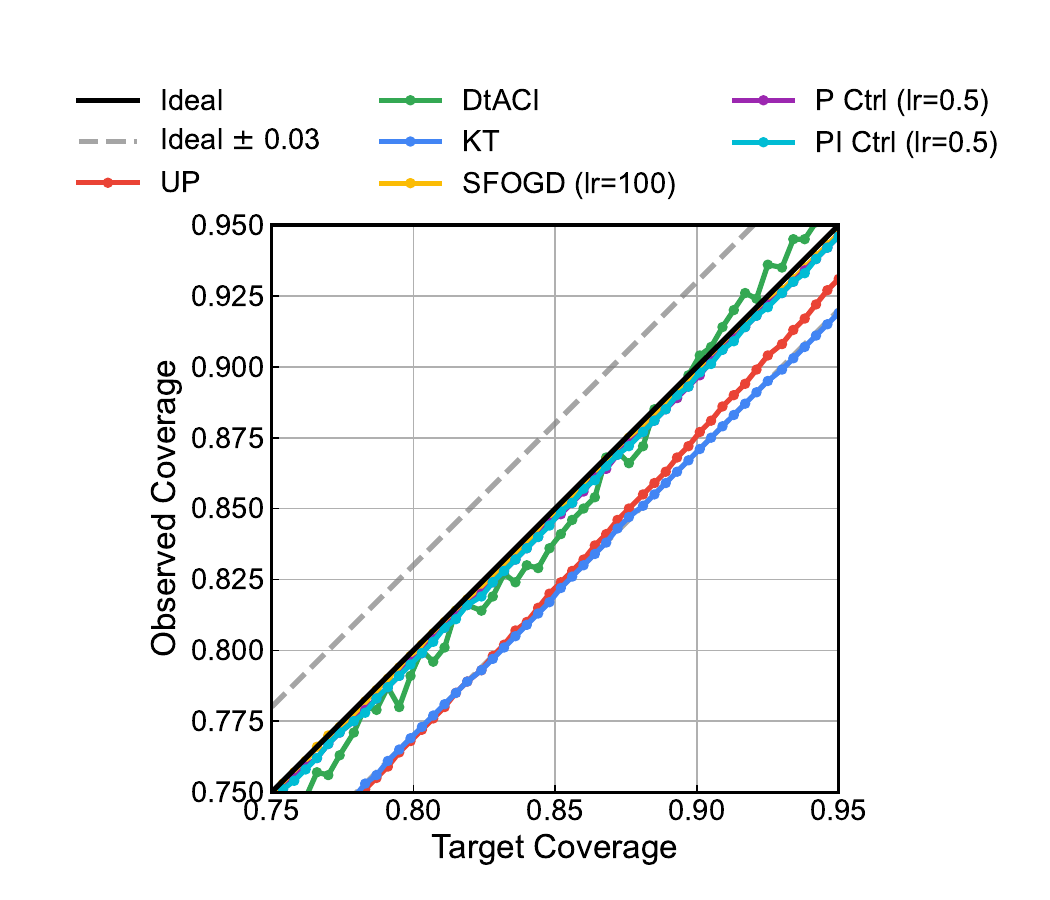}
    \caption{Realized vs. target coverage on AMZN. Most methods track the diagonal within a small tolerance ($\pm$ 0.03).}
    \label{fig:AMZN-tracking-targets}
  \end{minipage}
\end{figure}

\newpage
\subsection{GOOGL Dataset}
\textbf{Local Adaptivity.}
\begin{figure}[H]
  \centering
  \includegraphics[trim={0 0 0 1.2cm}, clip, width=0.85\columnwidth]{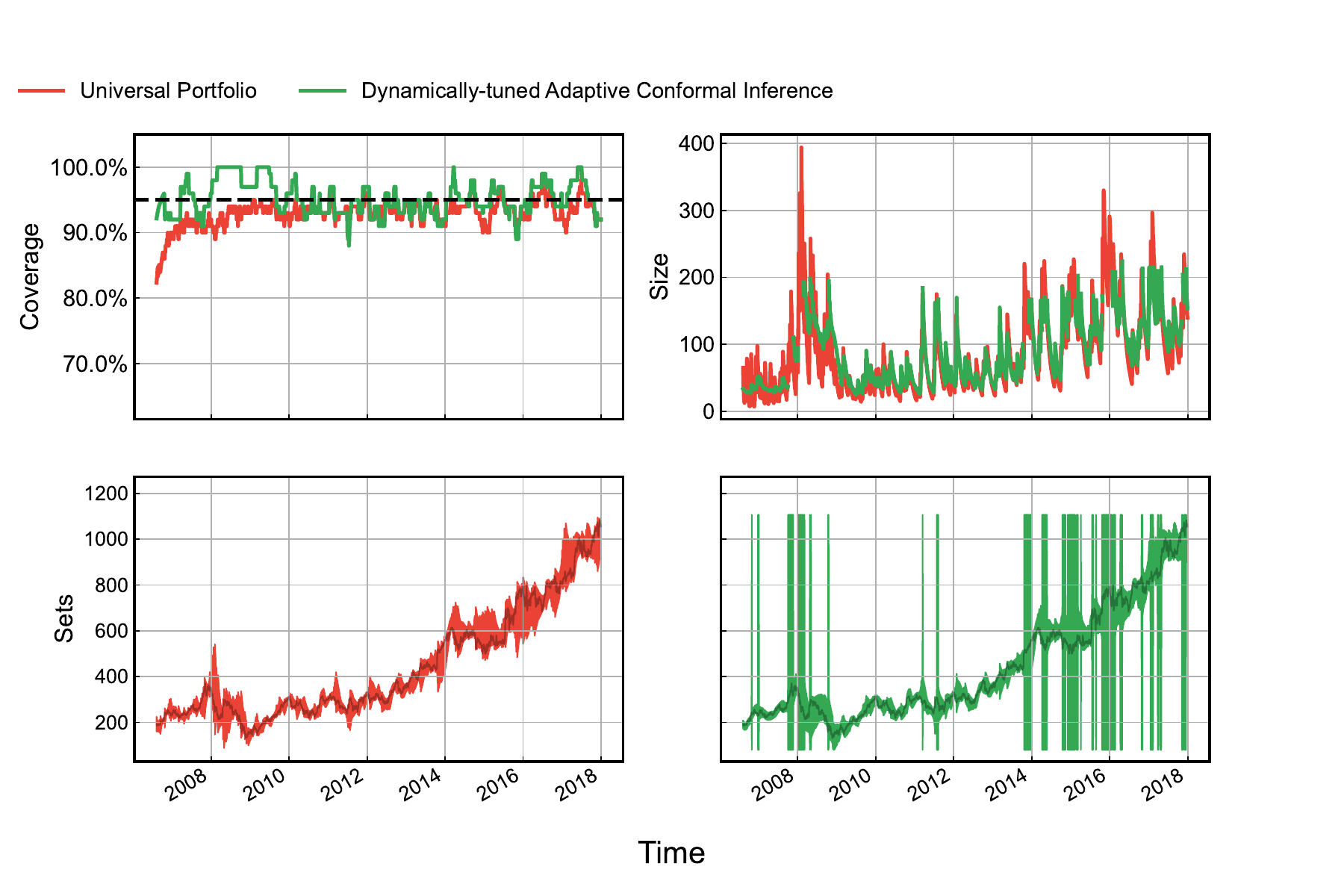}
  \caption{UP-OCP vs. DtACI for forecasting GOOGL stock return.}
  \label{fig:GOOGL-UP-vs-DtACI-local}
\end{figure} 

\begin{figure}[H]
  \centering
  \includegraphics[trim={0 0 0 1.2cm}, clip, width=0.85\columnwidth]{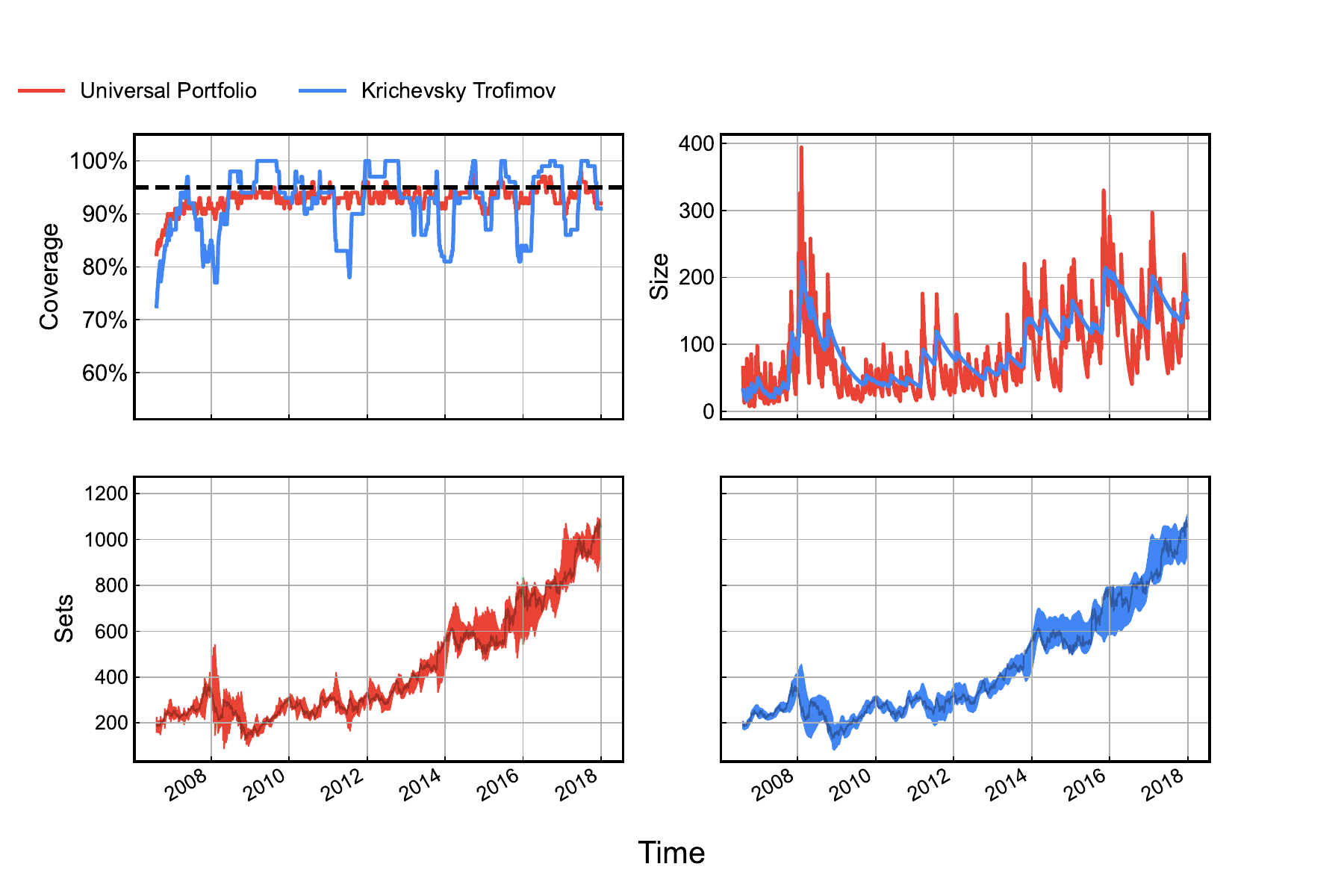}
  \caption{As in Figure~\ref{fig:GOOGL-UP-vs-DtACI-local}, UP-OCP vs. KT.}
  \label{fig:GOOGL-UP-vs-KT-local}
\end{figure} 

\begin{figure}[H]
  \centering
  \includegraphics[trim={0 0 0 1.2cm}, clip, width=0.85\columnwidth]{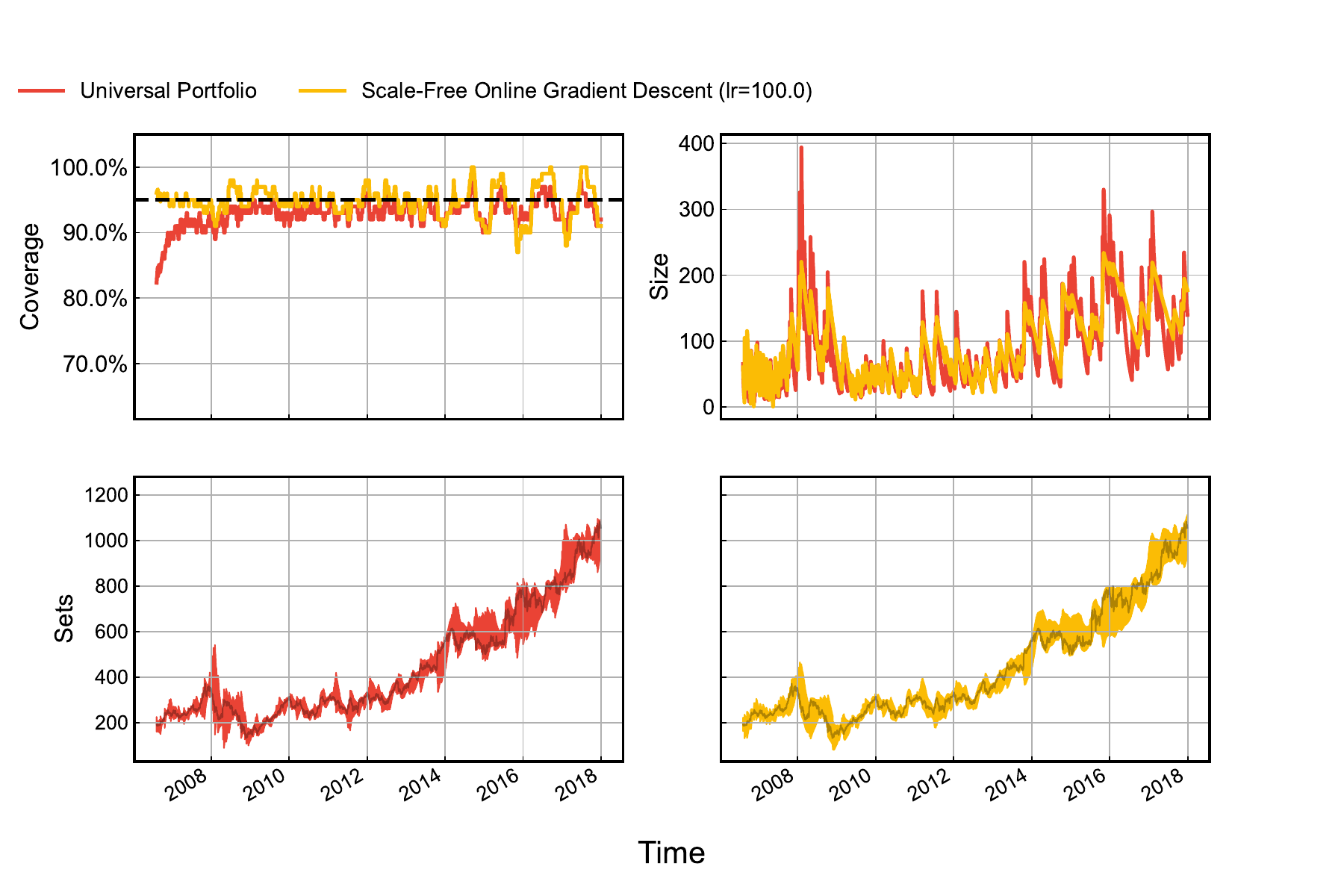}
  \caption{As in Figure~\ref{fig:GOOGL-UP-vs-DtACI-local}, UP-OCP vs. SFOGD (lr=100).}
  \label{fig:GOOGL-UP-vs-SGD-100-local}
\end{figure}

\begin{figure}[H]
  \centering
  \includegraphics[trim={0 0 0 1.2cm}, clip, width=0.85\columnwidth]{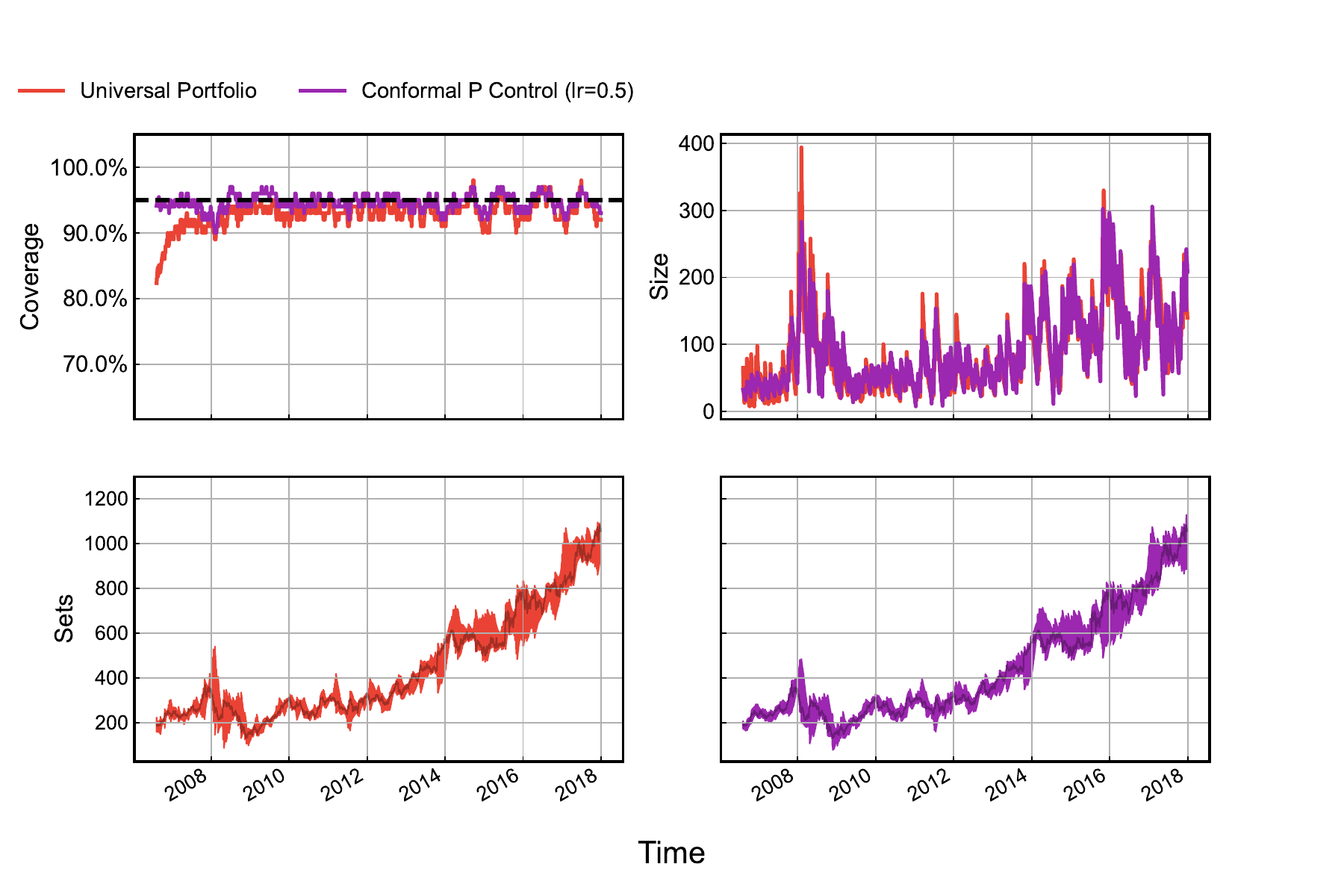}
  \caption{As in Figure~\ref{fig:GOOGL-UP-vs-DtACI-local}, UP-OCP vs. P Ctrl (lr=0.5).}
  \label{fig:GOOGL-UP-vs-P_05-local}
\end{figure}

\begin{figure}[H]
  \centering
  \includegraphics[trim={0 0 0 1.2cm}, clip, width=0.85\columnwidth]{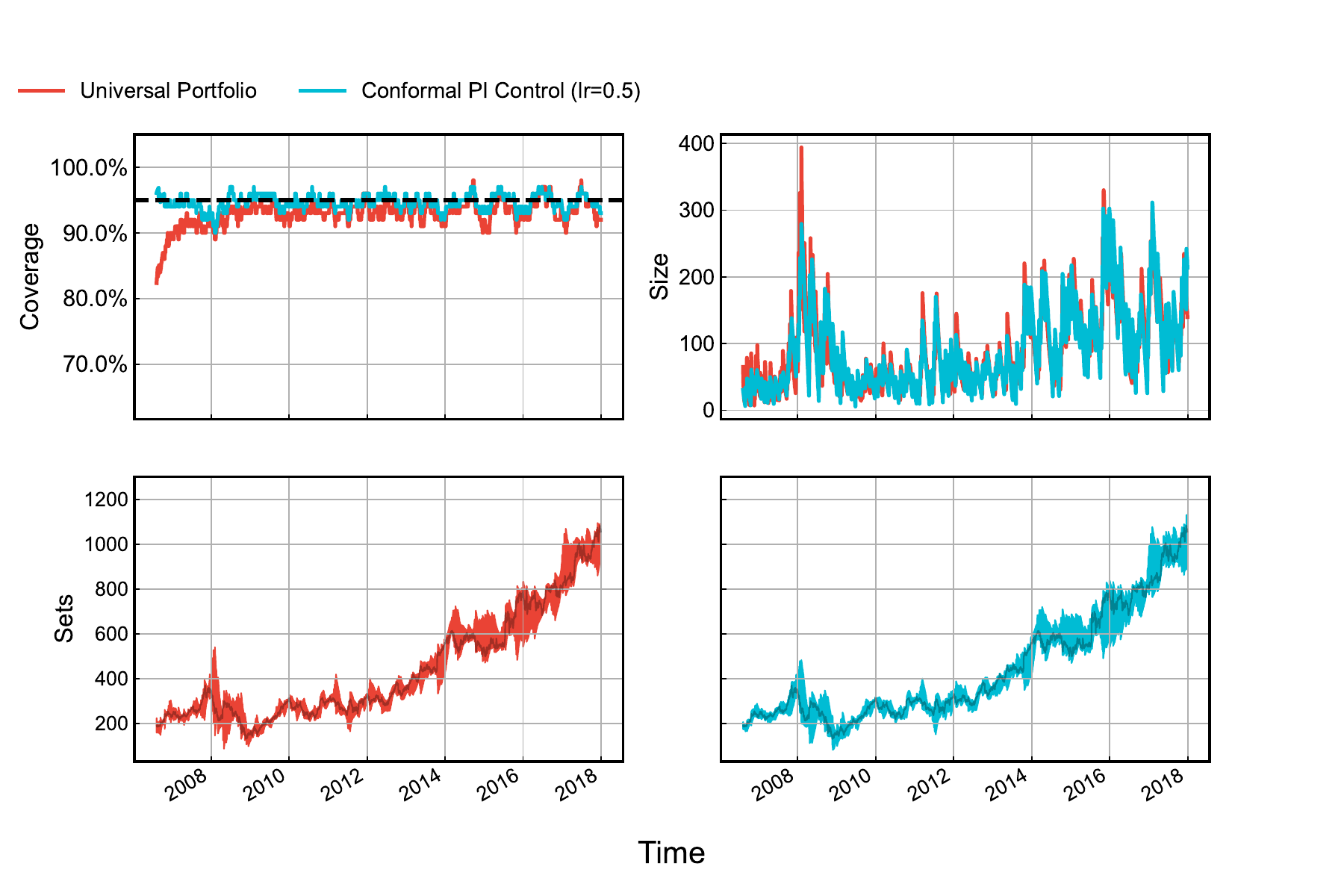}
  \caption{As in Figure~\ref{fig:GOOGL-UP-vs-DtACI-local}, UP-OCP vs. PI Ctrl (lr=0.5).}
  \label{fig:GOOGL-UP-vs-PI_05-local}
\end{figure}

\begin{table}[ht]
  \caption{Quantitative Comparison on the GOOGL Dataset.}
  \label{tab:GOOGL-1vall-metrics-full}
  \centering
  \begin{tabular}{lcccccc}
    \toprule
    & UP & KT & DtACI & SFOGD (lr=100) & P Ctrl (lr=0.5) & PI Ctrl (lr=0.5) \\
    \midrule
    Marginal coverage      & 0.932 & 0.925 & 0.952    & 0.948 & 0.946 & 0.946 \\
    Longest err sequence   & \textbf{2}     & 17    & 5        & 4     & \textbf{2}     & \textbf{2}     \\
    Average set size       & \textbf{86.7}  & 98.5  & $\infty$ & 95.8  & 91.8  & 90.4  \\
    Median set size        & \textbf{67.8}  & 90.6  & 79.2     & 86.1  & 74.4  & 72    \\
    75\% quantile set size & \textbf{121}   & 137   & 133      & 132   & 124   & 126   \\
    90\% quantile set size & \textbf{173}   & 169   & $\infty$ & \textbf{171}   & 176   & 178   \\
    95\% quantile set size & 204   & 188   & $\infty$ & 193   & 207   & 211   \\
    \bottomrule
  \end{tabular}
\end{table}

\newpage
\textbf{More Pareto Frontiers and Target-level Tracking.}

\begin{figure}[H]
  \centering
  % --- Top Row: 1x2 ---
  \begin{minipage}[t]{0.49\columnwidth}
    \centering
    \includegraphics[trim={0 0 0 0.5cm}, clip, width=\linewidth]{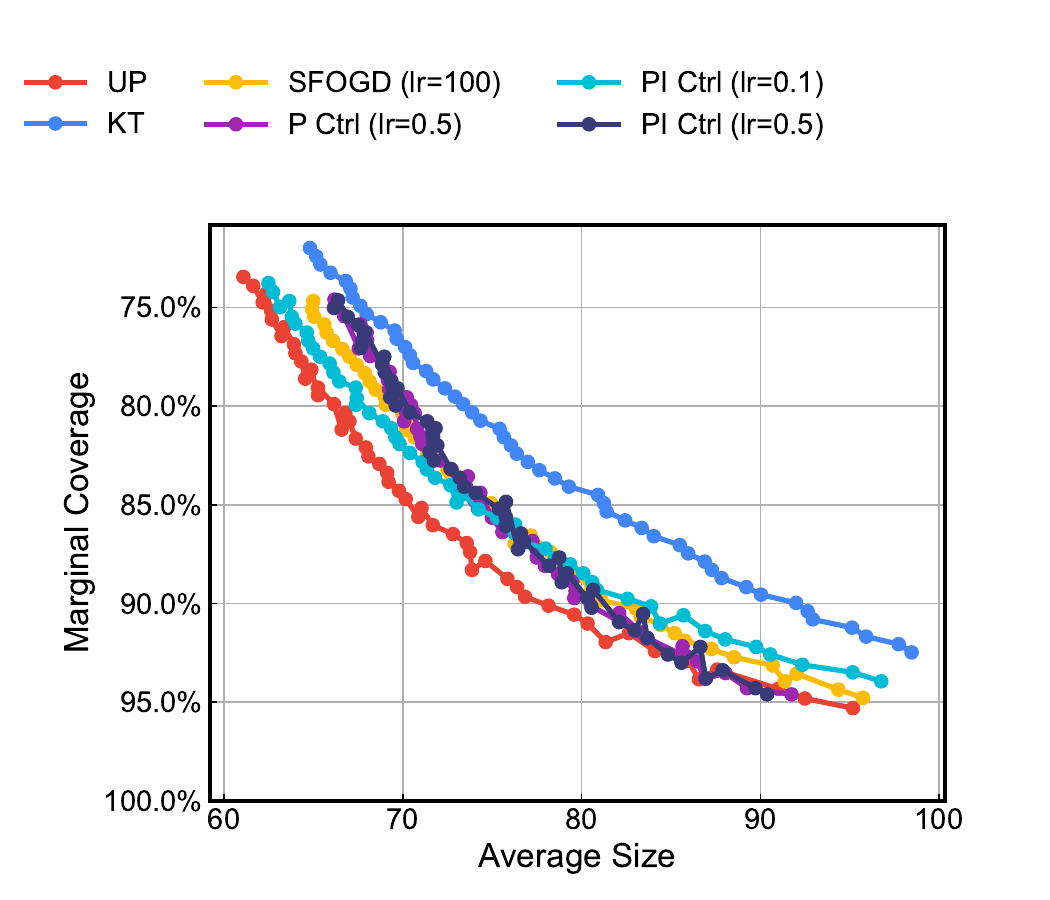}
    \caption{Mean prediction set sizes on GOOGL.}
    \label{fig:GOOGL-Pareto-average}
  \end{minipage}
  \hfill
  \begin{minipage}[t]{0.49\columnwidth}
    \centering
    \includegraphics[trim={0 0 0 0.5cm}, clip, width=\linewidth]{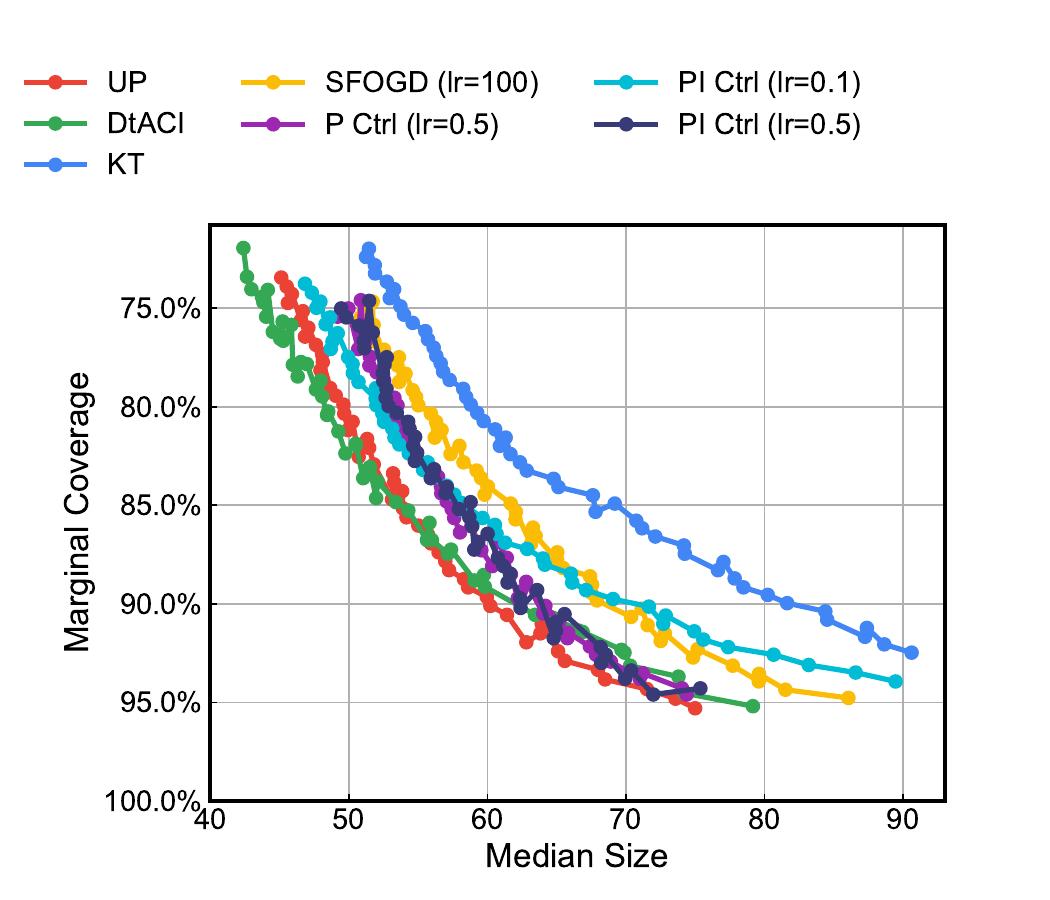}
    \caption{Median prediction set sizes on GOOGL.}
    \label{fig:GOOGL-Pareto-median}
  \end{minipage}
  
  \vspace{0.5cm}

  \begin{minipage}[t]{0.49\columnwidth}
    \centering
    \includegraphics[trim={0 0 0 0.5cm}, clip, width=\linewidth]{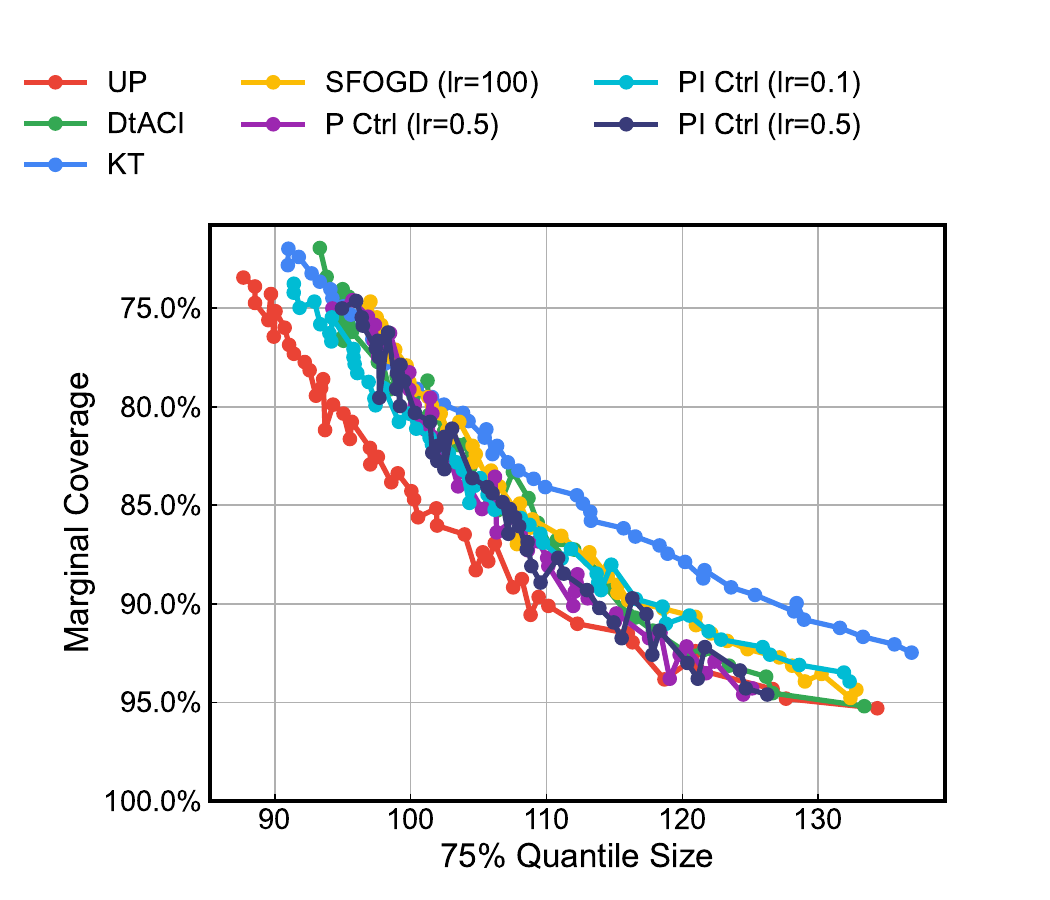}
    \caption{75\% quantile prediction set sizes on GOOGL.}
    \label{fig:GOOGL-Pareto-q75}
  \end{minipage}
  \hfill
  \begin{minipage}[t]{0.49\columnwidth}
    \centering
    \includegraphics[trim={0 0 0 1cm}, clip, width=\linewidth]{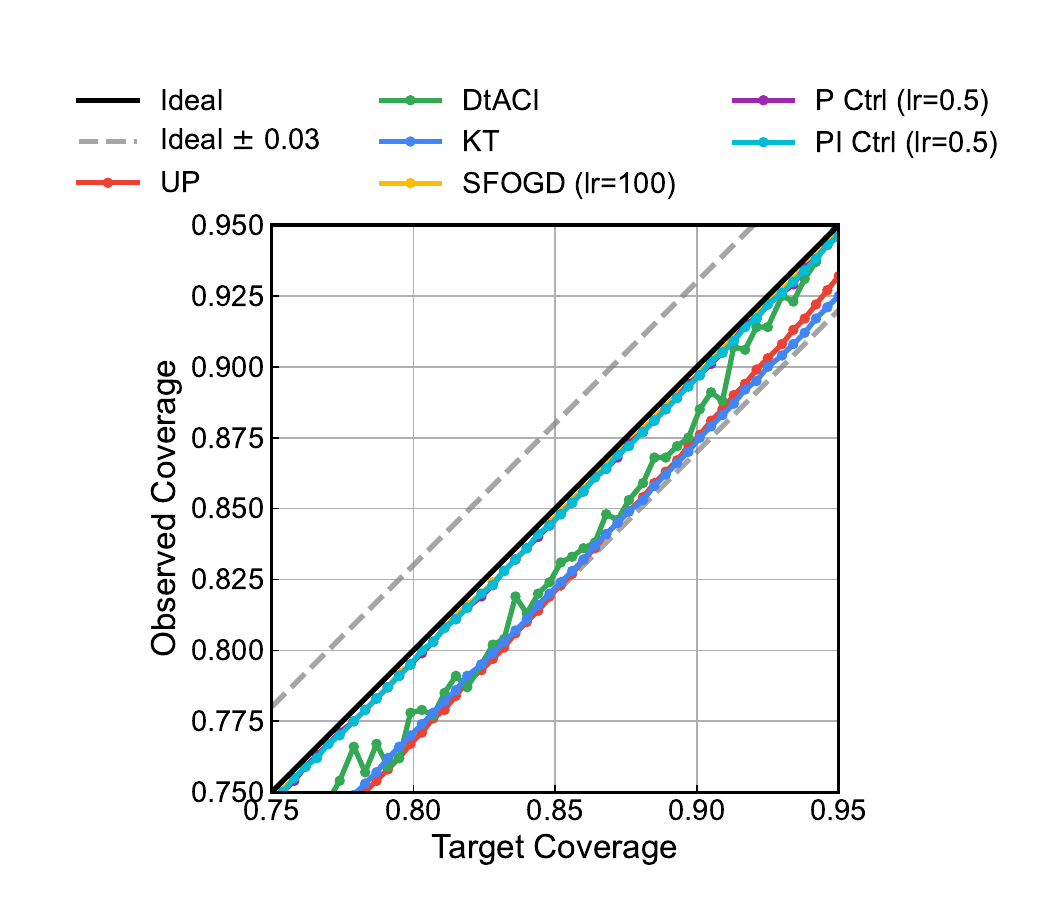}
    \caption{Realized vs. target coverage on GOOGL. Most methods track the diagonal within a small tolerance ($\pm$ 0.03).}
    \label{fig:GOOGL-tracking-targets}
  \end{minipage}
\end{figure}

\newpage
\subsection{Electricity Demand Dataset}

This data set measures electricity demand in New South Wales collected at half-hour increments from May
7th, 1996 to December 5th, 1998 (we zoom in on the first 2000 time points).

\textbf{Local Adaptivity.}
\begin{figure}[H]
  \centering
  \includegraphics[trim={0 0.6cm 0 1.7cm}, clip, width=0.85\columnwidth]{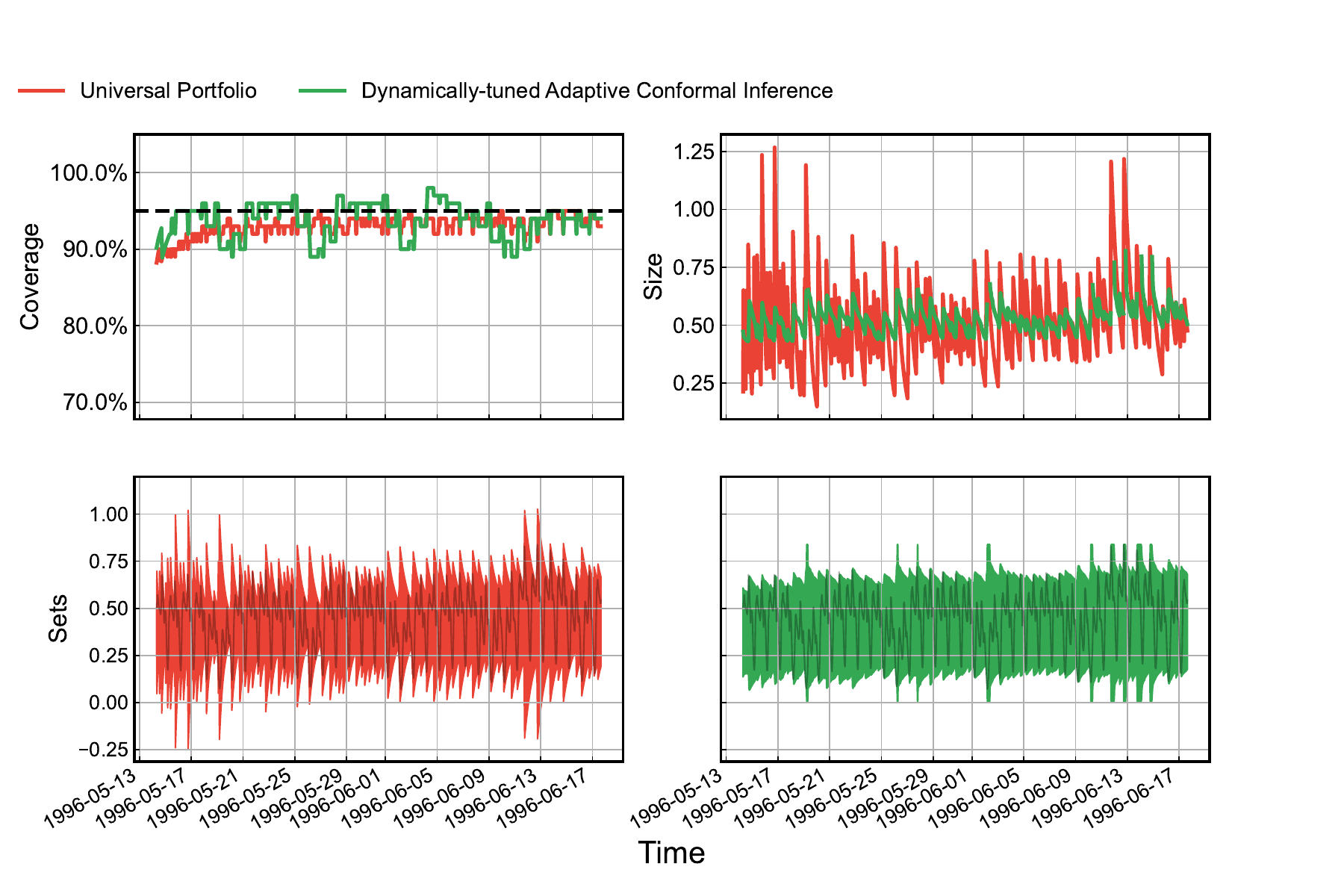}
  \caption{UP-OCP vs. DtACI for forecasting electricity demand.}
  \label{fig:elec2-UP-vs-DtACI-local}
\end{figure} 

\begin{figure}[H]
  \centering
  \includegraphics[trim={0 0.6cm 0 1.7cm}, clip, width=0.85\columnwidth]{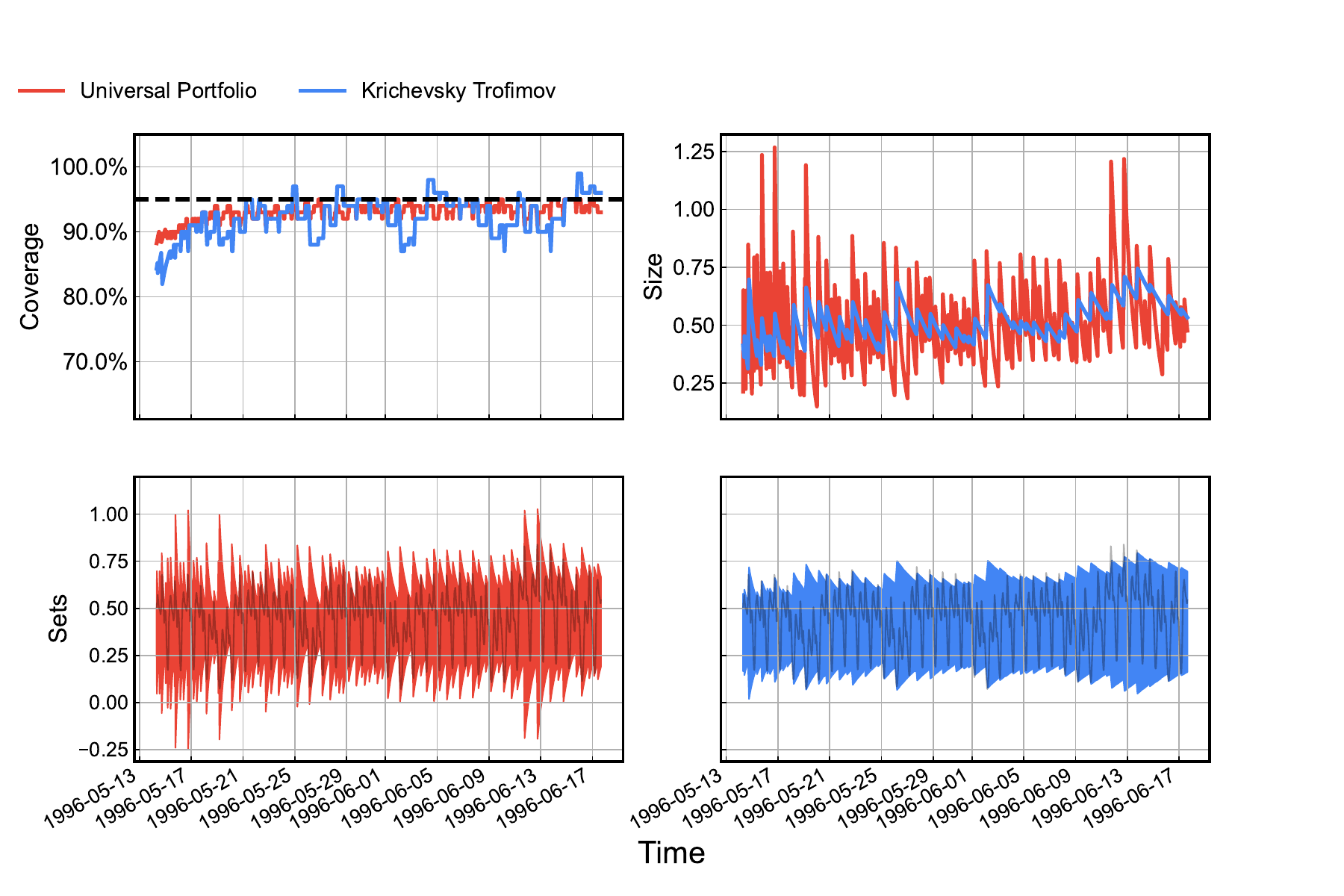}
  \caption{As in Figure~\ref{fig:elec2-UP-vs-DtACI-local}, UP-OCP vs. KT.}
  \label{fig:elec2-UP-vs-KT-local}
\end{figure} 

\begin{figure}[H]
  \centering
  \includegraphics[trim={0 0 0 1.2cm}, clip, width=0.85\columnwidth]{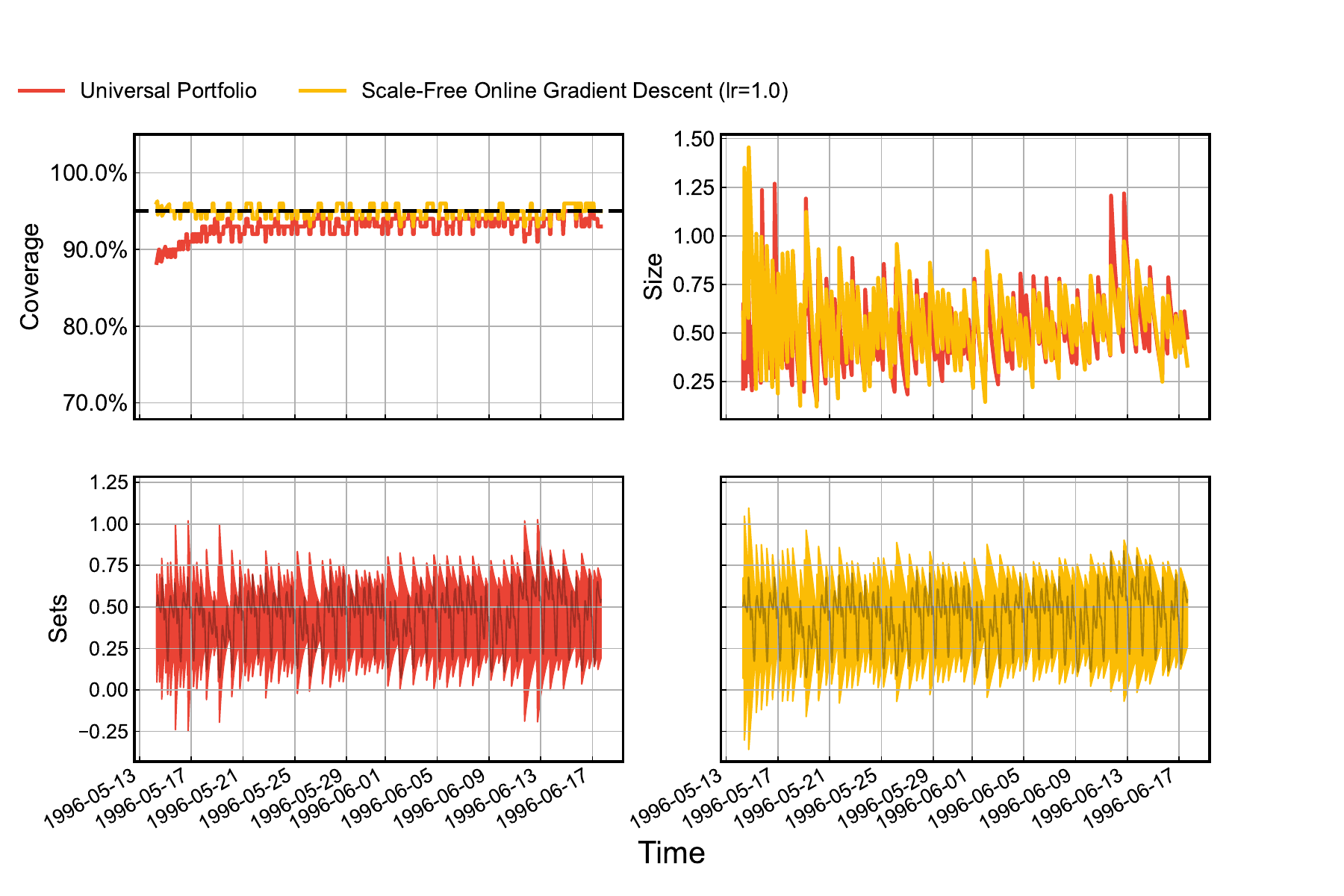}
  \caption{As in Figure~\ref{fig:elec2-UP-vs-DtACI-local}, UP-OCP vs. SFOGD (lr=1.0).}
  \label{fig:elec2-UP-vs-SGD-1-local}
\end{figure}

\begin{figure}[H]
  \centering
  \includegraphics[trim={0 0 0 1.2cm}, clip, width=0.85\columnwidth]{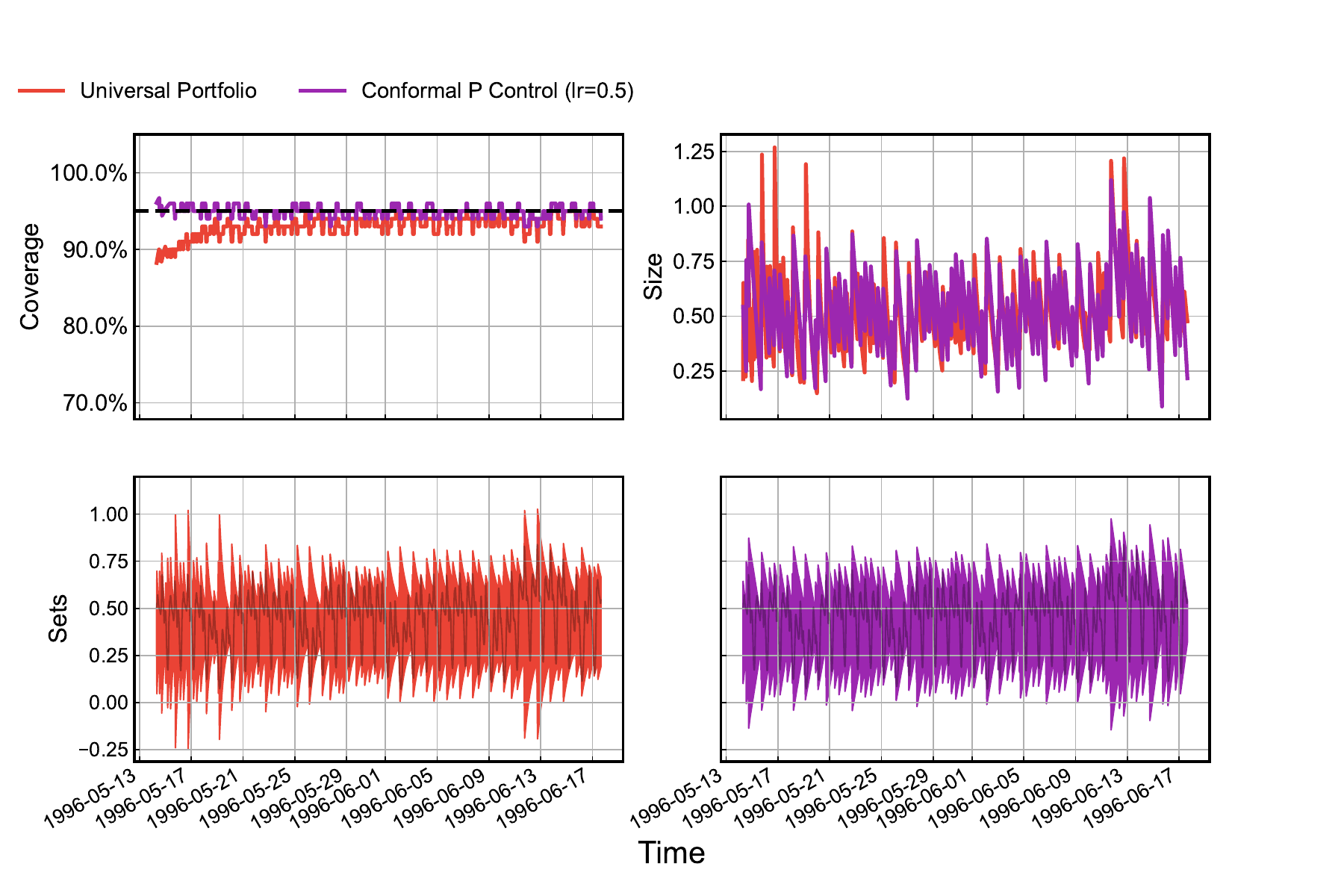}
  \caption{As in Figure~\ref{fig:elec2-UP-vs-DtACI-local}, UP-OCP vs. P Ctrl (lr=0.5).}
  \label{fig:elec2-UP-vs-P_05-local}
\end{figure}

\begin{figure}[H]
  \centering
  \includegraphics[trim={0 0 0 1.2cm}, clip, width=0.85\columnwidth]{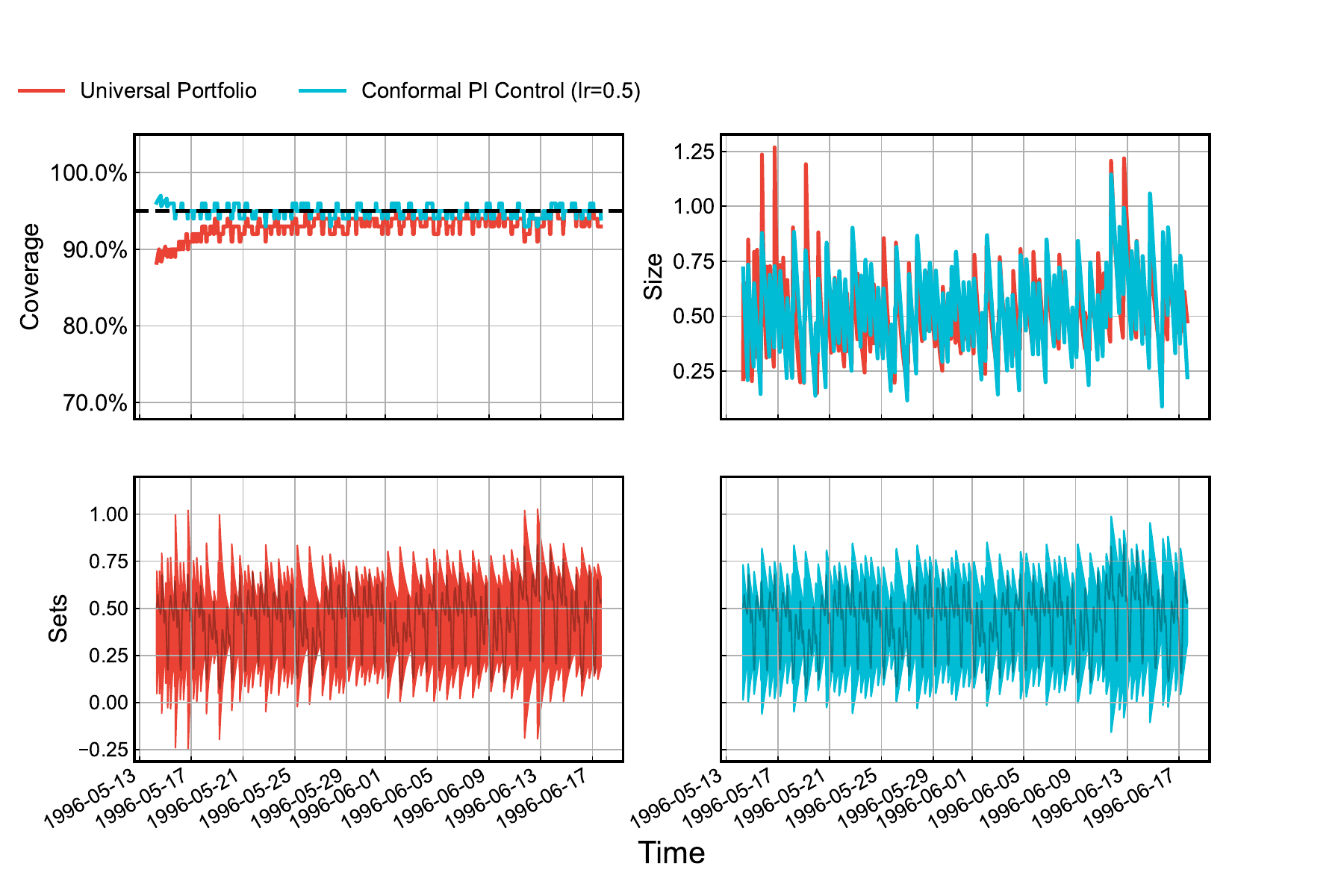}
  \caption{As in Figure~\ref{fig:elec2-UP-vs-DtACI-local}, UP-OCP vs. PI Ctrl (lr=0.5).}
  \label{fig:elec2-UP-vs-PI_05-local}
\end{figure}

\begin{table}[ht]
  \caption{Quantitative Comparison on the electricity demand dataset.}
  \label{tab:elec2-intermittent-1vall-metrics-full}
  \centering
  \begin{tabular}{lcccccc}
    \toprule
    & UP & KT & DtACI & SFOGD (lr=1.0) & P Ctrl (lr=0.5) & PI Ctrl (lr=0.5) \\
    \midrule
    Marginal coverage      & 0.933 & 0.927 & 0.939    & 0.95  & 0.95  & 0.95  \\
    Longest err sequence   & \textbf{3}     & 7     & 6        & \textbf{2}     & \textbf{1}     & \textbf{2}     \\
    Average set size       & \textbf{0.507} & 0.518 & $\infty$ & 0.561 & 0.528 & 0.533 \\
    Median set size        & \textbf{0.487} & 0.506 & 0.53     & 0.552 & 0.524 & 0.525 \\
    75\% quantile set size & 0.593 & 0.572 & 0.563    & 0.664 & 0.637 & 0.65  \\
    90\% quantile set size & 0.715 & 0.635 & 0.619    & 0.772 & 0.742 & 0.748 \\
    95\% quantile set size & 0.793 & 0.664 & 0.671    & 0.862 & 0.806 & 0.821 \\
    \bottomrule
  \end{tabular}
\end{table}

\newpage
\textbf{More Pareto Frontiers and Target-level Tracking.}

\begin{figure}[H]
  \centering
  % --- Top Row: 1x2 ---
  \begin{minipage}[t]{0.49\columnwidth}
    \centering
    \includegraphics[trim={0 0 0 0.5cm}, clip, width=\linewidth]{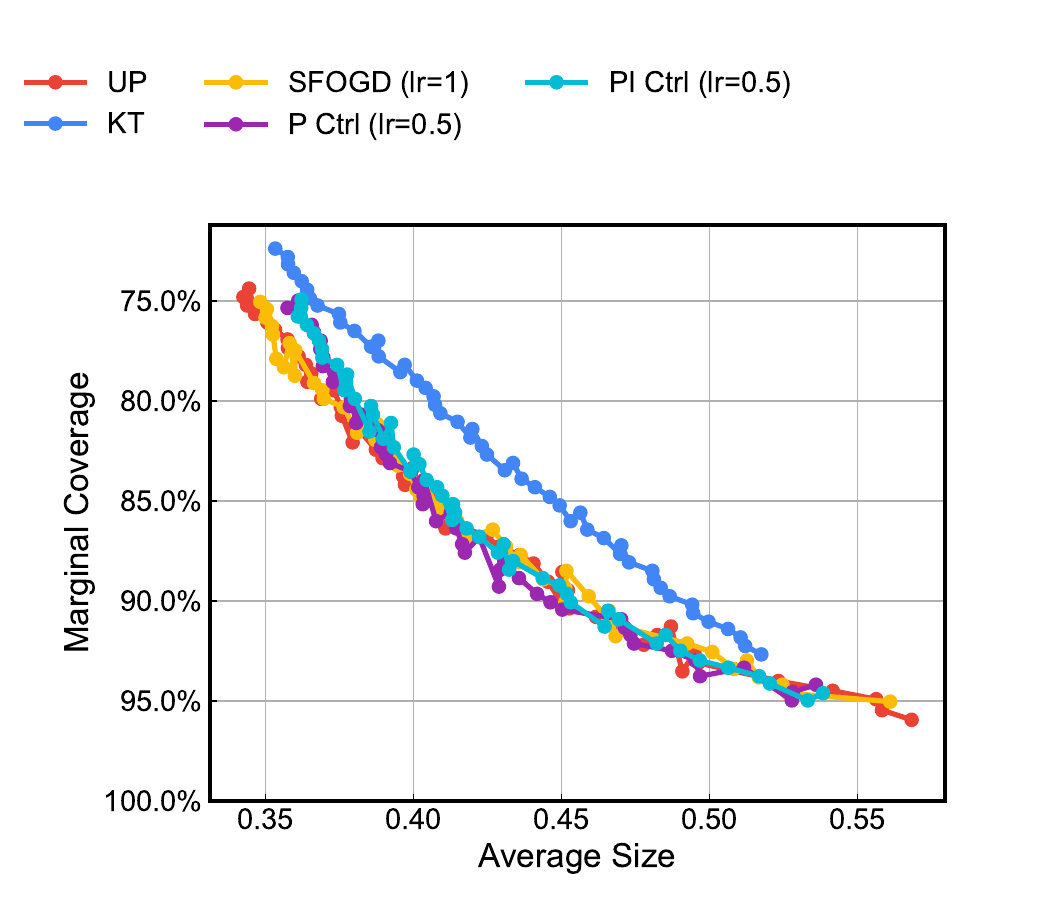}
    \caption{Mean prediction set sizes on electricity demand.}
    \label{fig:elec2-intermittent-Pareto-average}
  \end{minipage}
  \hfill
  \begin{minipage}[t]{0.49\columnwidth}
    \centering
    \includegraphics[trim={0 0 0 0.5cm}, clip, width=\linewidth]{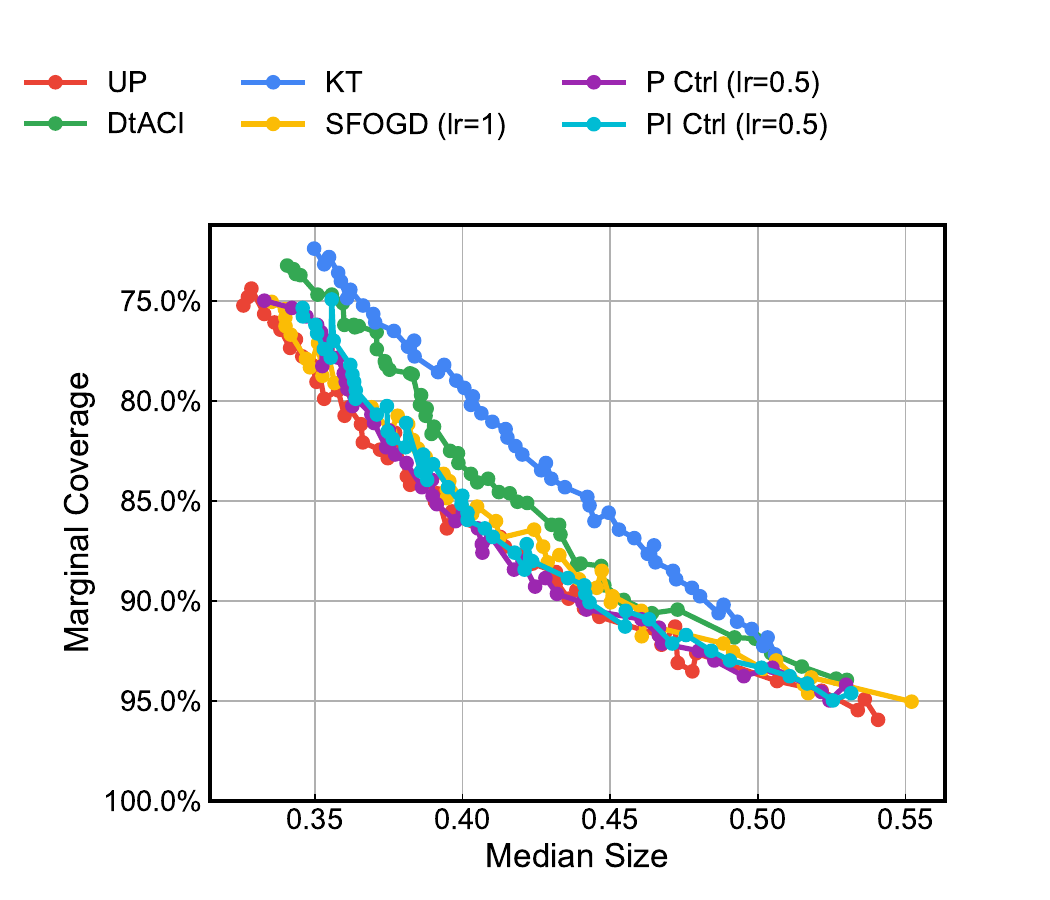}
    \caption{Median prediction set sizes on electricity demand.}
    \label{fig:elec2-intermittent-Pareto-median}
  \end{minipage}
  
  \vspace{0.5cm}

  \begin{minipage}[t]{0.49\columnwidth}
    \centering
    \includegraphics[trim={0 0 0 0.5cm}, clip, width=\linewidth]{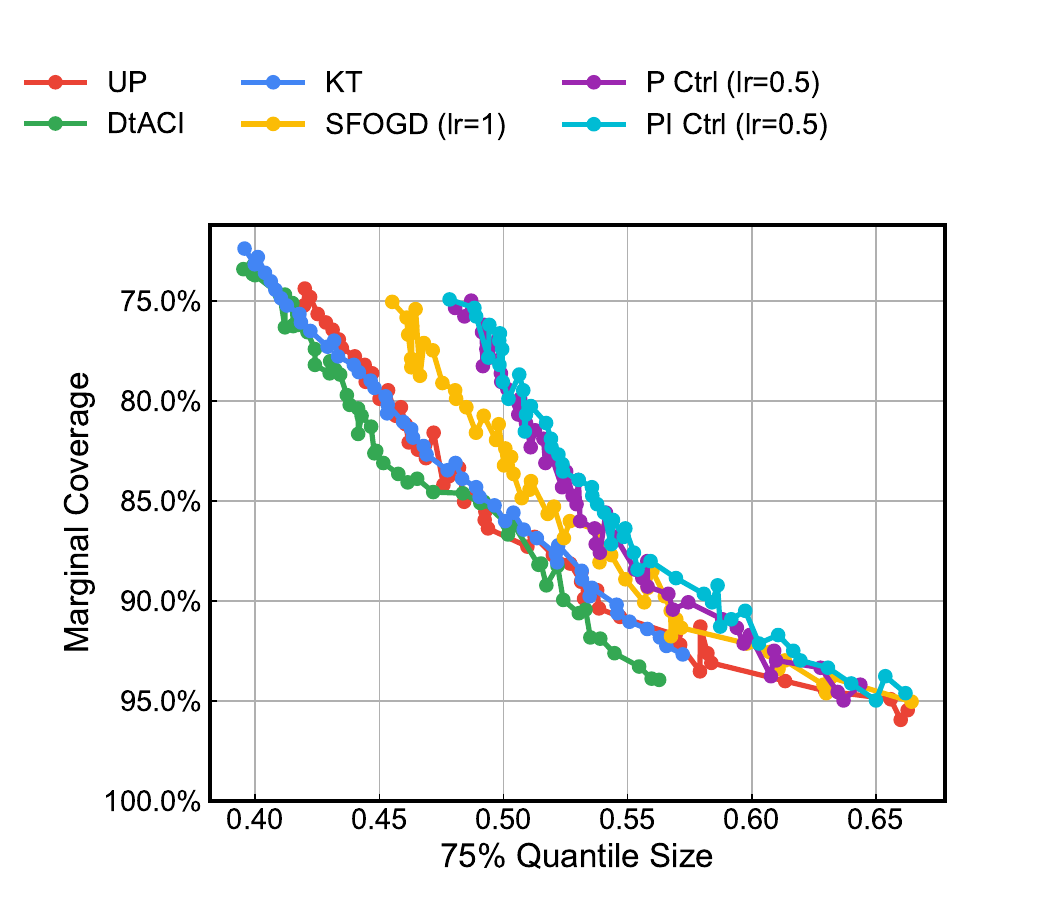}
    \caption{75\% quantile prediction set sizes on electricity demand.}
    \label{fig:elec2-intermittent-Pareto-q75}
  \end{minipage}
  \hfill
  \begin{minipage}[t]{0.49\columnwidth}
    \centering
    \includegraphics[trim={0 0 0 1cm}, clip, width=\linewidth]{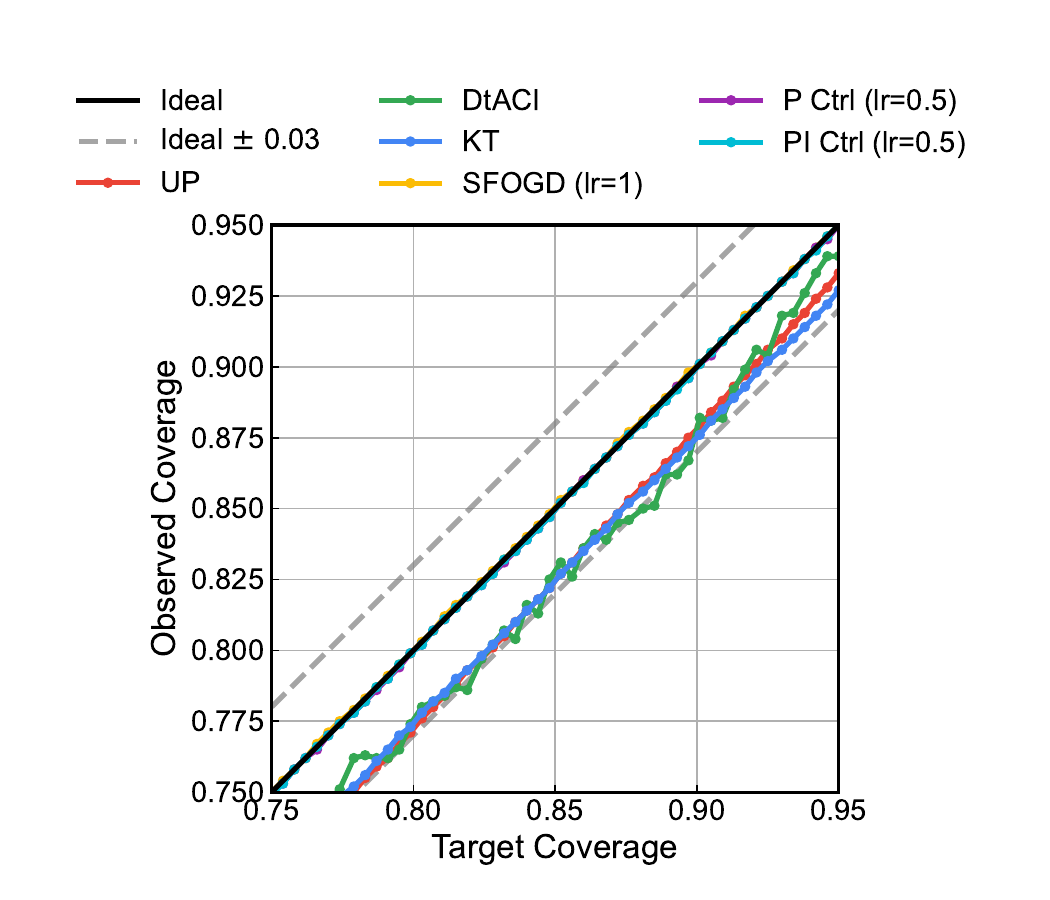}
    \caption{Realized vs. target coverage on electricity demand. Most methods track the diagonal within a small tolerance ($\pm$ 0.03).}
    \label{fig:elec2-intermittent-tracking-targets}
  \end{minipage}
\end{figure}

\newpage
\subsection{Synthetic Sinusoid}
\label{app:sinusoid}
We generate the nonconformity scores $S_t$ as a sinusoidal wave distorted by Gaussian noise. Formally, for $t=1, \dots, T$,
\[
    S_t = \max\left(0, \left[\sin\left(\frac{2\pi t}{P}\right) + 0.5\right] S_{\text{mag}} + S_{\min} + \epsilon_t \right),
\]
where $\epsilon_t \mathop{\sim}\limits^{\text{i.i.d.}} \mathcal{N}(0, \sigma^2)$. We fix the period $P=200$, the magnitude scaler $S_{\text{mag}}=10$, and the minimum offset $S_{\min}=2$. The noise scale is set to $\sigma=0.3$. The total sequence length is $T=3000$.

\textbf{Local Adaptivity.}
\begin{figure}[H]
  \centering
  \includegraphics[trim={0 0 0 1.2cm}, clip, width=0.85\columnwidth]{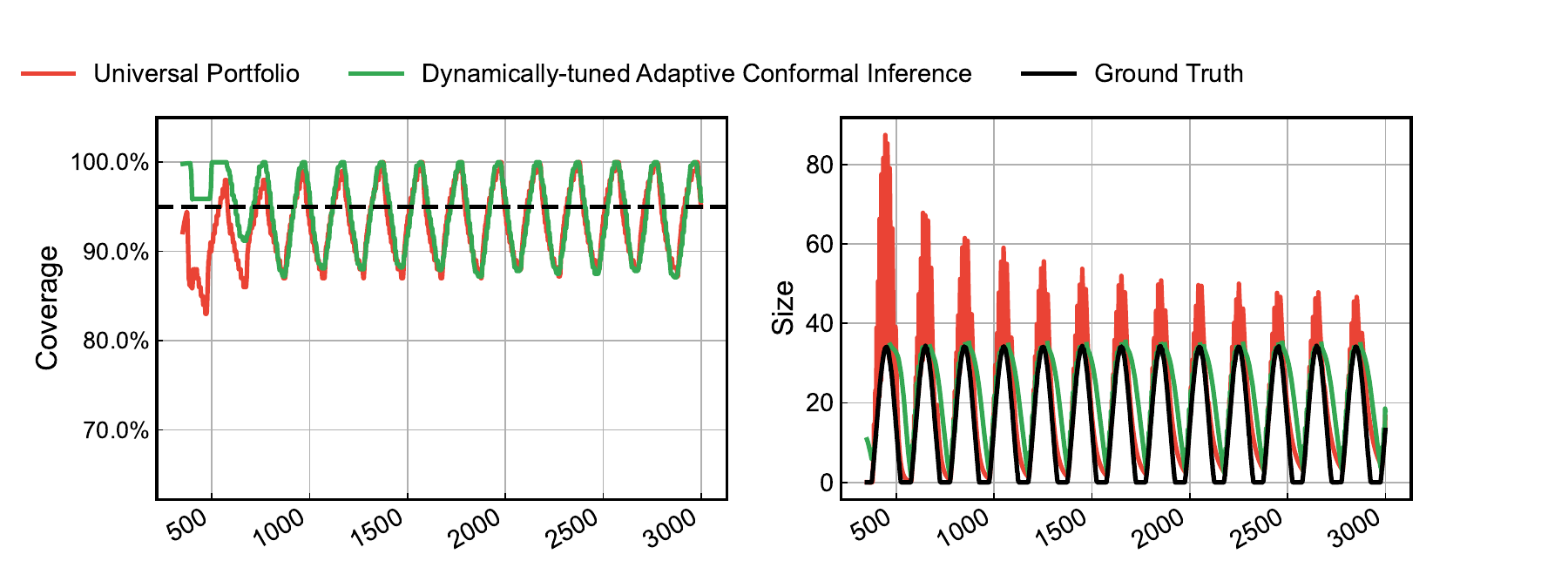}
  \caption{UP-OCP vs. DtACI for forecasting synthetic sinusoid data.}
  \label{fig:sinusoid-UP-vs-DtACI-local}
\end{figure} 

\begin{figure}[H]
  \centering
  \includegraphics[trim={0 0 0 1.2cm}, clip, width=0.85\columnwidth]{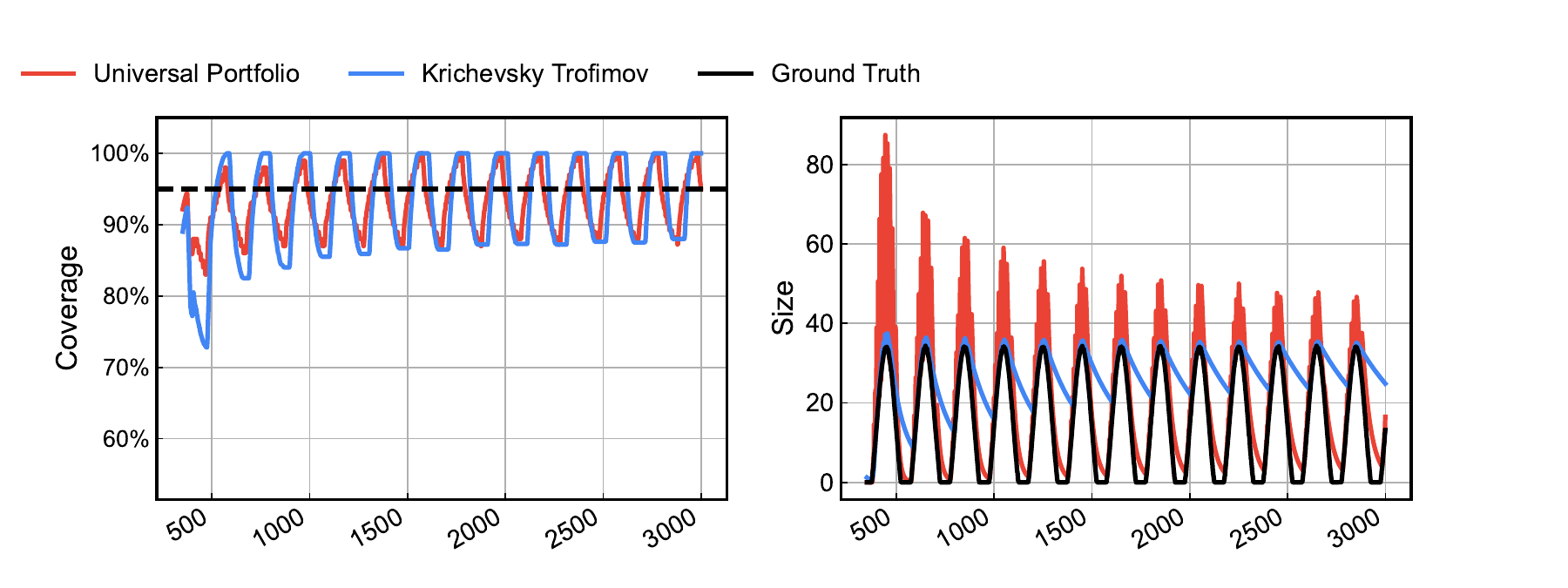}
  \caption{As in Figure~\ref{fig:sinusoid-UP-vs-DtACI-local}, UP-OCP vs. KT.}
  \label{fig:sinusoid-UP-vs-KT-local}
\end{figure} 

\begin{figure}[H]
  \centering
  \includegraphics[trim={0 0 0 1.2cm}, clip, width=0.85\columnwidth]{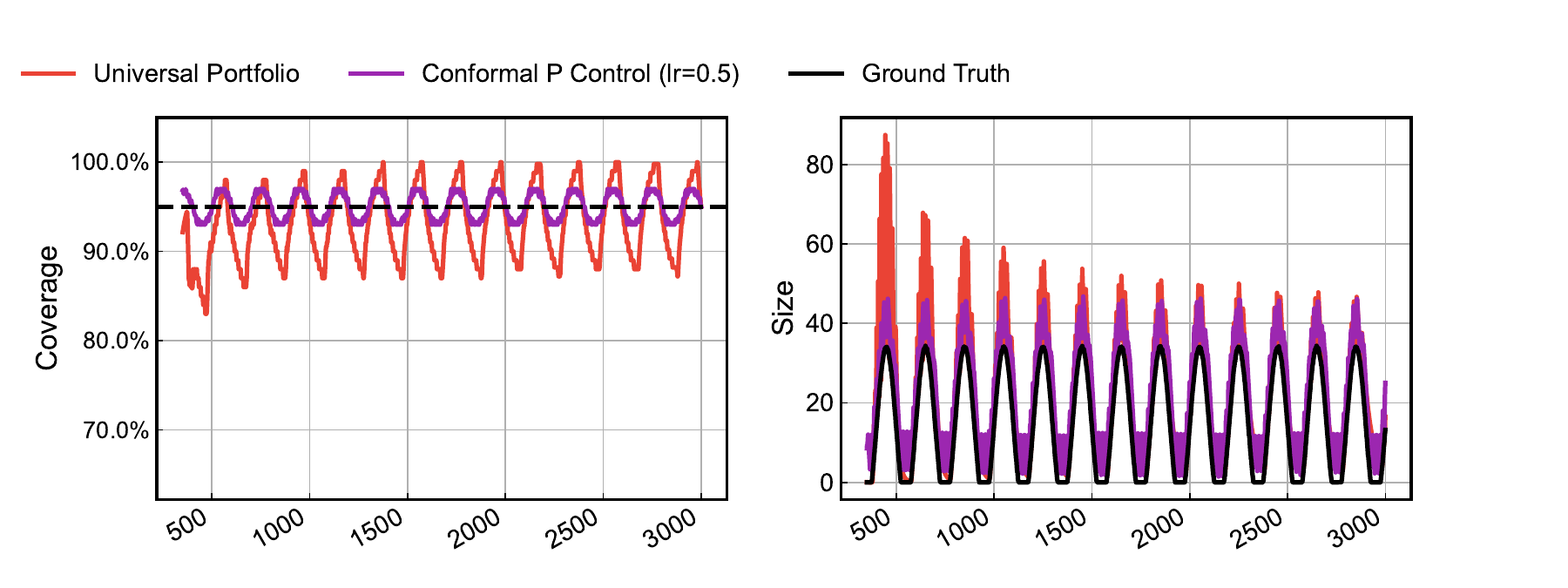}
  \caption{As in Figure~\ref{fig:sinusoid-UP-vs-DtACI-local}, UP-OCP vs. P Ctrl (lr=0.5).}
  \label{fig:sinusoid-UP-vs-P_05-local}
\end{figure}

\begin{figure}[H]
  \centering
  \includegraphics[trim={0 0 0 1.2cm}, clip, width=0.85\columnwidth]{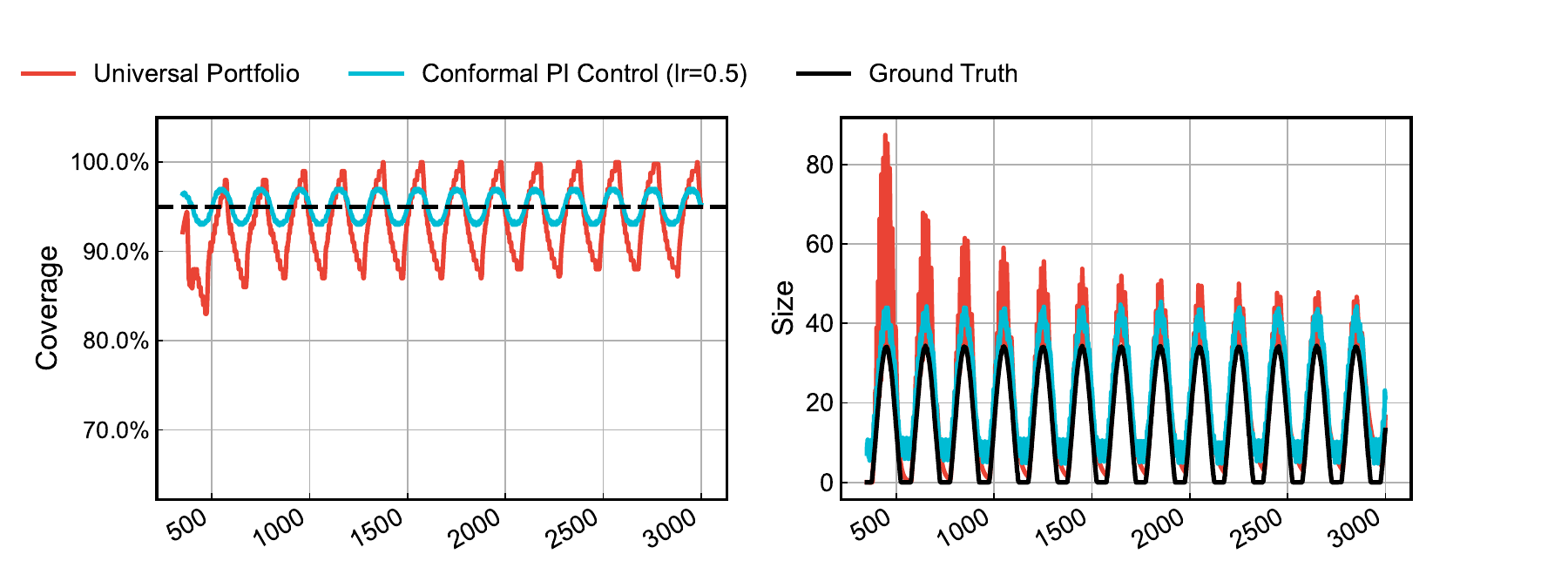}
  \caption{As in Figure~\ref{fig:sinusoid-UP-vs-DtACI-local}, UP-OCP vs. PI Ctrl (lr=0.5).}
  \label{fig:sinusoid-UP-vs-PI_05-local}
\end{figure}

\begin{figure}[H]
  \centering
  \includegraphics[trim={0 0 0 1.2cm}, clip, width=0.85\columnwidth]{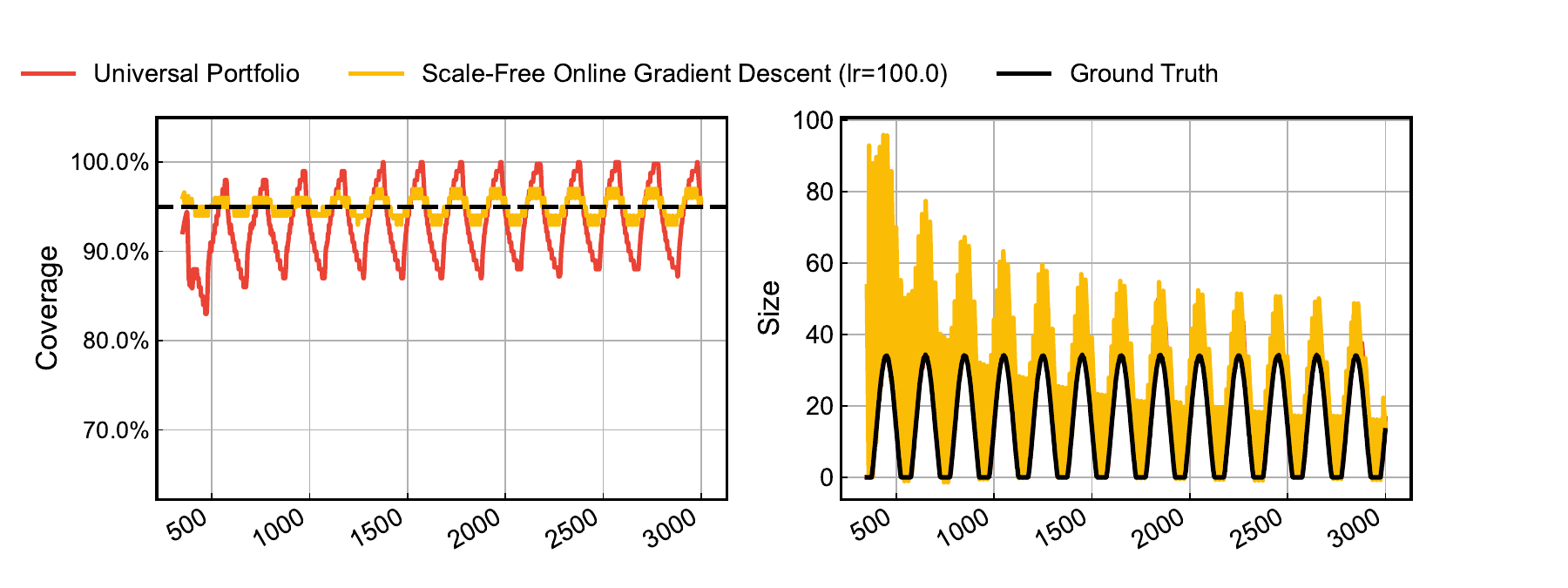}
  \caption{As in Figure~\ref{fig:sinusoid-UP-vs-DtACI-local}, UP-OCP vs. SFOGD (lr=100).}
  \label{fig:sinusoid-UP-vs-SFOGD_100-local}
\end{figure}

\begin{table}[H]
  \caption{Quantitative Comparison on the Sinusoid Dataset (synthetic).}
  \label{tab:sinusoid-1vall-metrics-full}
  \centering
  \begin{tabular}{lcccccc}
    \toprule
    & UP & KT & DtACI & SFOGD (lr=100) & P Ctrl (lr=0.5) & PI Ctrl (lr=0.5) \\
    \midrule
    Marginal coverage      & 0.931 & 0.928 & 0.942    & 0.95  & 0.95  & 0.95  \\
    Average set size       & \textbf{21.1}  & 26.9  & $\infty$ & 28.6  & 22.8  & 22.8  \\
    Median set size        & \textbf{17.9}  & 27.6  & 26.4     & 27.7  & 21.9  & 21.5  \\
    75\% quantile set size & 34.1  & 32.1  & 33.2     & 40    & 34.9  & 34.9  \\
    90\% quantile set size & 43.2  & 34.5  & $\infty$ & 50.4  & 41.8  & 41.9  \\
    95\% quantile set size & 48.3  & 35.3  & $\infty$ & 57.8  & 45    & 45.1  \\
    \bottomrule
  \end{tabular}
\end{table}

\newpage
\textbf{More Pareto Frontiers and Target-level Tracking.}

\begin{figure}[H]
  \centering
  % --- Top Row: 1x2 ---
  \begin{minipage}[t]{0.49\columnwidth}
    \centering
    \includegraphics[trim={0 0 0 0.5cm}, clip, width=\linewidth]{sinusoid_mean_burnin50_pareto_frontier.pdf}
    \caption{Mean prediction set sizes on synthetic sinusoid data.}
    \label{fig:sinusoid-Pareto-average}
  \end{minipage}
  \hfill
  \begin{minipage}[t]{0.49\columnwidth}
    \centering
    \includegraphics[trim={0 0 0 0.5cm}, clip, width=\linewidth]{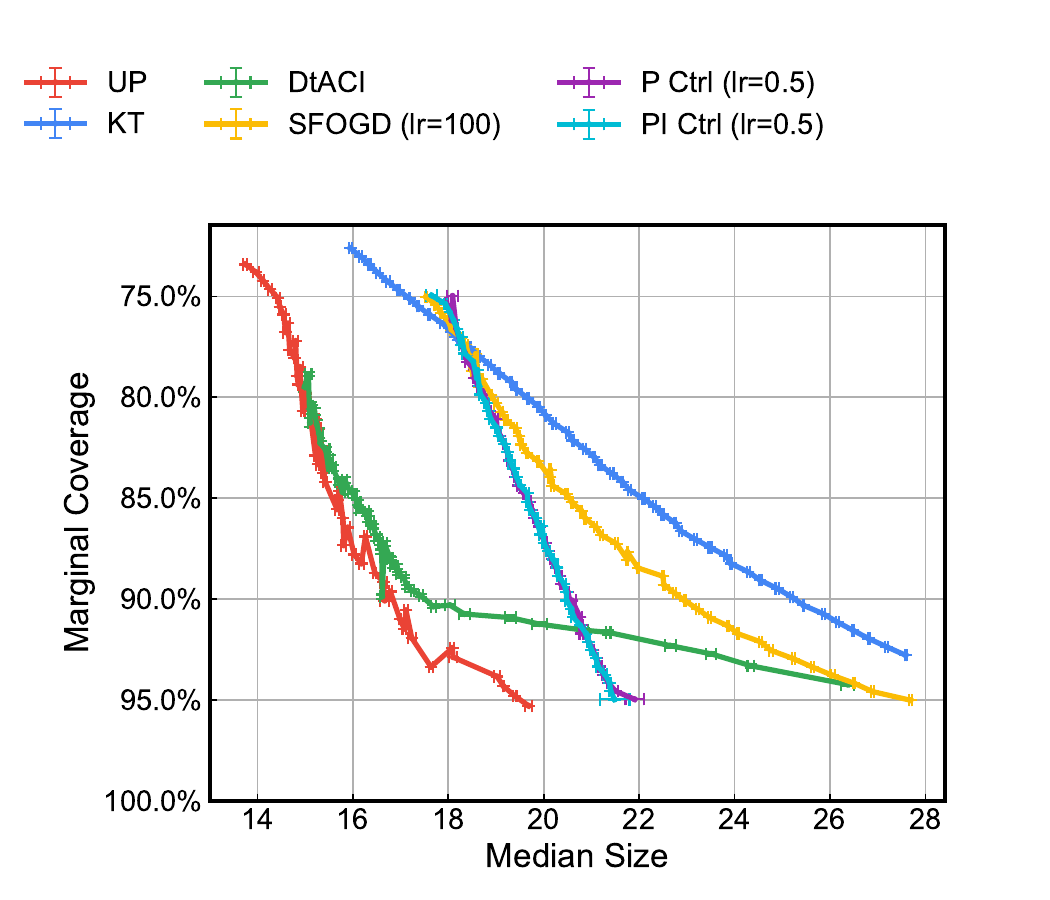}
    \caption{Median prediction set sizes on synthetic sinusoid data.}
    \label{fig:sinusoid-Pareto-median}
  \end{minipage}
  
  \vspace{0.5cm}

  \begin{minipage}[t]{0.49\columnwidth}
    \centering
    \includegraphics[trim={0 0 0 0.5cm}, clip, width=\linewidth]{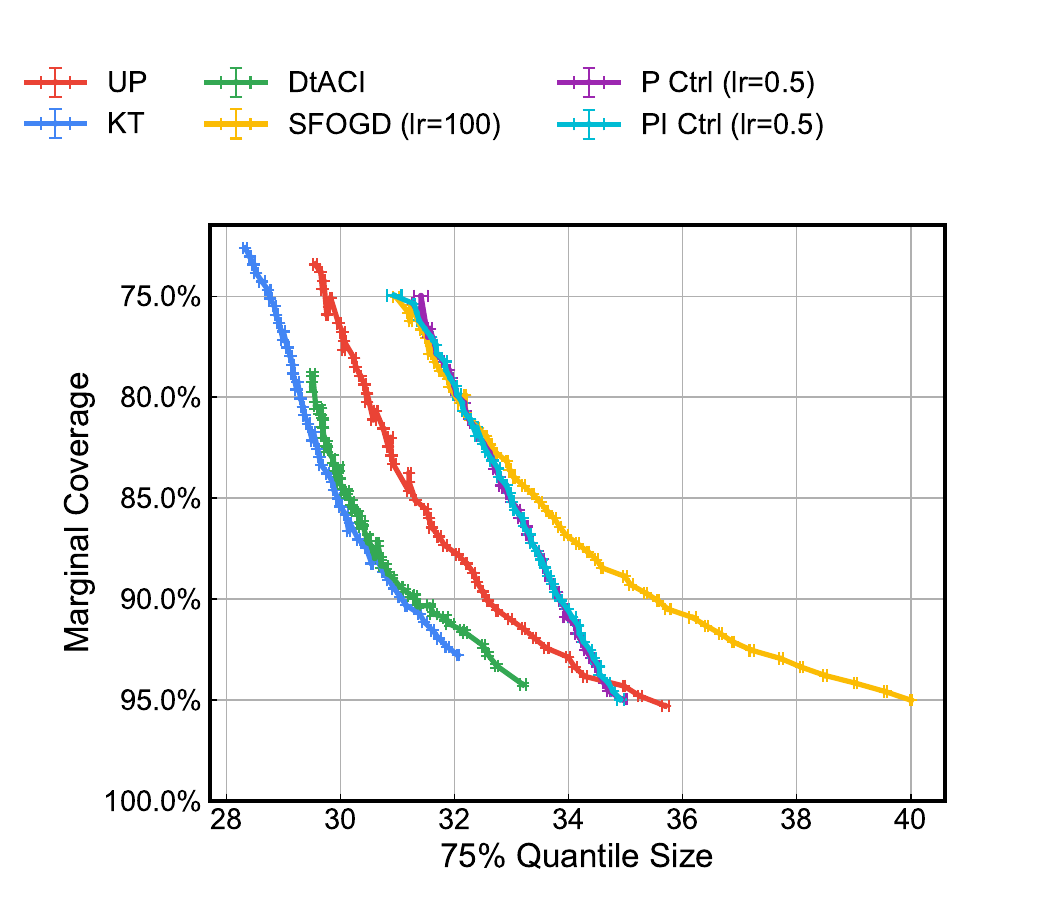}
    \caption{75\% quantile prediction set sizes on synthetic sinusoid data.}
    \label{fig:sinusoid-Pareto-q75}
  \end{minipage}
  \hfill
  \begin{minipage}[t]{0.49\columnwidth}
    \centering
    \includegraphics[trim={0 0 0 1cm}, clip, width=\linewidth]{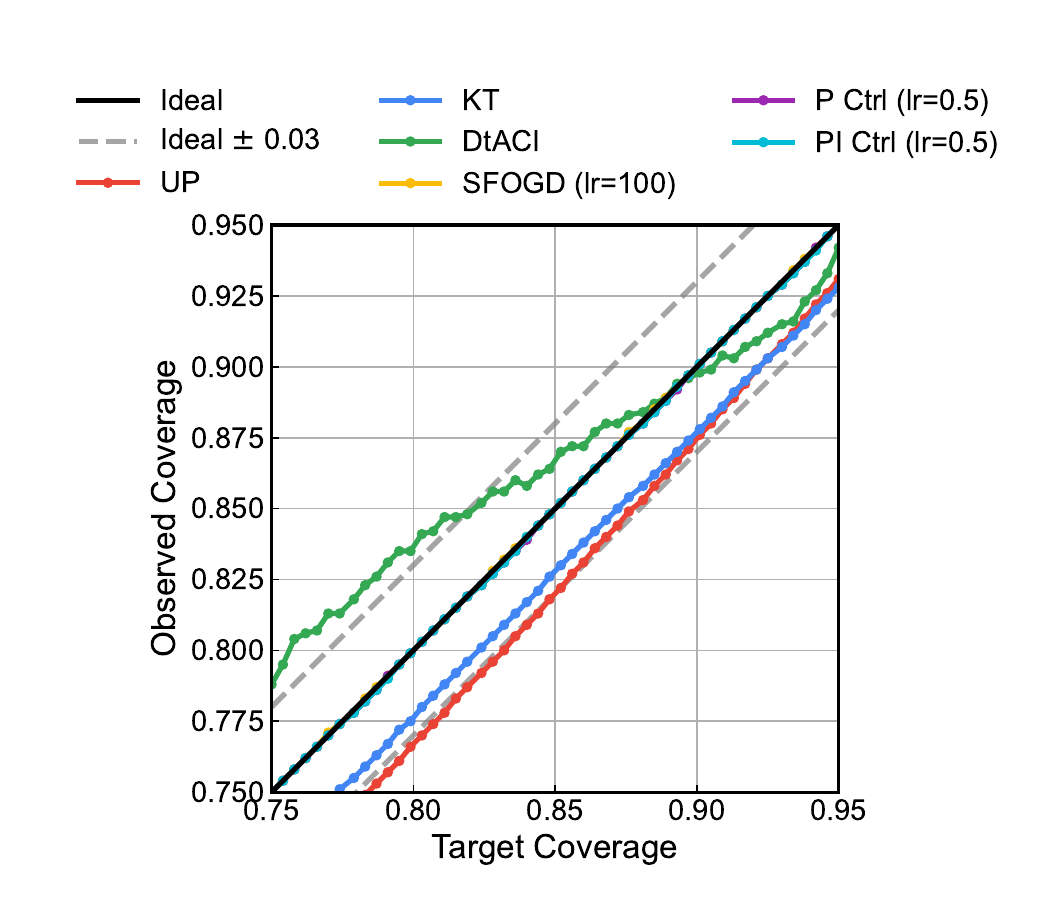}
    \caption{Realized vs. target coverage on synthetic sinusoid data. Most methods track the diagonal within a small tolerance ($\pm$ 0.03).}
    \label{fig:sinusoid-tracking-targets}
  \end{minipage}
\end{figure}

\newpage
\subsection{Stationary Trend with Random waves}

Here we try to examine robustness against sparse, heavy-tailed random waves, in comparison with fixed wave positions with periodicity for sinusoid. The nonconformity scores $S_t$ are generated via a process involving a constant baseline, sparse exponential noise, and a rolling window to create wavelet structures.

Formally, for $t=1, \dots, T$, we define a constant baseline $C=10$. We sample a sparsity mask $B_t$ and a noise magnitude $E_t$ as:
\begin{align*}
    B_t &\sim \text{Bernoulli}(p), \\
    E_t &\sim \text{Exponential}(1/\sigma),
\end{align*}
where $p=0.1$ represents the spike probability and $\sigma=10$ is the scale parameter. We first compute an intermediate score $\tilde{S}_t$ by applying multiplicative noise only when the mask is active:
\[
    \tilde{S}_t = C (1 + B_t E_t).
\]
To simulate locally correlated volatility rather than isolated point outliers, the final score $S_t$ is obtained by applying a centered rolling max-filter of window size $W=25$:
\[
    S_t = \max_{\tau \in [t - \lfloor W/2 \rfloor, t + \lfloor W/2 \rfloor]} \tilde{S}_\tau.
\]
The total sequence length is $T=3000$.

\textbf{Local Adaptivity.}

\begin{figure}[H]
  \centering
  \includegraphics[trim={0 0 0 1.2cm}, clip, width=0.85\columnwidth]{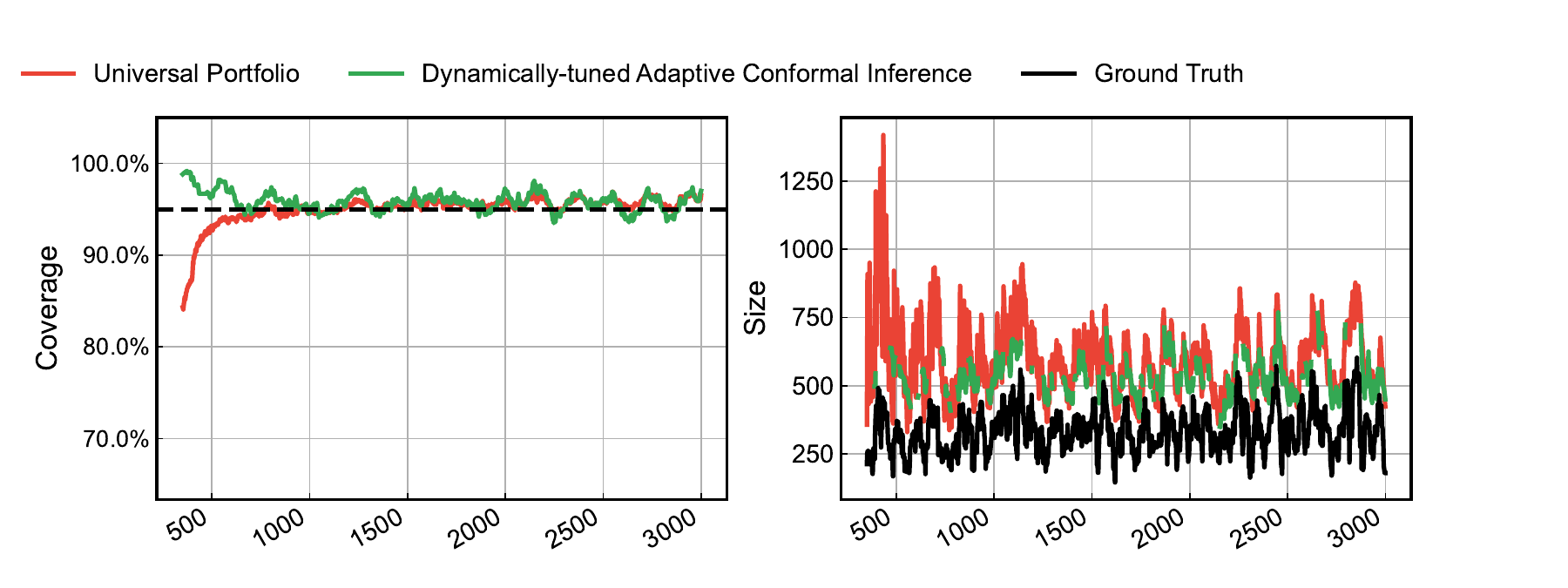}
  \caption{UP-OCP vs. DtACI for forecasting stationary synthetic data with random waves.}
  \label{fig:stationary-UP-vs-DtACI-local}
\end{figure} 

\begin{figure}[H]
  \centering
  \includegraphics[trim={0 0 0 1.2cm}, clip, width=0.85\columnwidth]{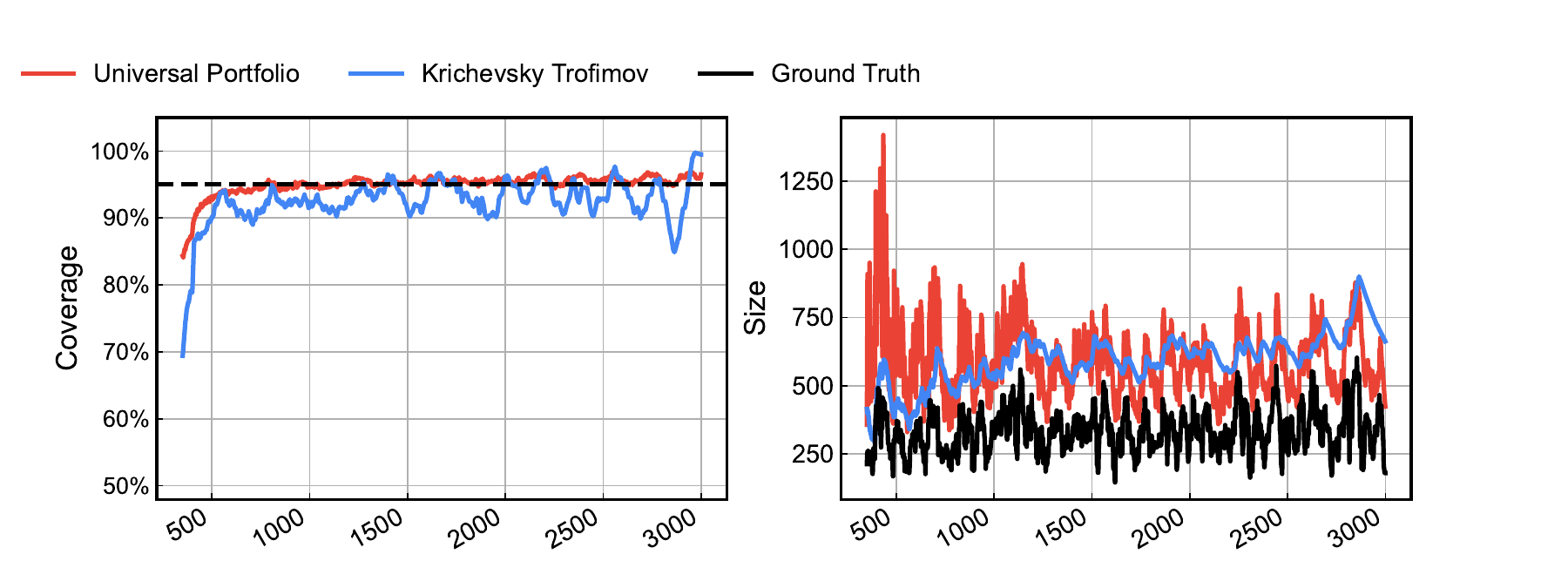}
  \caption{As in Figure~\ref{fig:stationary-UP-vs-DtACI-local}, UP-OCP vs. KT.}
  \label{fig:stationary-UP-vs-KT-local}
\end{figure} 

\begin{figure}[H]
  \centering
  \includegraphics[trim={0 0 0 1.2cm}, clip, width=0.85\columnwidth]{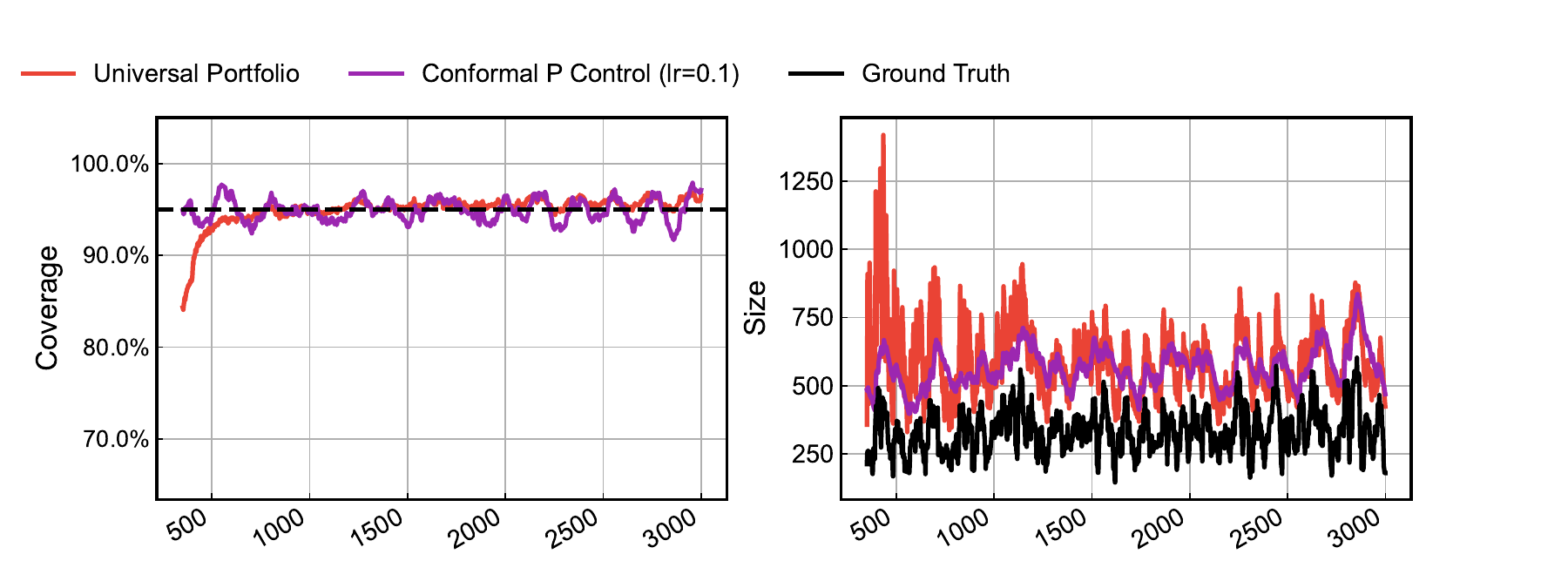}
  \caption{As in Figure~\ref{fig:stationary-UP-vs-DtACI-local}, UP-OCP vs. P Ctrl (lr=0.1).}
  \label{fig:stationary-UP-vs-P_01-local}
\end{figure}

\begin{figure}[H]
  \centering
  \includegraphics[trim={0 0 0 1.2cm}, clip, width=0.85\columnwidth]{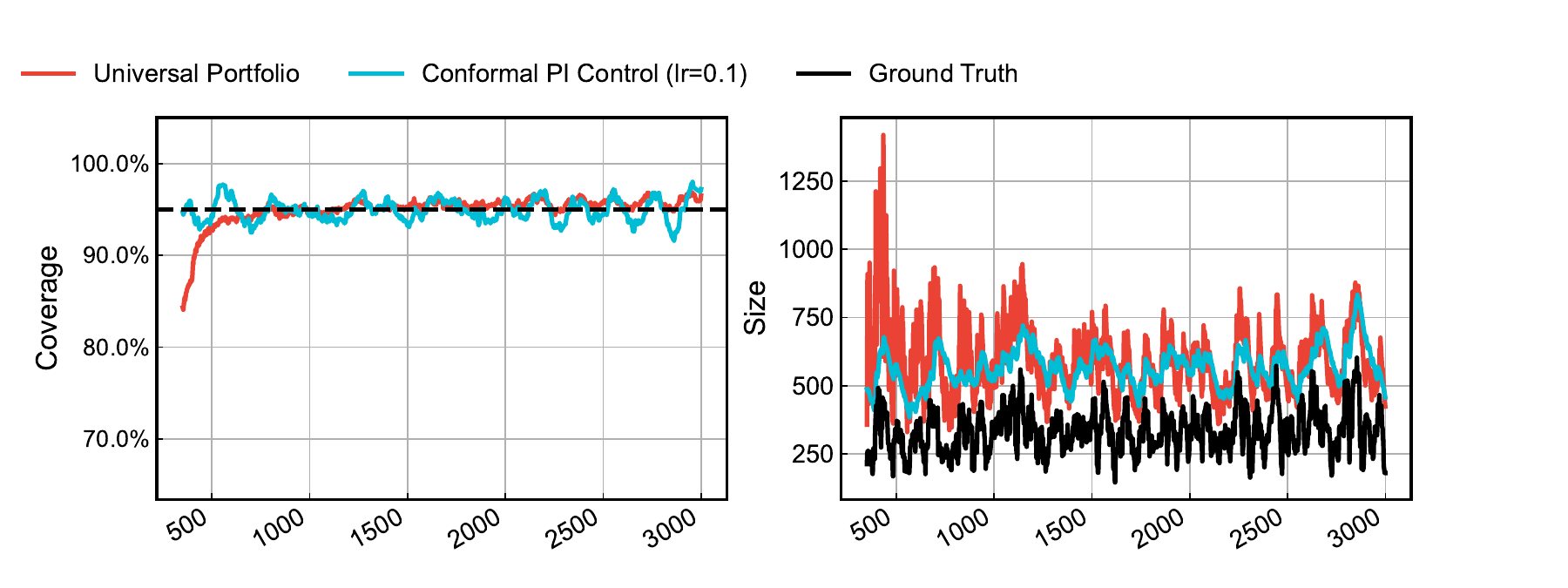}
  \caption{As in Figure~\ref{fig:stationary-UP-vs-DtACI-local}, UP-OCP vs. PI Ctrl (lr=0.1).}
  \label{fig:stationary-UP-vs-PI_01-local}
\end{figure}

\begin{figure}[H]
  \centering
  \includegraphics[trim={0 0 0 1.2cm}, clip, width=0.85\columnwidth]{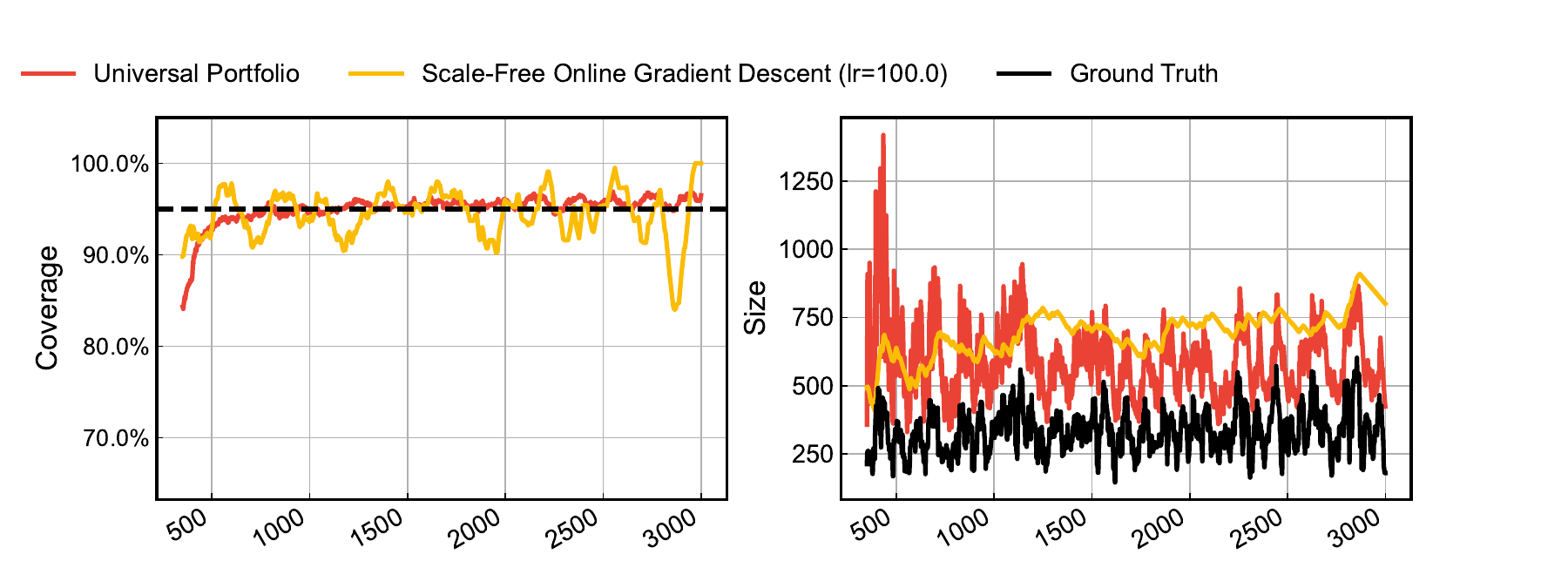}
  \caption{As in Figure~\ref{fig:stationary-UP-vs-DtACI-local}, UP-OCP vs. SFOGD (lr=100).}
  \label{fig:stationary-UP-vs-SFOGD_100-local}
\end{figure}

\begin{table}[H]
  \caption{Quantitative Comparison on the Stationary Dataset (synthetic).}
  \label{tab:stationary-1vall-metrics-full}
  \centering
  \begin{tabular}{lcccccc}
    \toprule
    & UP & KT & DtACI & SFOGD (lr=100) & P Ctrl (lr=0.1) & PI Ctrl (lr=0.1) \\
    \midrule
    Marginal coverage      & \textbf{0.952} & 0.93  & 0.958    & 0.946 & 0.949 & 0.949 \\
    Average set size       & 592   & 596   & $\infty$ & 694   & 566   & 567   \\
    Median set size        & \textbf{500}   & 574   & 494      & 691   & 531   & 534   \\
    75\% quantile set size & 746   & 724   & 665      & 841   & 692   & 694   \\
    90\% quantile set size & 1070  & 892   & $\infty$ & 962   & 882   & 880   \\
    95\% quantile set size & 1320  & 994   & $\infty$ & 1050  & 1000  & 1010  \\
    \bottomrule
  \end{tabular}
\end{table}

\newpage
\textbf{More Pareto Frontiers and Target-level Tracking.}

\begin{figure}[H]
  \centering
  % --- Top Row: 1x2 ---
  \begin{minipage}[t]{0.49\columnwidth}
    \centering
    \includegraphics[trim={0 0 0 0.5cm}, clip, width=\linewidth]{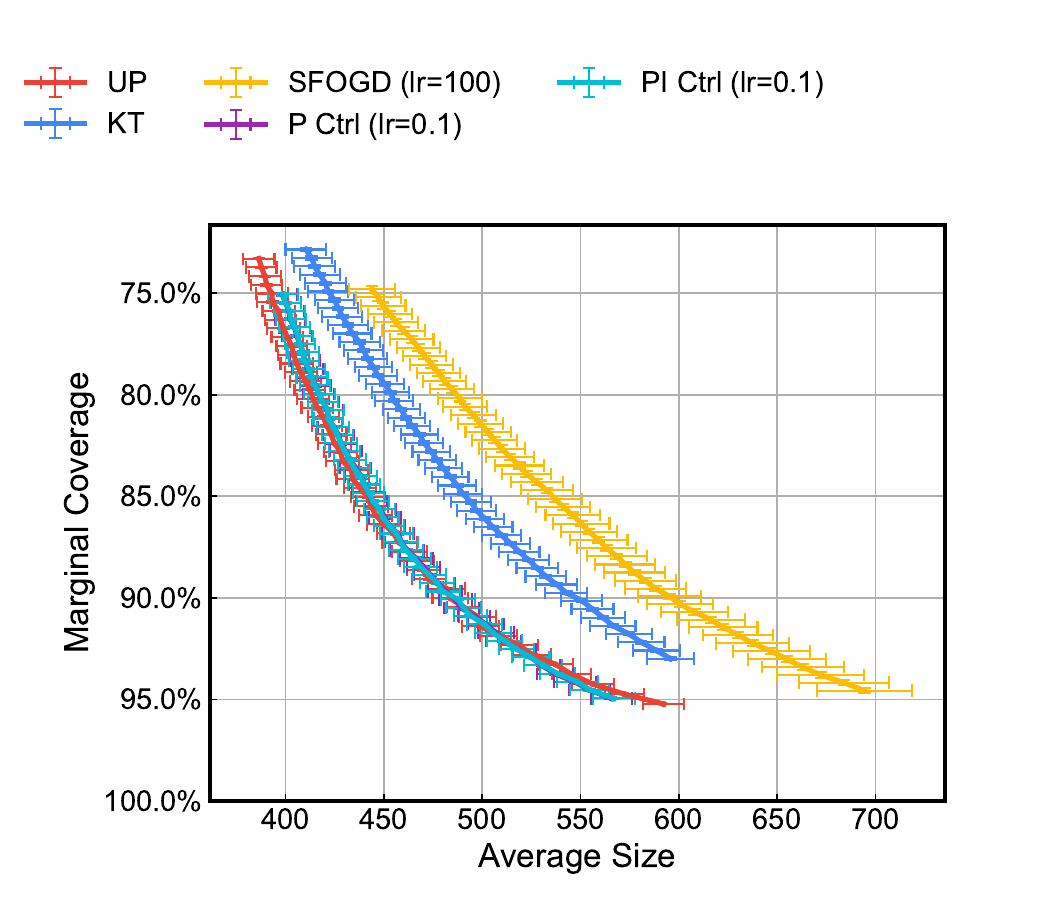}
    \caption{Mean prediction set sizes on synthetic stationary data with random waves.}
    \label{fig:stationary-Pareto-average}
  \end{minipage}
  \hfill
  \begin{minipage}[t]{0.49\columnwidth}
    \centering
    \includegraphics[trim={0 0 0 0.5cm}, clip, width=\linewidth]{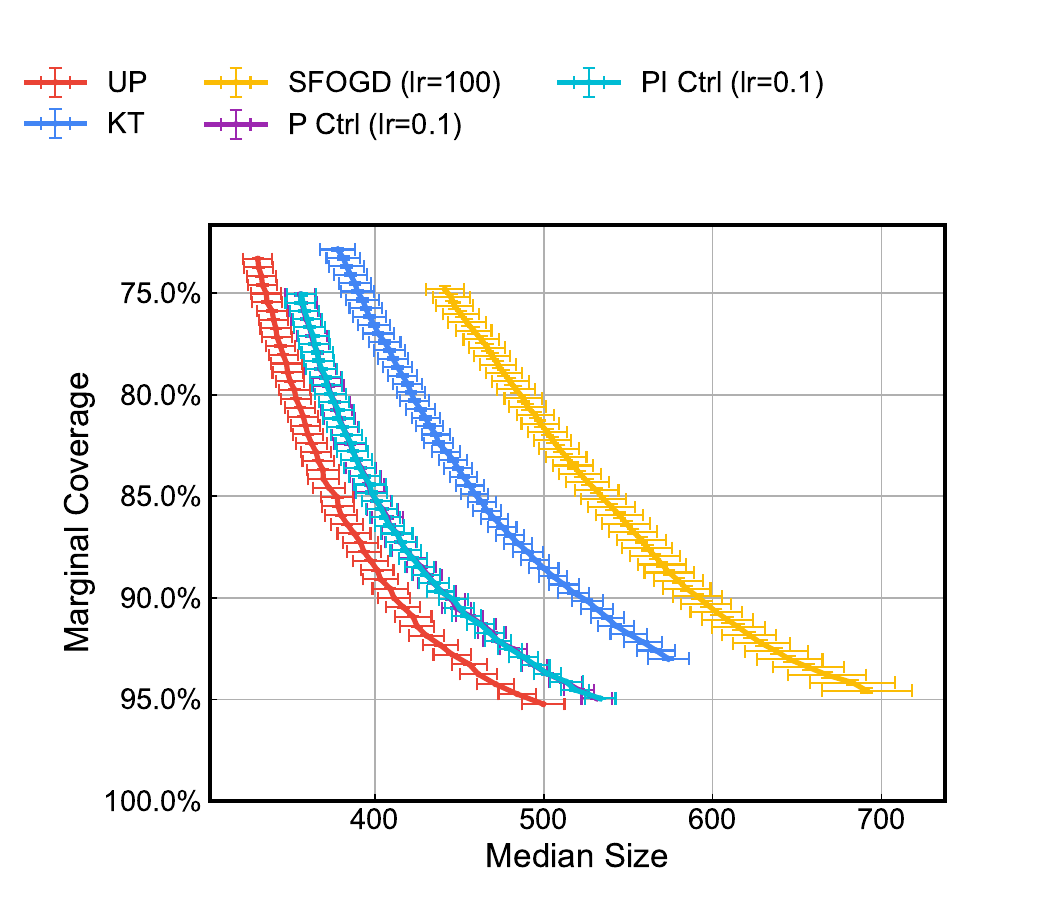}
    \caption{Median prediction set sizes on synthetic stationary data with random waves.}
    \label{fig:stationary-Pareto-median}
  \end{minipage}
  
  \vspace{0.5cm}

  \begin{minipage}[t]{0.49\columnwidth}
    \centering
    \includegraphics[trim={0 0 0 0.5cm}, clip, width=\linewidth]{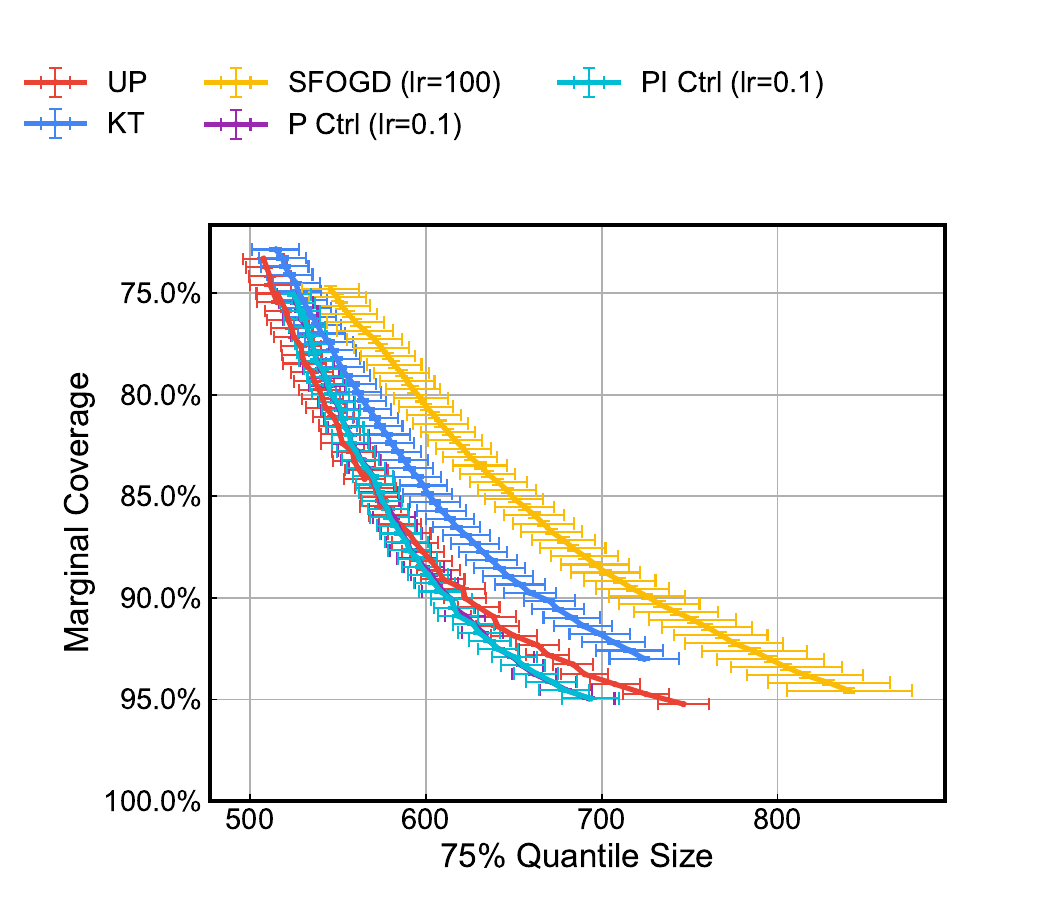}
    \caption{75\% quantile prediction set sizes on synthetic stationary data with random waves.}
    \label{fig:stationary-Pareto-q75}
  \end{minipage}
  \hfill
  \begin{minipage}[t]{0.49\columnwidth}
    \centering
    \includegraphics[trim={0 0 0 1cm}, clip, width=\linewidth]{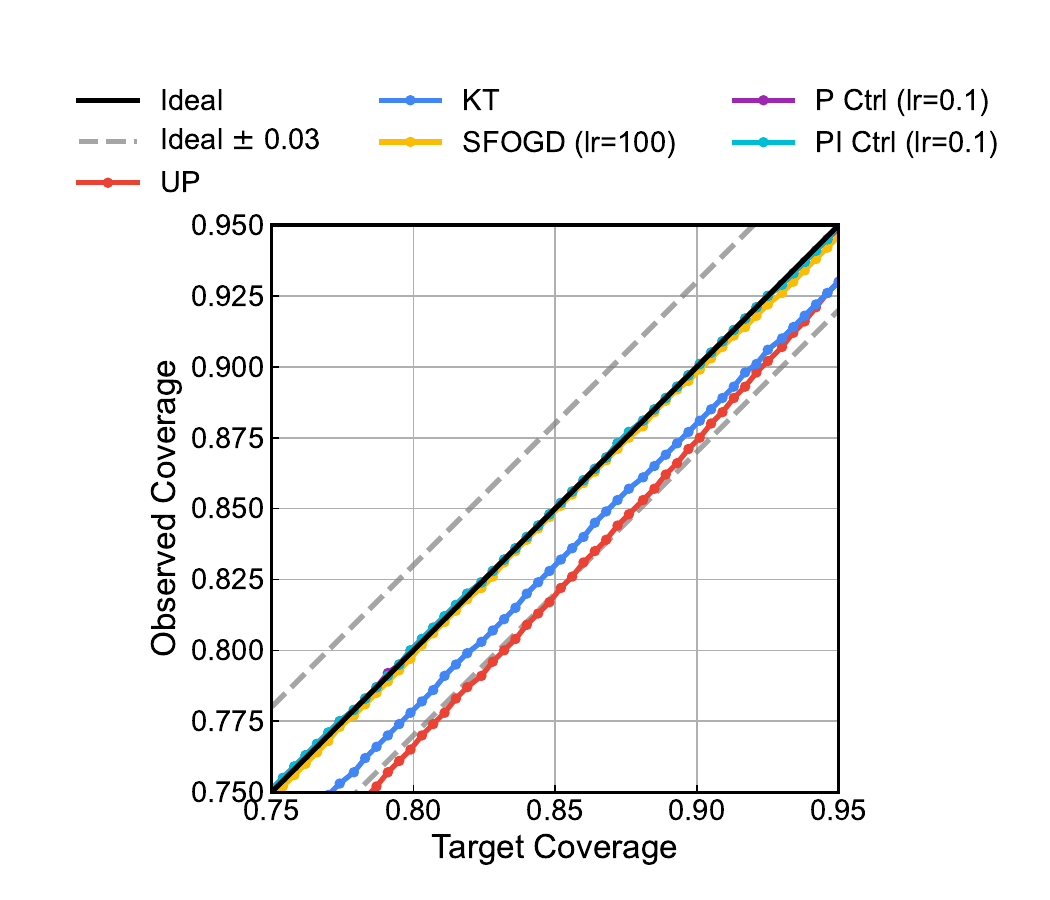}
    \caption{Realized vs. target coverage on synthetic stationary data with random waves. Most methods track the diagonal within a small tolerance ($\pm$ 0.03).}
    \label{fig:stationary-tracking-targets}
  \end{minipage}
\end{figure}

\newpage
\subsection{Quadratic Trend with Random Waves}
At last, we evaluate the algorithms on quadratic drift, which combines the sparse noise structure of the stationary regime with a non-stationary, monotonically increasing trend. This tests the ability of the algorithms to track a drifting baseline while remaining robust to random waves.

Formally, the underlying trend $T_t$ follows a quadratic trajectory starting at $0$ and ending at $20$ over $T=3000$ steps:
\begin{equation}
    T_t = \frac{20}{T^2} t^2.
\end{equation}
The noise generation follows the same multiplicative, sparse structure as the stationary wavelet experiment. We sample a Bernoulli mask $B_t \sim \text{Bernoulli}(0.1)$ and exponential noise $E_t \sim \text{Exponential}(1/10)$. The raw score $\tilde{S}_t$ applies this noise to the drifting baseline:
\begin{equation}
    \tilde{S}_t = T_t (1 + B_t E_t).
\end{equation}
Finally, the observed score $S_t$ is the result of a rolling max-pooling operation with window size $W=25$, creating locally correlated volatility structures around the drift.

\textbf{Local Adaptivity.}
\begin{figure}[H]
  \centering
  \includegraphics[trim={0 0 0 1.2cm}, clip, width=0.85\columnwidth]{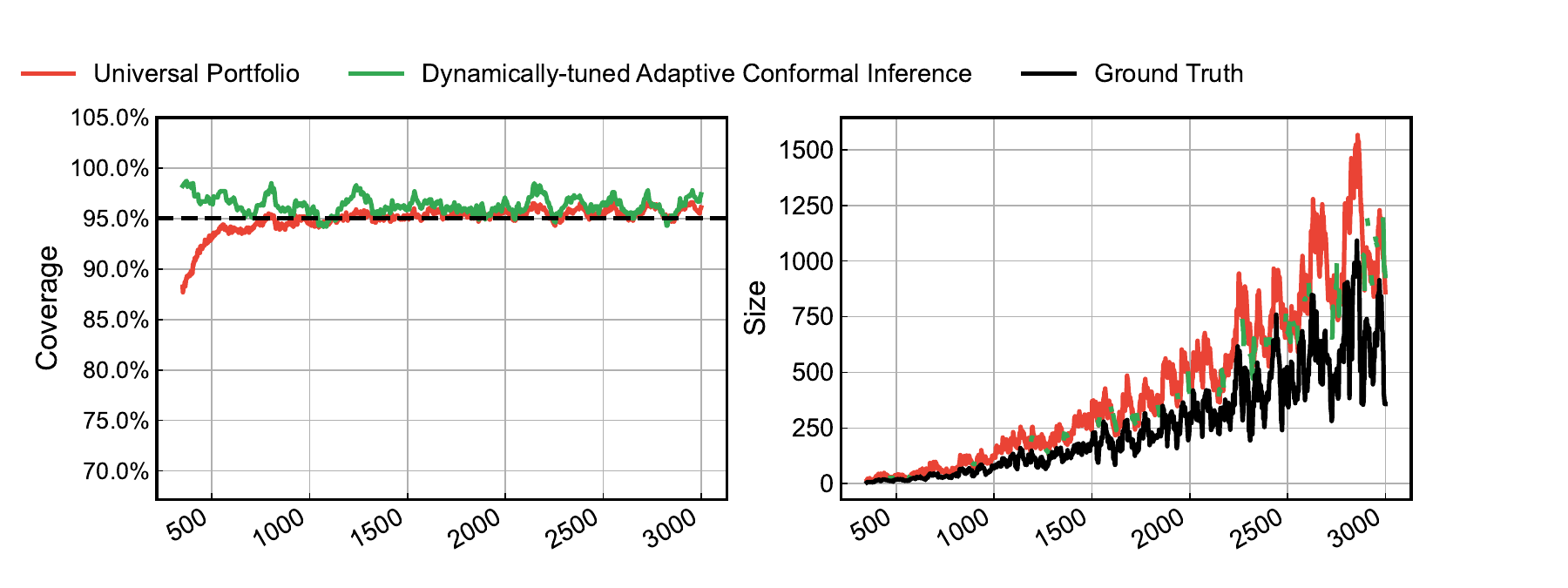}
  \caption{UP-OCP vs. DtACI for forecasting synthetic data with quadratic trend and random waves.}
  \label{fig:mix-UP-vs-DtACI-local}
\end{figure} 

\begin{figure}[H]
  \centering
  \includegraphics[trim={0 0 0 1.2cm}, clip, width=0.85\columnwidth]{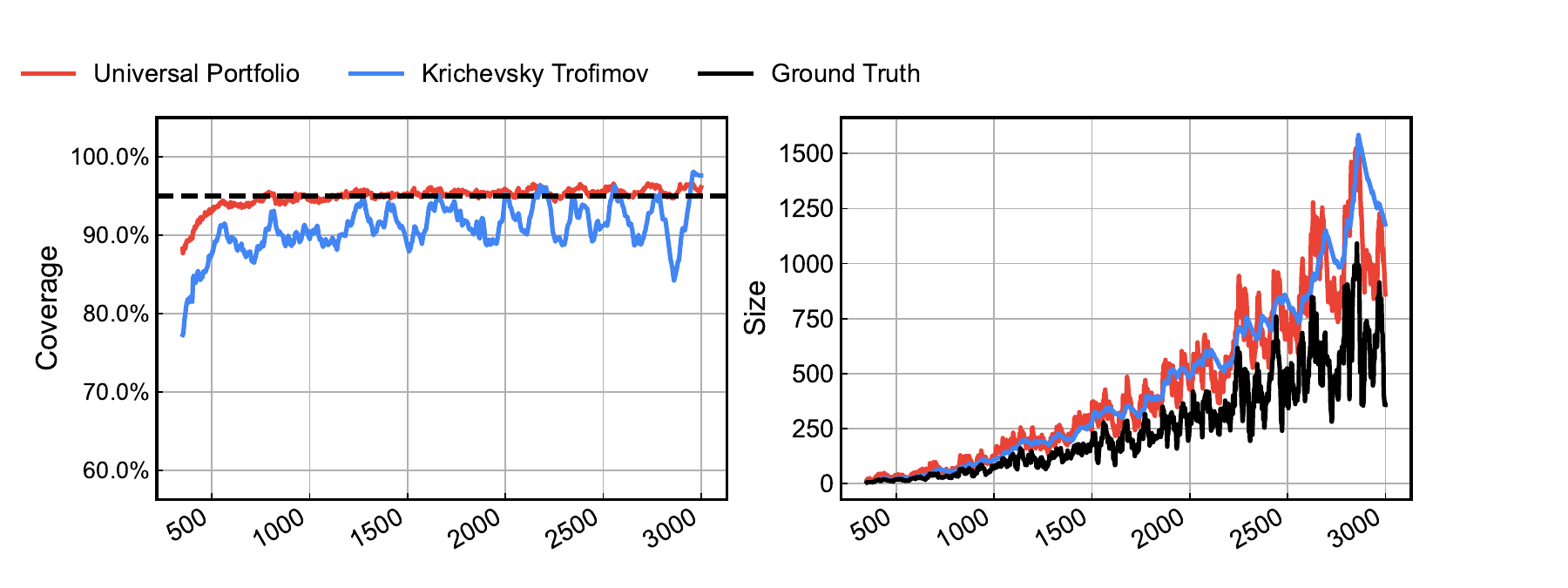}
  \caption{As in Figure~\ref{fig:mix-UP-vs-DtACI-local}, UP-OCP vs. KT.}
  \label{fig:mix-UP-vs-KT-local}
\end{figure} 

\begin{figure}[H]
  \centering
  \includegraphics[trim={0 0 0 1.2cm}, clip, width=0.85\columnwidth]{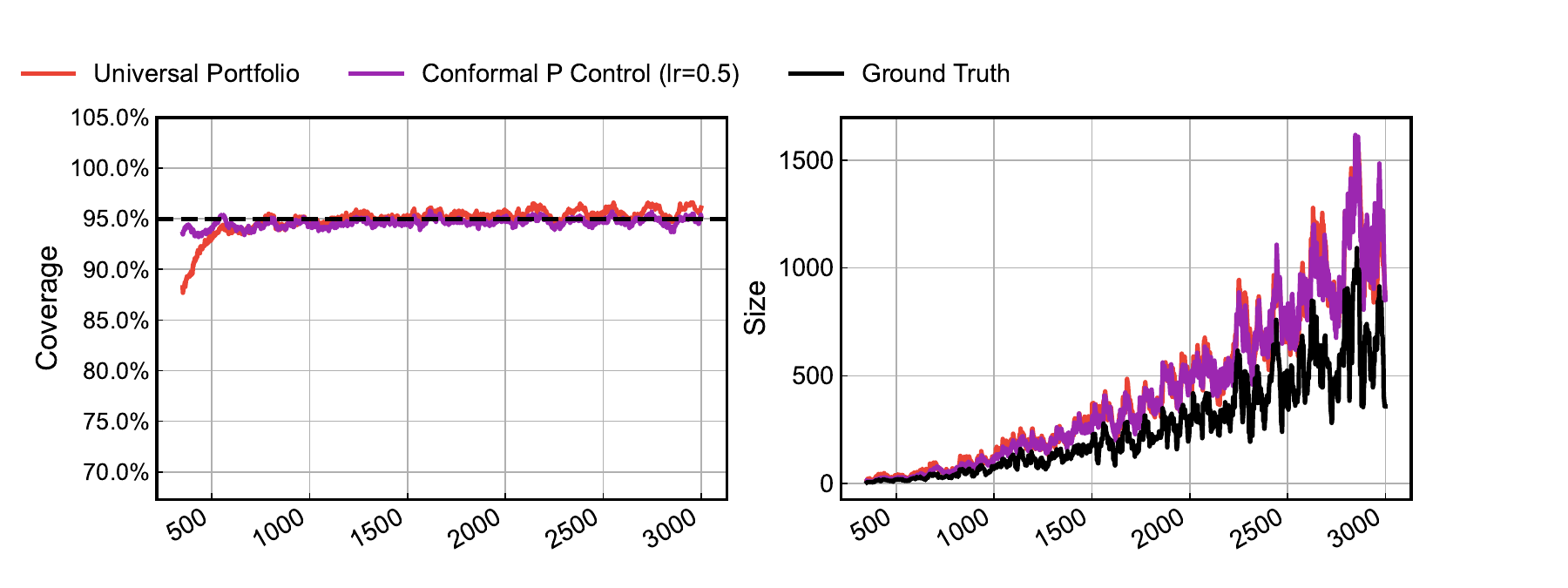}
  \caption{As in Figure~\ref{fig:mix-UP-vs-DtACI-local}, UP-OCP vs. P Ctrl (lr=0.5).}
  \label{fig:mix-UP-vs-P_05-local}
\end{figure}

\begin{figure}[H]
  \centering
  \includegraphics[trim={0 0 0 1.2cm}, clip, width=0.85\columnwidth]{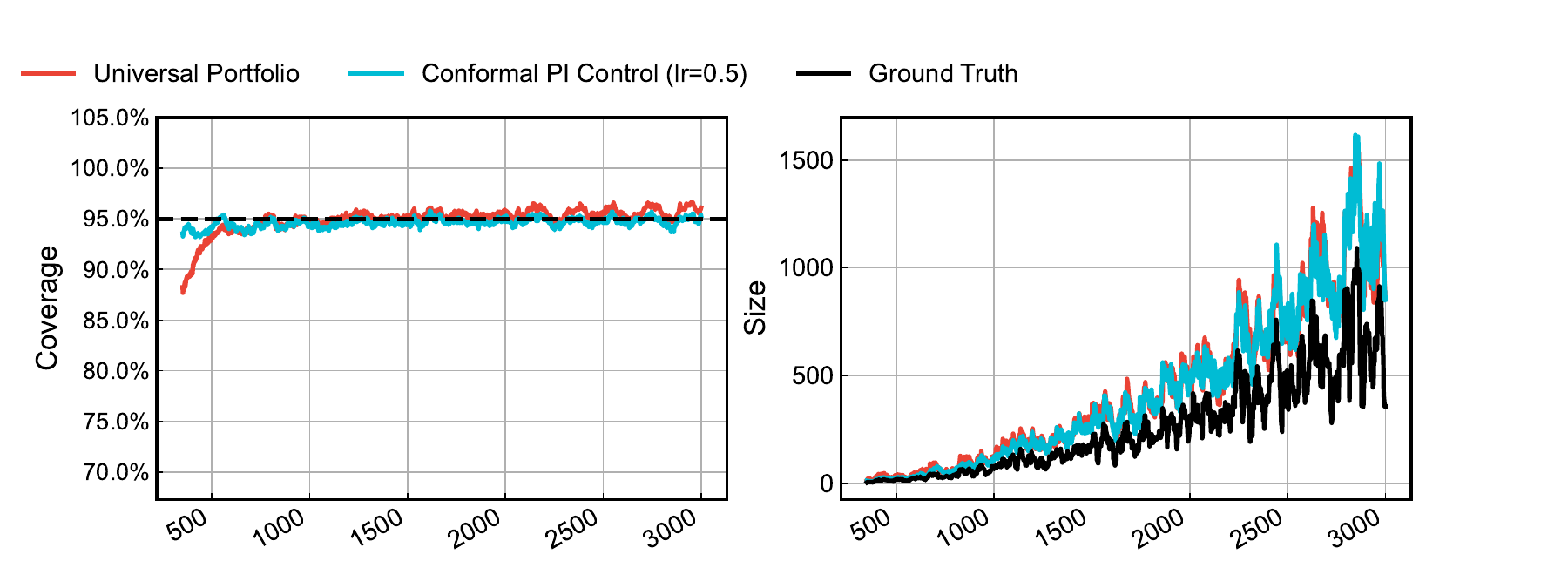}
  \caption{As in Figure~\ref{fig:mix-UP-vs-DtACI-local}, UP-OCP vs. PI Ctrl (lr=0.5).}
  \label{fig:mix-UP-vs-PI_05-local}
\end{figure}

\begin{figure}[H]
  \centering
  \includegraphics[trim={0 0 0 1.2cm}, clip, width=0.85\columnwidth]{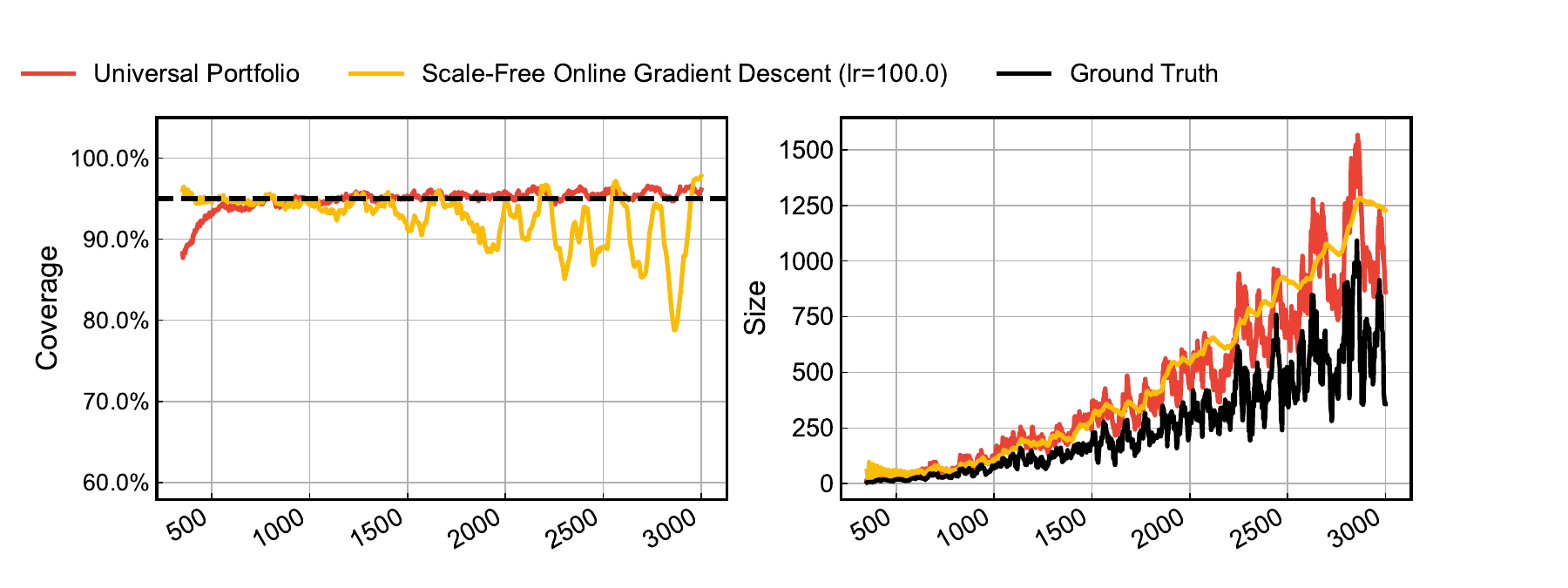}
  \caption{As in Figure~\ref{fig:mix-UP-vs-DtACI-local}, UP-OCP vs. SFOGD (lr=100).}
  \label{fig:mix-UP-vs-SFOGD_100-local}
\end{figure}

\begin{table}[H]
  \caption{Quantitative Comparison on the Mix Dataset (s synthetic).}
  \label{tab:mix-1vall-metrics-full}
  \centering
  \begin{tabular}{lcccccc}
    \toprule
    & UP & KT & DtACI & SFOGD (lr=100) & P Ctrl (lr=0.5) & PI Ctrl (lr=0.5) \\
    \midrule
    Marginal coverage      & \textbf{0.951} & 0.913 & 0.963    & 0.927 & 0.947 & 0.947 \\
    Average set size       & \textbf{433}   & 451   & $\infty$ & 467   & \textbf{429}   & \textbf{429}   \\
    Median set size        & \textbf{292}   & 308   & 406      & 324   & \textbf{285}   & \textbf{285}   \\
    75\% quantile set size & \textbf{624}   & 692   & $\infty$ & 790   & \textbf{626}   & \textbf{626}   \\
    90\% quantile set size & \textbf{1020}  & 1040  & $\infty$ & 1090  & \textbf{1020}  & \textbf{1020}  \\
    95\% quantile set size & 1320  & 1360  & $\infty$ & 1260  & 1310  & 1310  \\
    \bottomrule
  \end{tabular}
\end{table}

\newpage
\textbf{More Pareto Frontiers and Target-level Tracking.}

\begin{figure}[H]
  \centering
  % --- Top Row: 1x2 ---
  \begin{minipage}[t]{0.49\columnwidth}
    \centering
    \includegraphics[trim={0 0 0 0.5cm}, clip, width=\linewidth]{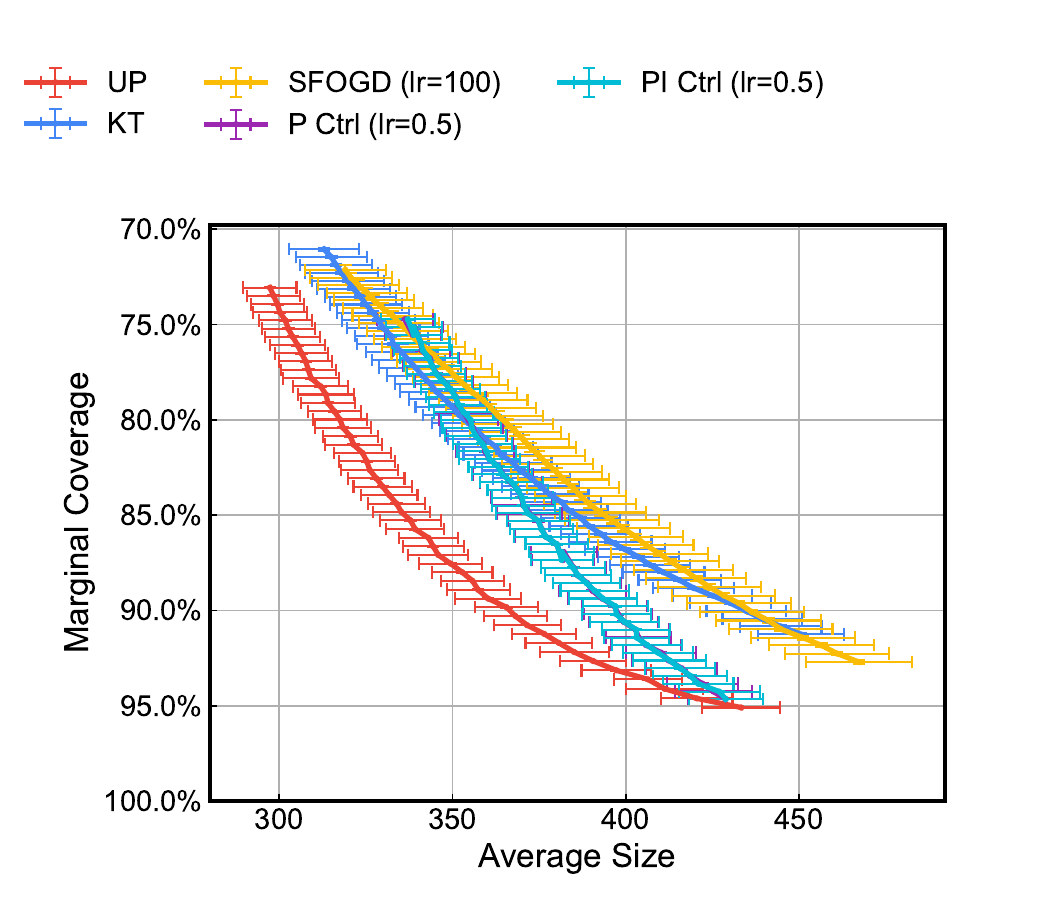}
    \caption{Mean prediction set sizes on synthetic data with quadratic trend and random waves.}
    \label{fig:mix-Pareto-average}
  \end{minipage}
  \hfill
  \begin{minipage}[t]{0.49\columnwidth}
    \centering
    \includegraphics[trim={0 0 0 0.5cm}, clip, width=\linewidth]{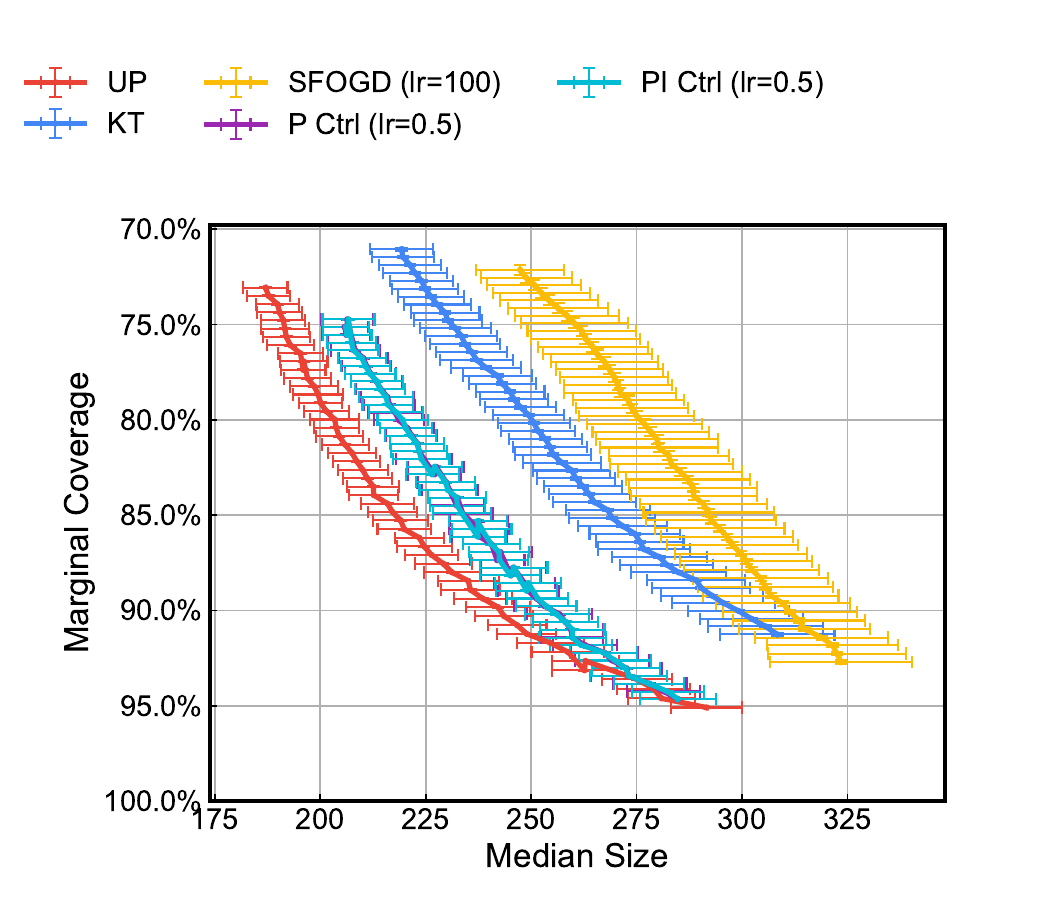}
    \caption{Median prediction set sizes on synthetic data with quadratic trend and random waves.}
    \label{fig:mix-Pareto-median}
  \end{minipage}
  
  \vspace{0.5cm} % Vertical spacing between rows

  % --- Bottom Row: Left Aligned ---
  \begin{minipage}[t]{0.49\columnwidth}
    \centering
    \includegraphics[trim={0 0 0 0.5cm}, clip, width=\linewidth]{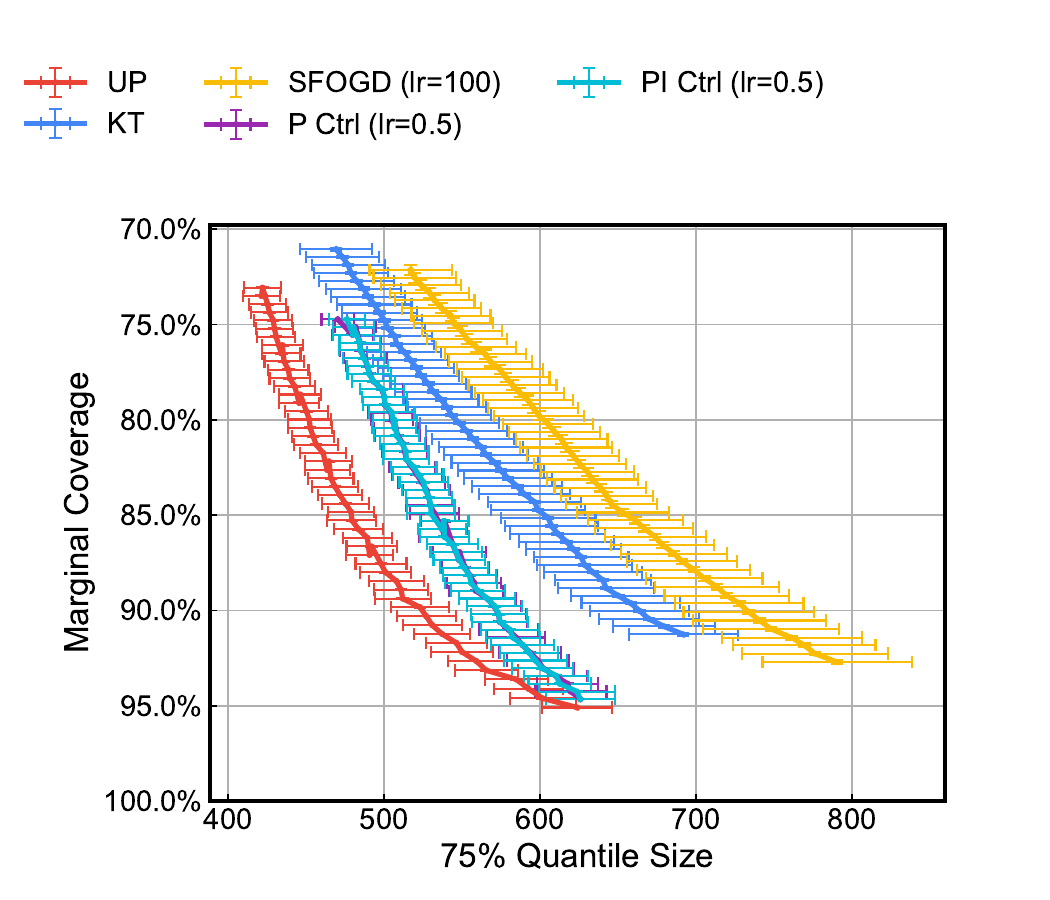}
    \caption{75\% quantile prediction set sizes on synthetic data with quadratic trend and random waves.}
    \label{fig:mix-Pareto-q75}
  \end{minipage}
  \hfill
  \begin{minipage}[t]{0.49\columnwidth}
    \centering
    \includegraphics[trim={0 0 0 1cm}, clip, width=\linewidth]{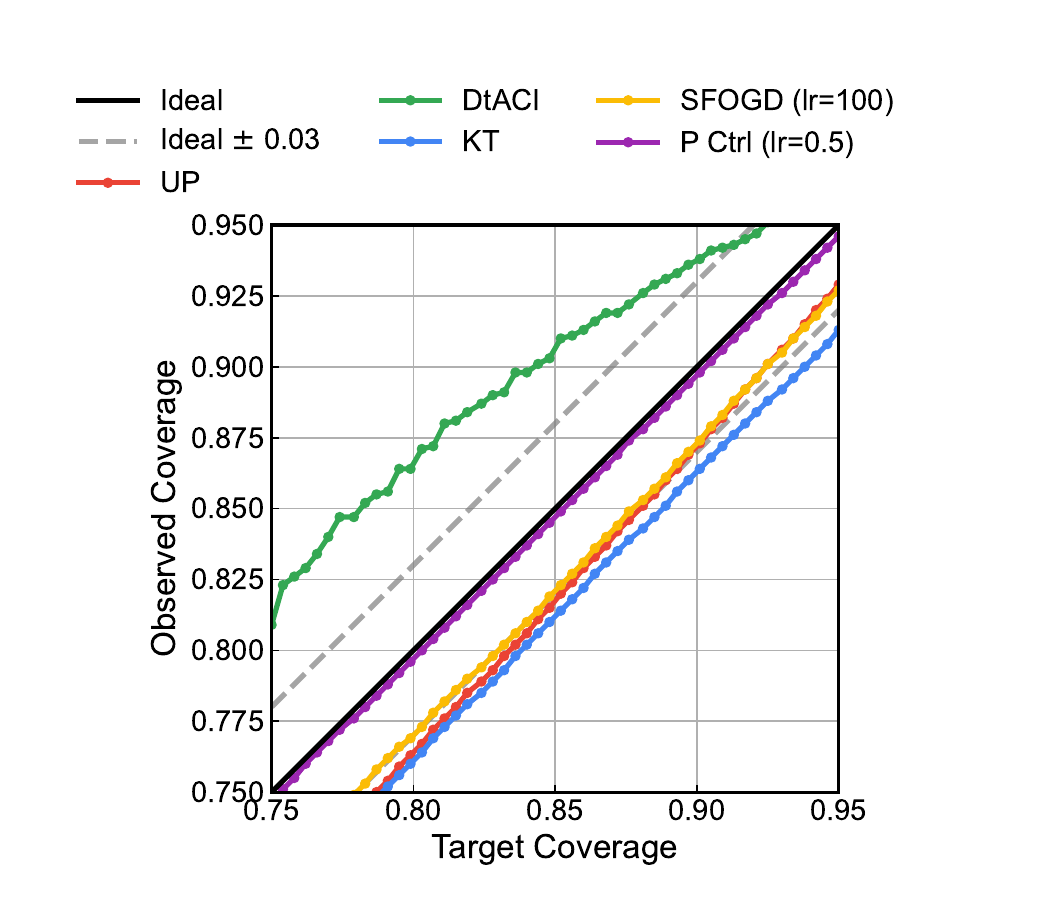}
    \caption{Realized vs. target coverage on synthetic data with quadratic trend and random waves. Most methods track the diagonal within a small tolerance ($\pm$ 0.03).}
    \label{fig:mix-tracking-targets}
  \end{minipage}
\end{figure}

\newpage
\section{Sensitivity of Parameterized Baselines to Hyperparameters}
\label{app:sensitivity}

In this section, we empirically demonstrate the sensitivity of parameterized OCP methods (SF-OGD, P/PI Control) to hyperparameter choices. While tuned baselines can achieve competitive performance (as shown in Section~\ref{sec:experiments}), selecting these parameters requires an oracle or grid search that is not feasible in a true online setting. We give three such examples of failure below. These examples underscore that parameterized methods can not be naively plugged in; they require careful tuning. UP-OCP avoids these divergence modes by design without requiring manual tuning.

\begin{figure}[H]
  \centering
  \includegraphics[trim={0 0 0 1.2cm}, clip, width=0.85\columnwidth]{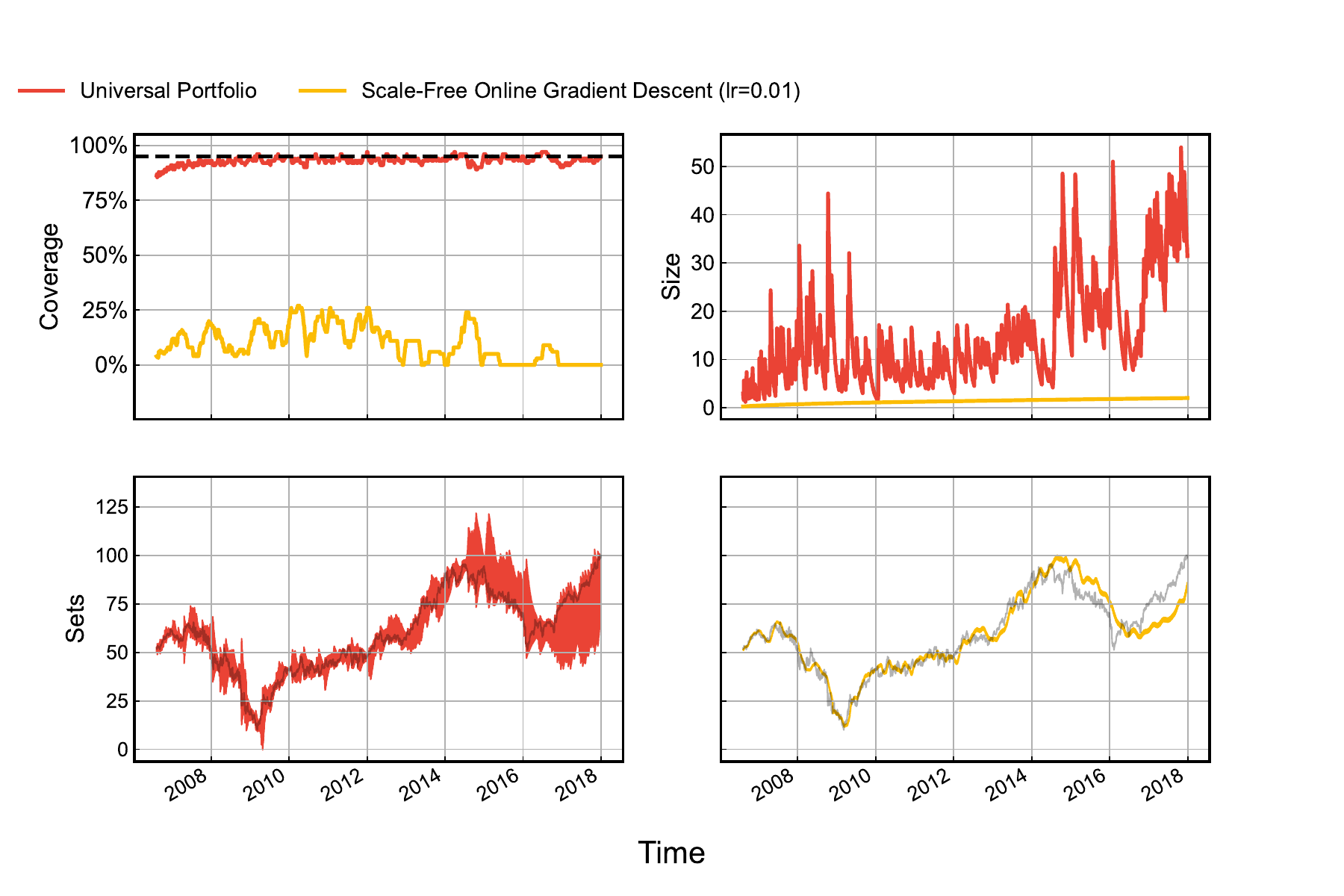}
  \caption{UP-OCP vs. SF-OGD (lr=0.01) on AXP dataset; SF-OGD (lr=0.01)fails to expand the prediction sets sufficiently, resulting in intervals that are consistently too narrow (right panel). The marginal coverage (yellow, left panel) collapses to nearly 0\%, far below the 95\% target.}
  \label{fig:undercoverage}
\end{figure} 

\begin{figure}[H]
  \centering
  \includegraphics[trim={0 0 0 1.2cm}, clip, width=0.85\columnwidth]{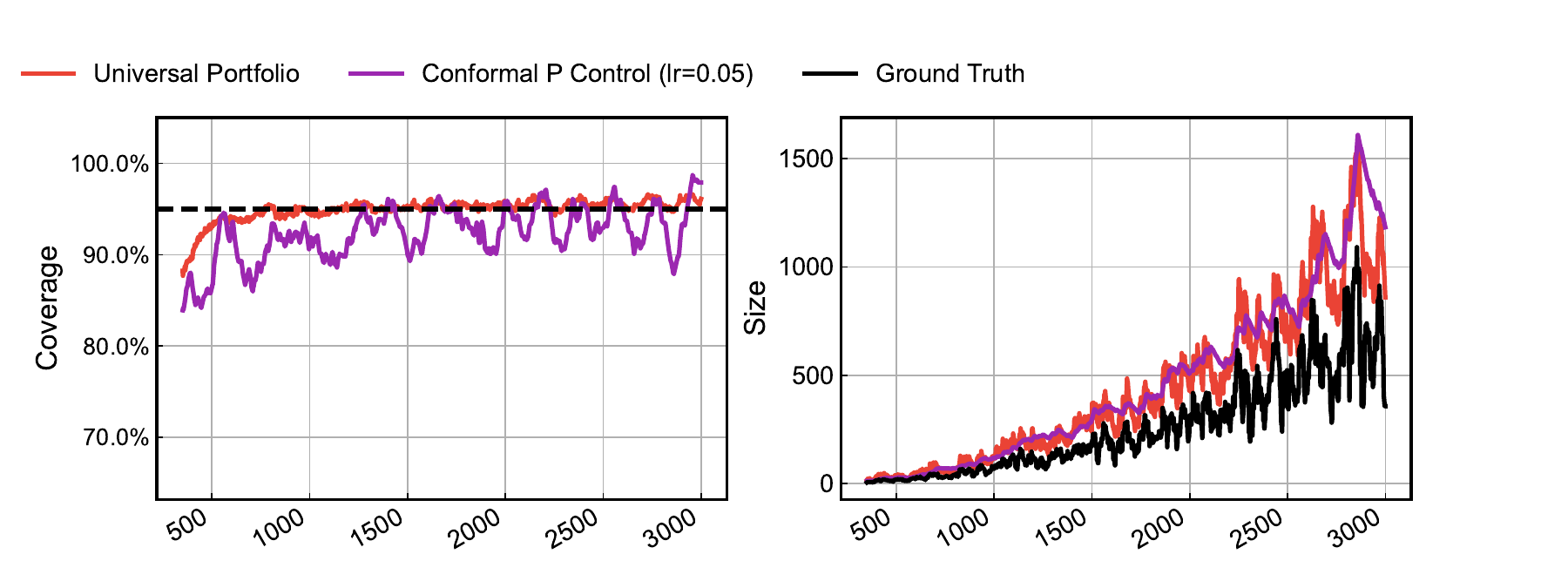}
  \caption{UP-OCP vs. P Ctrl (lr=0.05) on synthetic data with quadratic trend and random waves; While the parameter-free UP-OCP (red) maintains stable coverage near the target, the P Controller (purple) exhibits significant oscillation. Intuitively this indicates that the controller is over-reacting to single data points. }
  \label{fig:instability}
\end{figure} 

\begin{figure}[H]
  \centering
  \includegraphics[trim={0 0 0 0}, clip, width=0.85\columnwidth]{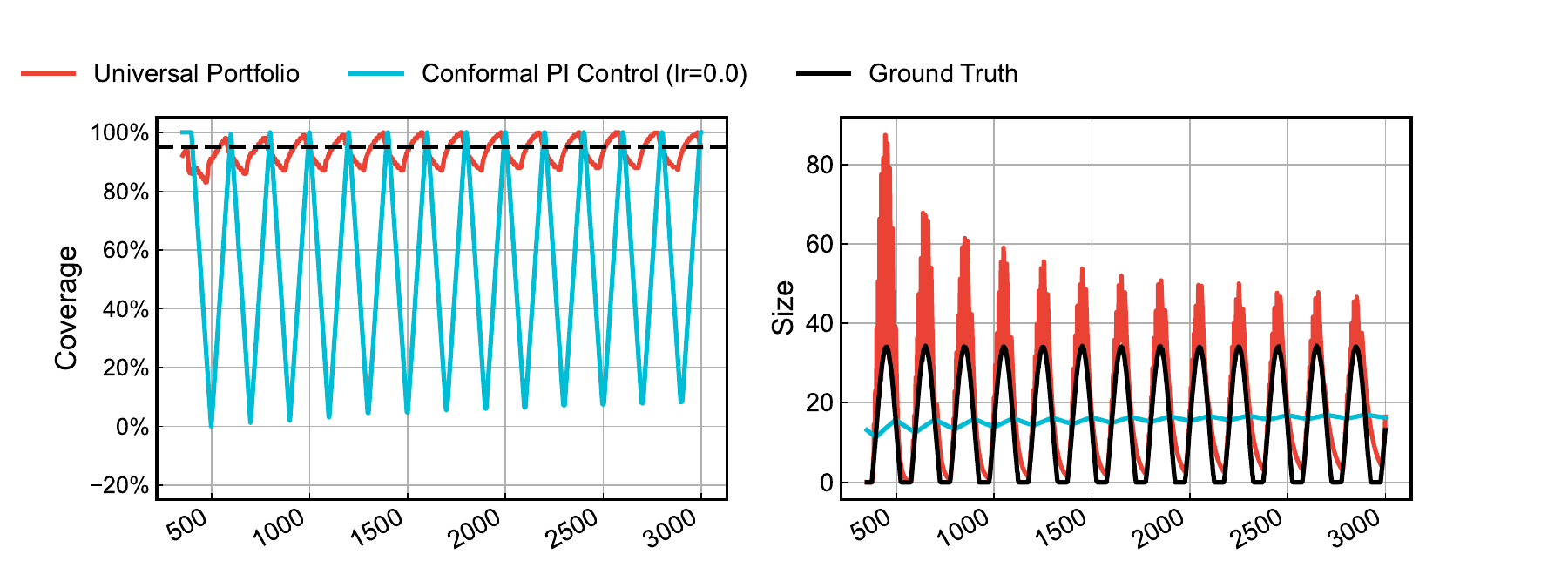}
  \caption{UP-OCP vs. PI Ctrl (lr=0.0) on synthetic sinusoid data; PI Controller fails to adapt to periodicity of the ground truth (black). The prediction set sizes (cyan, right panel) remain effectively constant. The coverage (left panel) oscillates deterministically between 0\% and 100\% as the ground truth noise wave passes in and out. In contrast, UP-OCP (red) correctly modulates the interval width to track the sinusoidal pattern, maintaining valid coverage.}
  \label{fig:no-tracking}
\end{figure}

\end{document}